\documentclass{article}
\usepackage[scaled=1.0]{helvet}
\usepackage[margin=1in]{geometry}
\usepackage[utf8]{inputenc}
\usepackage[T1]{fontenc}
\usepackage{setspace} 
\usepackage{hyperref}
\usepackage{url}
\usepackage{booktabs}
\usepackage{amsfonts}
\usepackage{nicefrac}
\usepackage{microtype}
\usepackage[table]{xcolor}

\usepackage{amsmath} 
\usepackage{amsthm} 
\usepackage{bm} 
\usepackage{bbold} 
\usepackage[labelfont=bf]{caption} 
\usepackage{algpseudocode} 
\usepackage{algorithm} 
\usepackage{graphicx} 
\usepackage{float} 
\usepackage{todonotes} 
\usepackage[export]{adjustbox} 
\usepackage{wrapfig} 
\usepackage{natbib}
\usepackage{rotating}
\usepackage[normalem]{ulem} 

\usepackage{colortbl}
\usepackage{array}
\usepackage{xcolor}

\usepackage[section]{placeins}

\newcolumntype{C}{>{\centering\arraybackslash}m{2.6cm}}

\newcommand{\filledcell}{%
    \raisebox{-0.5ex}{\color{black}\rule{2.6cm}{2.5ex}}%
}

\algnewcommand{\LineComment}[1]{\State \(\triangleright\) \textit{#1}} 

\newtheorem{theorem}{Theorem}[section]
\newtheorem{proposition}[theorem]{Proposition}

\newtheorem{lemma}[theorem]{Lemma}

\title{\textbf{AutoHebb: Self-Modifying Neural Networks that Zero-Shot Infer and Continually Learn}}

\author{%
  Daniel Levine\textsuperscript{1}\thanks{Correspondence to:
    \texttt{daniel.levine@yale.edu},
    \texttt{david.vandijk@yale.edu},
    \texttt{kaykenneth@gmail.com}.},
  David van Dijk\textsuperscript{1}\footnotemark[1],
  Kenneth Kay\textsuperscript{2,3}\footnotemark[1]
  \\[0.5em]
  {\normalfont\small
    \textsuperscript{1}Yale University}\\
  {\normalfont\small
    \textsuperscript{2}Howard Hughes Medical Institute, UCSF}\\
  {\normalfont\small
    \textsuperscript{3}Department of Brain and Cognitive Sciences, University of Rochester}
}
\date{}

\begin{document}
\maketitle
\vspace{-20pt}
\begin{abstract}
\setlength{\parindent}{2.5em}
\noindent Humans and animals have a unique capacity to generalize and learn in novel circumstances. This capacity manifests in various ways, including "0-shot" inference (ZI) and continual learning (CL) — two distinctively powerful cognitive abilities. Despite this importance, efforts to model and explain ZI and CL are seriously limited in several ways. First, existing models typically lack components critical to ZI and CL in the brain, namely, strong inductive biases (for ZI) and autonomous control of learning (for CL). Second, models that incorporate brain-like components can be less tractable to analysis if model components are overly complex (lacking simplicity), or too restricted to narrow problem classes if model components are overly bespoke (lacking expressivity).

To address these limitations, we introduce autonomous Hebbian networks ("AutoHebb"), a class of online learning neural networks that realize strong inductive biases and autonomously controlled learning, and are composed of simple yet expressive feedforward architectures. Unlike conventional models that learn via externally controlled optimization (e.g. standard gradient descent), these networks contain auxiliary architecture that combines external feedback with internally generated activity to produce a global learning signal (neuromodulator) fed back to the network. This global learning signal in turn gates a biologically grounded local learning rule (3-factor Hebbian rule) that updates plastic weights within the network. Using meta-learning to search for appropriate fixed (non-plastic) weights, we find that AutoHebb networks (i) succeed on ZI tasks, (ii) exhibit a collection of emergent (untrained) ZI and CL abilities, including rapid knowledge assembly and length generalization, (iii) admit reverse-engineering into explicit learning-inference algorithms that are in turn amenable to theoretical analysis, and (iv) outperform competitive learning models on both ZI and CL tasks, including learning multiple 0-shot tasks. Lastly, performance comparisons to models closely related to AutoHebb yet lacking specific components of AutoHebb implied an essential role for three model components: meta-learning, Hebbian plasticity, and autonomous neuromodulatory control. These findings indicate that AutoHebb may be useful in hypothesis generation, analysis, and discovery of biologically grounded forms of learning and generalization.
\end{abstract}

\section{Introduction}
\setlength{\parindent}{2.5em}

Humans and animals routinely encounter novel circumstances across a lifetime. Novel circumstances present a double challenge: first, generalization, which entails drawing upon prior knowledge to respond appropriately; second, learning, which entails acquiring knowledge while retaining prior knowledge. Brains overcome both challenges  \cite{kudithipudi2022biological} — yet it remains an open and central question what components, processes, and principles are responsible. 

Among types of generalization and learning, two are distinctively powerful and potentially foundational, yet remain deficient and poorly understood, thereby motivating new approaches. The first is "0-shot" ability, or the ability to respond to a novel setting, problem, or input in the absence of direct training or lifetime examples: an extreme form of generalization that requires that some form of strong prior structure, i.e.\ prior knowledge and/or inductive biases, already be in place \cite{lake2017building,goyal2022inductive}. In particular, 0-shot ability can be understood as a response to combinatorial (compositional) novelty -- i.e.\ where input is composed of new combinations of familiar elements -- thereby entailing inference of unseen structure (relational, conceptual, or semantic) underlying novel combinations. Given this requirement, we here call this ability "0-shot inference" (ZI). 
    
The second is continual learning (CL), the ability to learn progressively from a non-stationary stream of input \cite{ParisiContinuallifelonglearning2019,wang2024comprehensivesurveycontinuallearning,van_de_Ven_2025}. Despite being natural to the brain, CL is a longstanding limitation of models, thus suggesting a critical role for brain-like components. Further, CL is now widely regarded to entail more than the traditional criterion of avoiding catastrophic forgetting, namely entailing also \textit{transfer}: the ability to use ongoing learning to enhance past knowledge and/or subsequent learning. Transfer, particularly in the "backward" direction \cite{lopez2017gradient}, implies that systems that realize CL must be able to revisit and revise prior knowledge -- put differently, prior knowledge and ongoing learning must somehow closely, or even intrinsically, interact.
    
	Past work points to the foundational importance of both ZI and CL. ZI relates closely to insight learning \cite{kohler2018mentality}, cognitive maps \cite{TolmanPSYCHOLOGICALREVIEWCOGNITIVE1948,behrens2018cognitive}, relational and semantic memory \cite{cohen1993memory,eichenbaum2004conditioning}, and schemas \cite{bartlett1995remembering,rumelhart2017schemata,rumelhart1986sequential}, each characterized by prior structural knowledge that resolves combinatorial novelty. Studies across animal species indicate these abilities are widespread and evolutionarily ancient \cite{giurfa2001concepts,bird2009insightful,avargues2013conceptual,dona2020charles,lazareva2017categories,weise2022paper}, implying a foundational role in natural cognition. Indeed, various successes in AI/ML also point to the importance of prior constraints \cite{mitchell1980need,goyal2022inductive,bengio2021deep}, i.e.\ strong or appropriate inductive biases, including structured inductive biases \cite{battaglia2018relationalinductivebiasesdeep}. However, whether and how such inductive biases are realized remains an open question, apart from special cases \cite{saxe2019mathematical,lippl2024amathematicaltheory,abbe2024generalization,geerts2025relational,lippl2025when}.
	
    Past study of CL implies that major components remain similarly unresolved. Remarkably, most approaches rely on components external to the model and/or manually provided by the researcher, in contrast to the brain. These components include external memory buffers (e.g.\  \cite{lopez2017gradient,rolnick2019experiencereplaycontinuallearning}), external architectural augmentation (e.g.\  \cite{rusu2022progressiveneuralnetworks}), externally cued task identities (e.g.\  \cite{van2022three}), and, most widely used, externally calculated gradients (e.g. \cite{kirkpatrick2017overcomingcatastrophic,zenke2017continuallearning,lopez2017gradient,farajtabar2019orthogonalgradientdescentcontinual}).
    
    Here we hypothesize that a major limitation of current models is their divergence from an essential characteristic of learning in the brain: that learning is controlled autonomously by the brain itself. In particular, by virtue of being \textit{self-generated}, the signals in the brain that govern learning can reflect pre-existing structure -- whether innate or learned previously within a lifetime (via self-organizing processes such as Hebbian (or local) plasticity \cite{fremaux2016neuromodulated,schultz2016dopamine,gyorgy2019brain}) -- relevant to learning and generalization. Crucially, learning processes that exploit this pre-existing structure can in principle surpass forms of learning that are agnostic ("oblivious") to this pre-existing structure, including those relying on external components (e.g., loss functions and conventional gradients) \cite{kirkpatrick2017overcomingcatastrophic,farajtabar2019orthogonalgradientdescentcontinual,rypesc2025gradientfreecontinuallearning}. It is worth noting that the physical identity, behavioral/cognitive meaning, and exact roles of self-generated learning signals in the brain (e.g.\ dopamine, acetylcholine, plateau potentials; reward prediction error, novelty, surprise) remain open research questions \cite{fremaux2016neuromodulated,gerstner2018eligibility,magee2020synaptic}. These observations highlight that autonomous control of learning remains a crucial yet open issue in CL.

	In this study, we introduce "AutoHebb", a class of feedforward neural network that learns from a continuous input stream (online) and realizes both strong inductive biases (for ZI) and autonomous control of learning (for CL). AutoHebb is a plastic network, a neural network with weights that are updated locally \cite{soltoggio2018born,miconi2018differentiable,najarro2020meta,tyulmankov2022meta}), and builds directly on recent work establishing autonomous control of learning in recurrent plastic networks \cite{miconi2018backpropamine,pmlr-v202-miconi23a,miconi2025neural}. To gain both simplicity and expressivity, AutoHebb introduces feedforward architectural modifications to make network operations more amenable to analysis (simplicity), while also expanding the range of applicable tasks (expressivity). On four ZI/CL tasks that each involve simple 0-shot inferential structure, we find that AutoHebb networks consistently outperform popular learning models, express emergent (untrained) learning-generalization behaviors, and admit reverse-engineering algorithms. These results indicate that AutoHebb has exceptional ZI/CL abilities among modern learning models, and may be more broadly useful in investigating additional forms of biologically grounded learning and generalization, including as models of learning in the brain (neuroscience and psychology) and in applied settings.

\begin{figure}
    \centering
    \includegraphics[width=1\linewidth]{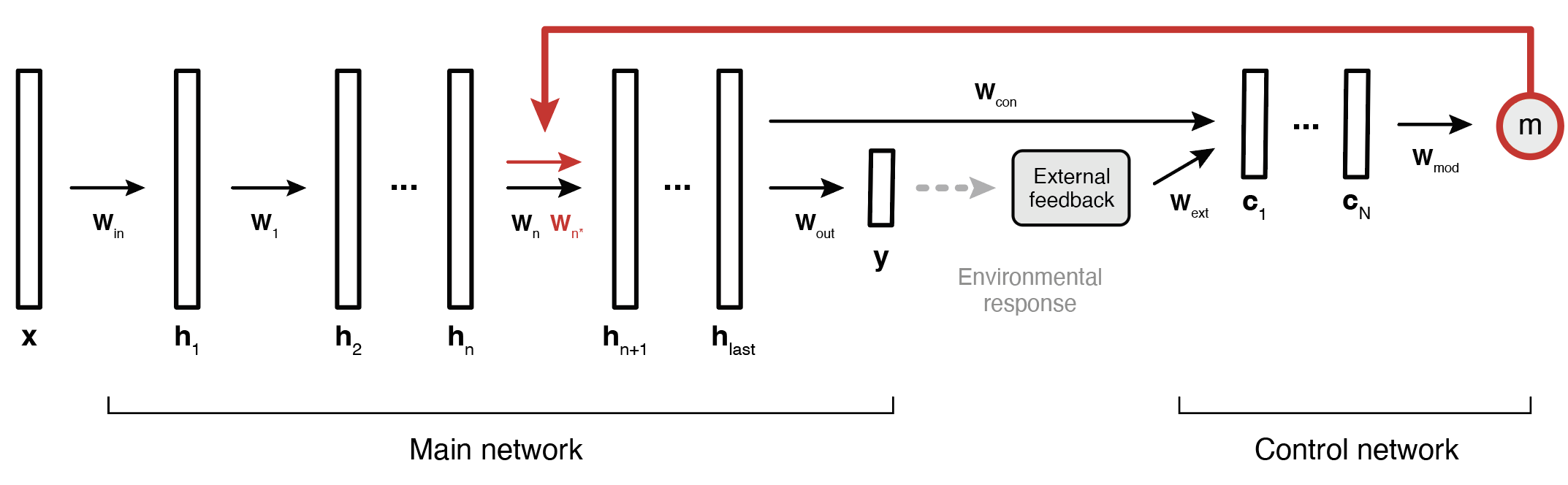}
    \caption{\textbf{AutoHebb architecture}. Network weights ($\mathbf{W}$) are either fixed (black) or plastic (red). The architecture consists of a "main" network (here an MLP with hidden layers $\mathbf{h}$) and an auxiliary "control" network (here an MLP with hidden layers $\mathbf{c}$) that receives two inputs: (1) internally generated activity from the main network (here from last hidden layer $\mathbf{h}_{\text{last}}$), and (2)  external input reflecting environmental (external) feedback given an output ($\mathbf{y}$) generated by the network. The control network generates a scalar-valued neuromodulatory signal ($m$, red circle) that gates changes to the plastic weights within network (here $\mathbf{W}_{n^*}$).  See Appendix \ref{appendixsec:autohebbtraining} for further details; additional network variants in Figure \ref{appendixfig:autohebb_two_arch}.}
    \label{fig:architecture}
\end{figure}

\section{AutoHebb}\label{sec:autohebb}

\subsection{Architecture}\label{subsec:autohebbarch}

AutoHebb is a class of simple yet expressive feedforward neural networks that implement autonomous control of learning. AutoHebb consists of two subnetworks, the \textbf{main} and \textbf{control} networks, and two types of weights, \textbf{fixed} and \textbf{plastic}. The role of the control network is to generate a (scalar-valued) global learning signal (`neuromodulator', akin to dopamine) whose function is to gate updates to the plastic weights. AutoHebb builds on recent work on recurrent plastic networks that have learning neuromodulation (\cite{miconi2018backpropamine}) by introducing three modifications: restriction to solely feedforward connectivity, isolation of neuromodulator generation to a separate subnetwork, and localization of plastic weights within the network (i.e.\ to specific layers). These modifications have four aims:\ (1) to enhance tractability to analysis, since each modification spatially restricts neural computation within the model architecture, (2) to expand the range of applicable tasks to those that do not require temporal/sequential inputs within trials, (3) to access network models that may better match experimental data (neural or behavioral) from the brain, (4) to access network models that may have better learning and generalization properties. In the present study, we mainly focus on aims (1) and (2).

The main network is a feedforward architecture with hidden layer activity given by
\begin{equation}\label{eq:autohebbhiddenstate}
\mathbf{h}_{i} = g((\mathbf{W}_i+\mathbf{A}_i \odot\mathbf{P}_i(t))\mathbf{h}_{i-1} + \mathbf{b}_i)\, ,
\end{equation}
where $g$ is an activation function ($\tanh$ unless stated otherwise), $\mathbf{W}_i$ and $\mathbf{b}_i$ are fixed weights and biases, and $\mathbf{A}_i \odot\mathbf{P}_i(t)$ is the Hadamard product of the plastic weights $\mathbf{P}_i(t)$ with fixed plasticity parameters $\mathbf{A}_i$, which govern the sign and magnitude of plastic weight changes. In this work, we consider binary response tasks with targets in $\{0,1\}$; we generate responses from a network with $L$ layers by applying a fixed linear readout to the last hidden state ($\mathbf{h}_L$), consisting of weights $\mathbf{W}_{\text{out}}$ and a bias $b_{\text{out}}$, and then a sigmoid $\sigma(x) = (1+e^{-x})^{-1}$:
\begin{equation}\label{eq:autohebbreadout}
y = \sigma(\mathbf{W}_{\text{out}}\mathbf{h}_{L} + b_{\text{out}})\, .
\end{equation}

The control network is a feedforward network consisting of fixed weights (though plastic weights are possible), whose inputs are $\mathbf{h}_L$ and an external feedback signal $\bm{\phi}$ and output is a scalar neuromodulator signal
\begin{equation}\label{eq:autohebbcontrol}
m =\mathcal{C}(\mathbf{h}_{L},\bm{\phi})\, .
\end{equation}
For all tasks in this study, $\bm{\phi}$ is $1$ for correct responses and $0$ for incorrect responses.

Note that $t$ corresponds to inputs presented on successive trials of a task; we refer to $t$ as "trials". Each trial models four steps: presentation of input (e.g.\ task stimuli) to the network (`input'), generation of an output by the main network (`output'), generation of the neuromodulator by the control network (`control'), and a plastic weight update (`update'), which is gated by the neuromodulator. More specifically, AutoHebb models the events that occur in a task trial after the learner has generated a response (an output): given the (trial-specific) response, the external environment provides an external feedback signal $\bm{\phi}$ (Fig.\ \ref{fig:architecture}) to the learner; in AutoHebb, this external feedback signal is an input to the control network, which then generates the scalar neuromodulator $m$. The neuromodulator gates updates to the network's plastic weights via a \textbf{3-factor Hebbian rule}, a local learning rule that captures current essential knowledge of plasticity in deep brain regions  \cite{fremaux2016neuromodulated,gerstner2018eligibility,magee2020synaptic,schmidgall2024brain}:
\begin{align}\label{eq:3factorrule}
\mathbf{P}_i(t+1) & = \lambda\mathbf{P}_i(t) + \eta m\mathbf{h}_{i}\mathbf{h}_{i-1}^\top\, .
\end{align}
where  $\eta$  is a fixed learning rate parameter and $\lambda \leq 1$ is a decay rate. The rule is "3-factor" as the weight update combines pre-synaptic activity ($\mathbf{h}_{\text{i-1}}$), post-synaptic activity ($\mathbf{h}_{\text{i}}$), and the global neuromodulatory signal ($m$); it is local because each synapse updates using only quantities available at that synapse, without requiring backpropagated error signals or access to the global loss. These various features of AutoHebb: feedforward connectivity, modularity, localizability, and grounding in a brain-derived learning rule together yield a model that is biologically relevant, potentially tractable to analysis, applicable across diverse tasks.

\subsection{Meta-learning}\label{subsec:metalearning}

\begin{wrapfigure}{r}{0.5\textwidth}
    \centering
    \vspace{-\intextsep}
    \includegraphics[width=1.0\linewidth]{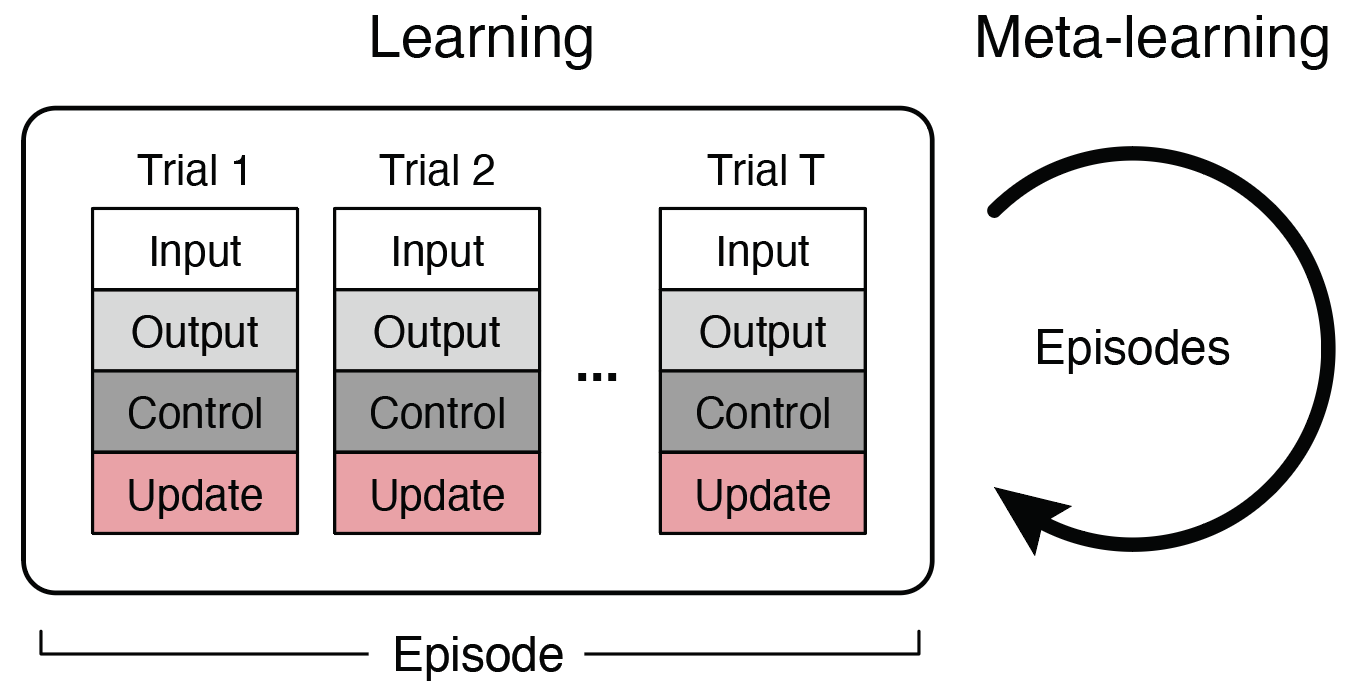}
    \caption{\textbf{Meta-learning in AutoHebb.} In AutoHebb networks, fixed (non-plastic) network weights were identified through meta-learning optimization, which involves two levels of adaptation: an "inner" and "outer" loop. In the inner loop, corresponding to learning, (3-factor Hebbian) plastic updates to the network occur across an episode, i.e.\ a sequence of trials (where each trial involves input, output, control, and update steps; see main text). At the start of each episode, fixed weights are frozen and plastic weights are set to 0. In the outer loop, corresponding to meta-learning (or meta-training), the network's fixed weights are updated via optimization based on performance within the episode and/or testing following the episode.}
    \label{fig:metalearning}
\end{wrapfigure}

In AutoHebb networks, all fixed (non-plastic) weights across both main and control subnetworks ($\mathbf{W}_i$, $\mathbf{b}_i$, $\mathbf{A}_i$, $\mathbf{W}_{\text{out}}$, $b_{\text{out}}$, and the fixed weights of control network $\mathcal{C}$), in addition to the learning rate parameter $\eta$ and decay term $\lambda$, are meta-learned via optimization. In neural networks, meta-learning is a common approach to finding fixed parameters that facilitate efficient learning (e.g.\ \cite{schmidhuber1987evolutionary, finn2017modelagnosticmetalearning,tyulmankov2022meta,NEURIPS2020_bdbd5ebf}). Note that these fixed parameters serve to implement inductive biases, prior knowledge, and other forms of prior structure -- innate or gained previously in a lifetime -- for a learning system; in this study, we focus on prior structure enabling ZI and CL, while noting that our approach (meta-learning in the AutoHebb architecture) can nonetheless be used to model other types of learning and adaptation, such as few-shot learning.

Typically, meta-learning is implemented as a nested loop with alternating "outer" and "inner" loop optimizations \cite{finn2017modelagnosticmetalearning,hospedales2020metalearningneuralnetworkssurvey,pmlr-v202-miconi23a}. With AutoHebb, the inner loop corresponds to 3-factor Hebbian learning rather than gradient descent; the outer loop could corresponds to any optimization algorithm, such as evolutionary algorithms (\cite{salimans2017evolutionstrategies}) or gradient descent (Figure \ref{fig:metalearning}, Algorithm \ref{alg:metalearning}).\ In this study, gradient descent with cross-entropy loss was used.

\begin{wrapfigure}{R}{0.5\textwidth}
\vspace{-2\intextsep}
\begin{minipage}{1.0\linewidth}
    \begin{algorithm}[H]
    \small
    \setstretch{1.0}
    \caption{Meta-learning in AutoHebb}
    \begin{algorithmic}[1]
    \Require Main network $f$, control network $\mathcal{C}$, and fixed parameters $\theta$, i.e.\ fixed network weights, fixed plasticity parameters, learning rate $\eta$, and decay $\lambda \in (0,1]$
    \Require Outer loop learning rate $\alpha$
    \Require $K$ meta-learning episodes
    \For{$k$ in $K$}
        \State $\mathbf{P} \leftarrow \mathbf{0}$
        \State Sample sequence of trials $\{(\mathbf{x}_t, \mathbf{y}_t)\}_{t=1}^T$
        \For{$t$ in $T$}
            \State $\mathbf{\hat{y}}_t, \mathbf{h}_L \leftarrow f_{\theta}(\mathbf{x}_t)$
            \State $\bm{\phi} \leftarrow \text{Feedback}(\mathbf{\hat{y}}_t)$
            \State $m \leftarrow \mathcal{C}_{\theta}(\mathbf{h}_L, \bm{\phi})$
            \State $\mathbf{P} \leftarrow \text{3FactorUpdate}(\mathbf{P}, m, \eta, \lambda)$
        \EndFor
        \State $\theta \leftarrow \theta - \displaystyle\frac{\alpha}{T}\sum_{t}\nabla_{\theta}\mathcal{L}(\mathbf{y}_t, \mathbf{\hat{y}}_t)$
    \EndFor
    \end{algorithmic}
    \label{alg:metalearning}
    \end{algorithm}
\end{minipage}
\vspace{-0.5\intextsep}
\end{wrapfigure}

\begin{figure}[!t]
    \centering
    \includegraphics[width=0.6\linewidth]{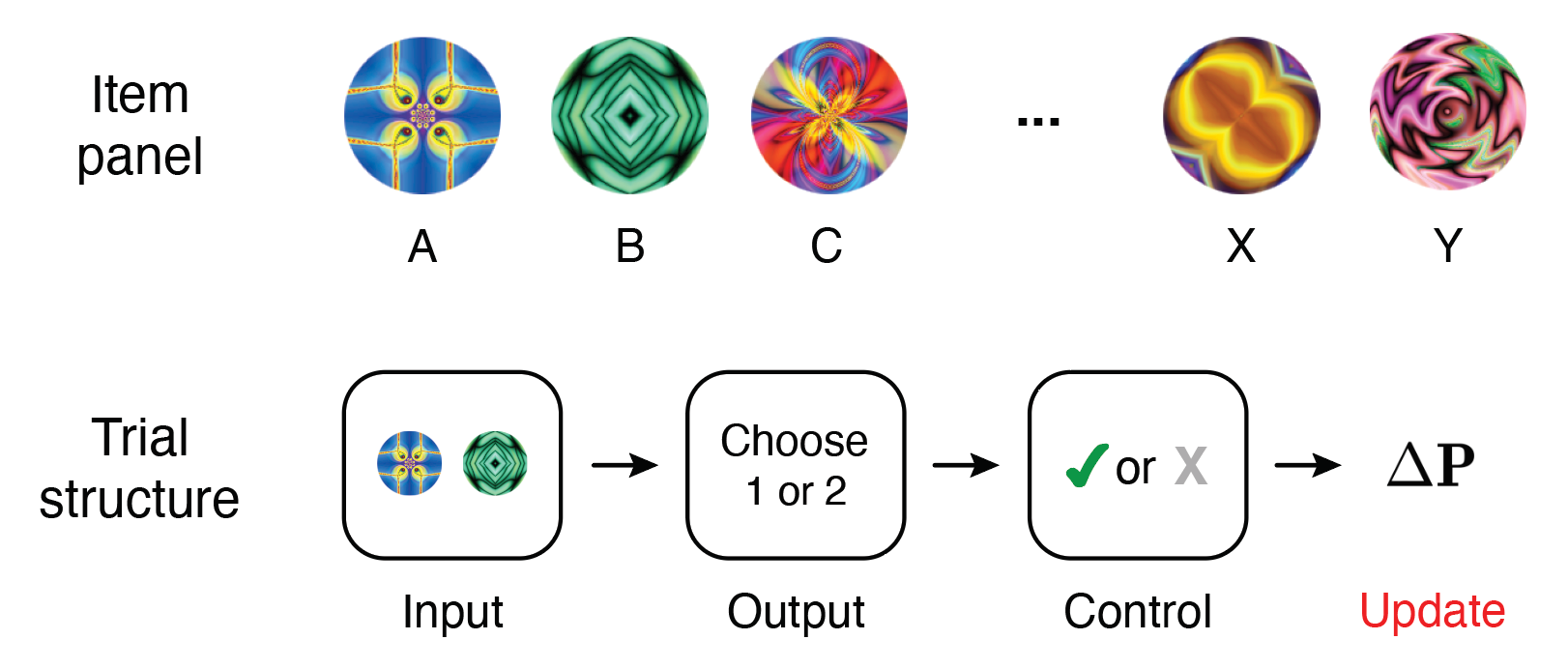}
    \caption{ \textbf{Task items and trial structure.} In the present study, a given task involves a set of items (item panel) that correspond to arbitrary sensory stimuli, such as fractal images. Trials consist of four steps: (1) presentation of two items (input), (2) generation of a binary output (e.g. choose item 1 vs. 2) by the model or participant (output), (3) generation of a learning control signal by the model or participant upon receiving external feedback from the environment (e.g. cash reward) (control), and (4) a learning or plasticity update (update), as governed by the control signal.}
    \label{fig:zi_task_schematic}
\end{figure}

\section{Tasks}\label{sec:tasks}

\subsection{Transitive inference (0-shot ranking)}\label{subsec:transitiveinference}

A classic 0-shot ability seen in humans and a wide range of animal species is transitive inference (TI), the ability to infer serial order (ranking) without explicit instruction \cite{bryant1971transitive,mcgonigle1977monkeys,jensen2017serial,jensen2019discovering,gazes2021does}:\ informally, "if $\text{A} > \text{B}$ and $\text{B} > \text{C}$, then $\text{A} > \text{C}$." The concept of serial order is foundational to numerous additional concepts that support inference and generalization, including space, time, magnitude, and number \cite{gazes2023thinking}. 
As a task, TI is typically tested by presenting subjects with a series of trials in which two items are presented per trial. Items are arbitrary stimuli (e.g.\ fractal images) and belong to an underlying ranking (items designated as A, B, C, etc., where A is highest rank)\ not shown to the participant, who receives reward if they choose the higher-ranked item in each trial. Trials are presented in two phases: training, consisting of only adjacent item pairs (e.g.\ A vs.\ B, B vs.\ C); and testing, consisting of all item pairs, including non-adjacent items (e.g.\ A vs.\ C, B vs.\ E). After each choice, external feedback is provided to the subject (e.g. reward vs.\ lack of reward). Note that non-adjacent item pairs constitute novel combinations of familiar elements (here items), making TI an instance of ZI (Figure \ref{fig:ti_task}A).

To model TI, we adopt a task format similar to that in past work \cite{miconi2025neural}. In each meta-learning episode (hereafter episode), we randomly generate a panel of $N$ items, where there exists an underlying ranking across items (A, B, C, etc.), and each item $I$ is presented to the network as an input vector $\mathbf{x}_I$ of dimension $d$ sampled uniformly from $\{-1,+1\}^{d}$. We set $N = 8$ and $d=100$ across experiments unless stated otherwise.

For each trial (schematic in Figure \ref{fig:zi_task_schematic}), a pair items $\mathbf{x}_I$ and $\mathbf{x}_J$ is randomly sampled. The correct answer (`ground-truth' label) is determined by the items' relative rankings: if item $I$ has higher ranking than item $J$, the label $y$ is $0$ (corresponding to "choose item 1"), and otherwise is $1$ ("choose item 2"). A trial input is the concatenation of the two items input vectors $\mathbf{x}_{IJ} = \begin{bmatrix}\mathbf{x}_I^\top & \mathbf{x}_J^\top\end{bmatrix}.$

In each episode, trials are divided into training vs.\ testing phases, mirroring classic TI task design \cite{mcgonigle1977monkeys,jensen2017serial,jensen2019discovering}. For the training phase, only adjacent item pairs are presented; for the testing phase, both adjacent and non-adjacent item pairs are presented. During meta-learning, plastic weights are updated across both phases. However, reported $0$-shot evaluations in Section \ref{sec:results} are done on networks for which plastic weights have been frozen at the end of the training phase, ensuring that measured performance on non-adjacent item pairs in the testing phase derives from only experience of adjacent (rather than non-adjacent) item pairs in the training phase. See Appendix \ref{appendixsec:autohebbtraining} for further details on training.

\begin{figure}[t]
    \centering
    \includegraphics[width=1\linewidth]{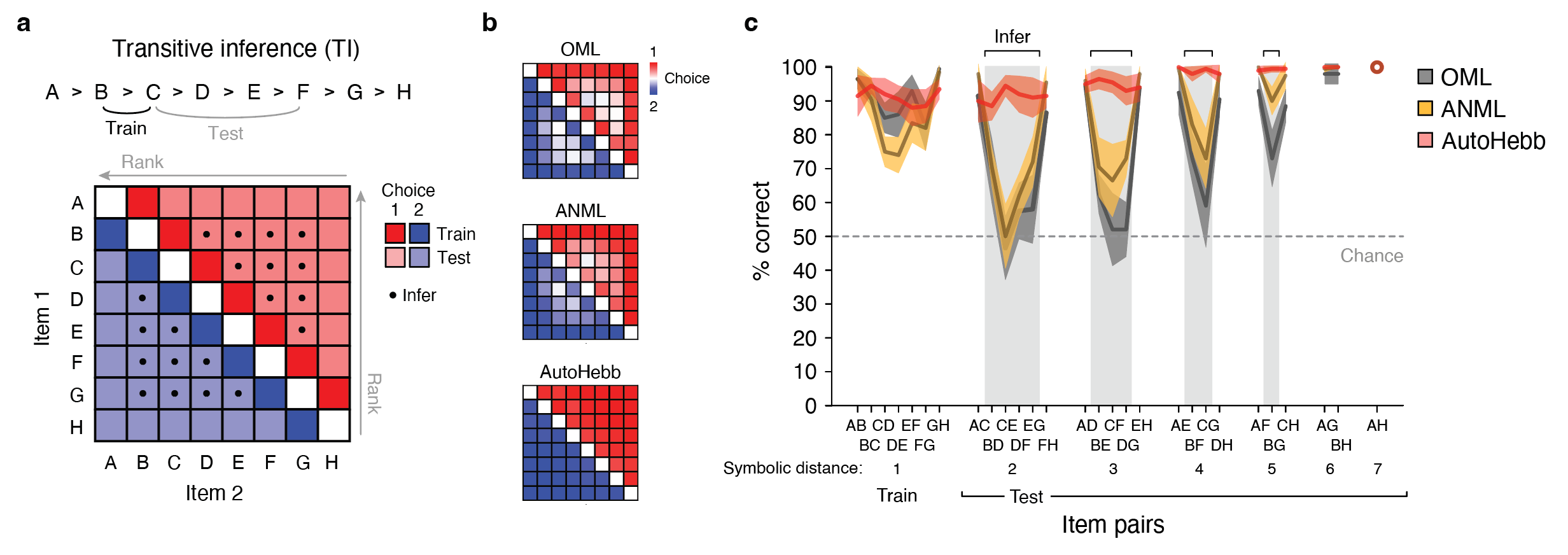}
    \caption{\textbf{Transitive inference (TI) task and performance of representative models. }TI tests 0-shot inference (ZI) of ranking. \textbf{a}, Task schematic. Items (A to H) belong to an unseen underlying ranking, and a pair of items are presented in each trial. In training trials, only adjacent items are presented; in test trials, non-adjacent items are presented; inferential (infer) trials are test trials that do not contain the first or last item. \textbf{b}, Average model output after training trials. \textbf{c}, Model performance by trial type. Trial types (item pairs) are grouped based on the difference in rank between items ("symbolic distance" \cite{jensen2017serial,jensen2019discovering}). Solid lines and shaded regions represent the mean and IQR, respectively. Example outputs across all models in Figures \ref{fig:ti_supp_symb} and \ref{fig:ti_supp_grid}. Model performance across all models reported in Table \ref{table:zeroshot}.}
    \label{fig:ti_task}
\end{figure}

\subsection{Transitive equivalence (0-shot grouping)}\label{subsec:transitiveequivalence}
A second 0-shot ability reported in both humans and animals is transitive equivalence (TE), or the ability to infer groupings (equivalence classes) without explicit instruction \cite{zeithamova2012thehippocampus,lippl2025when}: informally, "if $\text{A} = \text{B}$ and $\text{B} = \text{C}$, then $\text{A} = \text{C}$." As with ranking (see above), grouping is a concept that enables ZI, yet inference of grouping requires a fundamentally different kind of representation \cite{lippl2025when}. Tasks testing TE take various forms (e.g. acquired equivalence and associative inference tasks \cite{bunsey1996conservation,eichenbaum1996functional,shohamy2008integrating,schlichting2015learning,zeithamova2012thehippocampus}), with all such tasks entailing judging equivalence for novel combinations of items (e.g. "A = C?" where A and C have never been seen together before), making TE an instance of ZI. 

Here we adopt the two-item TE task format of \cite{lippl2025when} (schematized in Figure \ref{fig:te_task}a), which retains essential underlying 0-shot structure (equivalence classes or "group") and parallels the classic two-item format TI task. The trial sequence in each episode (i.e. training then testing phase; random sampling of items within each phase) follows that of TI (Section \ref{subsec:transitiveinference}), though here an item pair has label $1$ if the items are equivalent (same group) and $0$ otherwise (different group). In each episode, $M$ items are randomly generated for each of $Q$ groups. We assign indexing $i=1,\ldots,M$ to refer to items within each class. In the training phase, trials consist of item pairs $\mathbf{x}_i$ and $\mathbf{x}_j$ where items are sampled from any of the $Q$ groups, and the item pairs are subject to the condition that $|i-j| < 2$ (i.e. items are "adjacent" pairs, within- or across- group). In the testing phase, all possible item pairs are presented and the network must 0-shot infer whether items belong to the same or different group when $|i-j| \geq 2$. Unless stated otherwise, we set $Q=3$ and $M=3$; we investigate the ability of models to perform larger problems further below.

\begin{figure}[t]
    \centering
    \includegraphics[width=1\linewidth]{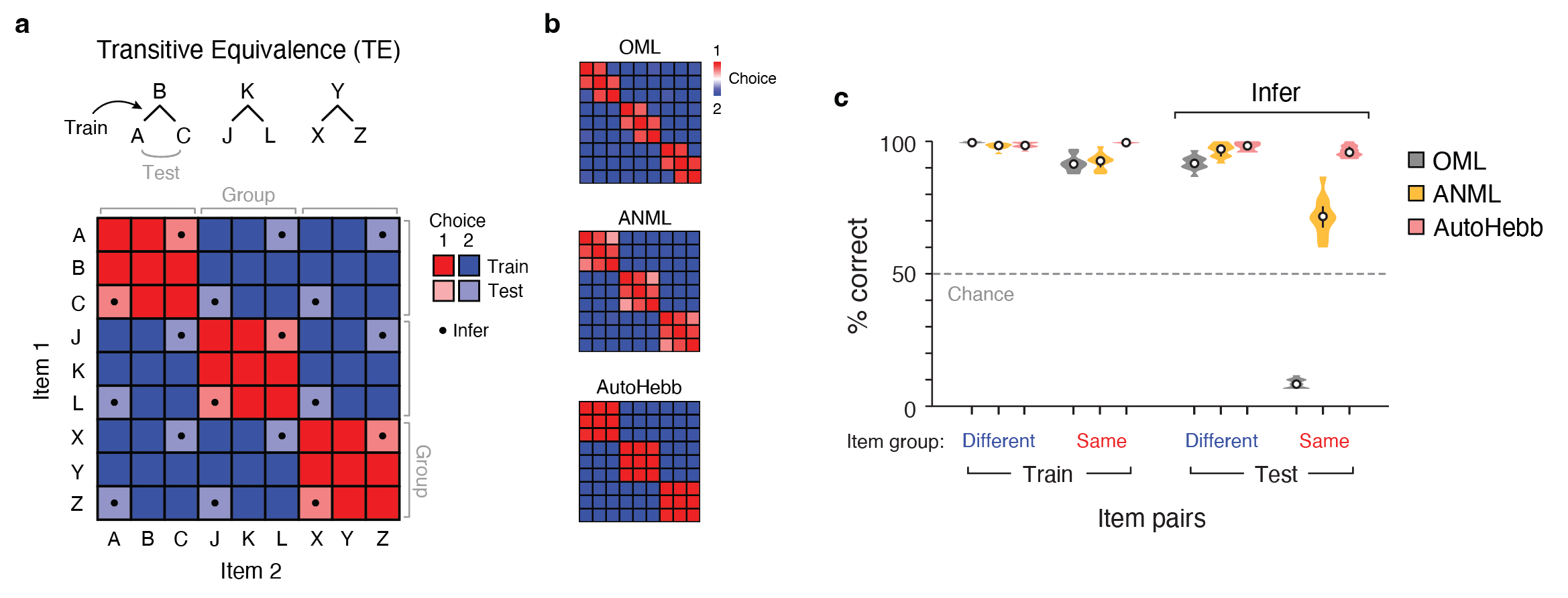}
    \caption{\textbf{Transitive equivalence (TE) task and performance of representative models.} TE tests 0-shot inference (ZI) of grouping. \textbf{a}, Task schematic. Items (A, B, C, etc.) belong to multiple unseen underlying groups (e.g.\ ABC and JKL), and pairs of items are presented in each trial. In training trials, particular item pairs are omitted both within- (e.g. AC and CA) and across- (AL and LA) group: these item pairs are testing trials, here equivalent to inferential (infer) trials. \textbf{b}, Average model output after training trials. \textbf{c}, Model performance across trial types. Trial types (item pairs) are separated based on same group (choice 1) or different group (choice 2). Example outputs across all models in Figure \ref{fig:te_supp_grid}. Model performance across all models reported in Table \ref{table:zeroshot}.}
    \label{fig:te_task}
\end{figure}

\subsection{Continual learning (CL)}\label{subsec:methodscontinual}
We propose three non-standard benchmark CL tasks derived from TI and/or TE: two test the underlying concepts (ranking and grouping) under more demanding conditions; the third tests multi-task ability. Since TI and TE test ZI, these three proposed tasks assess ZI and CL jointly.

TI linking (TI-L), also termed "list-linking" \cite{treichler1996concurrent,treichler2003linking,nelli2023neural,miconi2025neural}, tests the ability to globally order separate rankings (lists) when presented with minimal new information: informally, "if $\text{A} > \text{B} > \text C$ and $\text{D} > \text{E}  > \text{F}$, then upon learning $\text{C} > \text{D}$,  infer $\text{A} > \text{B} > \text C > \text{D} > \text{E} > \text{F}$." The ability underlying TI-L, which has been termed "rapid knowledge assembly" \cite{nelli2023neural}, is a unique test of CL transfer as the transfer is based on a single concept (ranking) and relies on extremely limited information (e.g. solely "$\text{C} > \text{D}$"). The trial sequence for TI-L differs from TI (in Section \ref{subsec:transitiveinference}) only the order of training trials:\ adjacent item pairs for items ranked $1$ to $N/2$ (`list 1') are shown first, then adjacent pairs for items ranked $(N/2) + 1$ to $N$ (`list 2'), and lastly a `linking pair' between items $N/2$ and $(N/2)+1$.

TE linking (TE-L), which we introduce here, tests the ability to update unseen equivalence class (grouping) relationships given minimal new information: informally, "if $\text{A} = \text{B} = \text{C}$ and $\text{X} = \text{Y} = \text{Z}$, then upon learning $\text{D} = \text{C}$,  infer that $\text{A} = \text{B} = \text{C} = \text{D}$ and $\text{D} \neq \text{X} = \text{Y} = \text{Z}$." Like TI-L, TE-L tests CL transfer based on a single concept (grouping) and likewise relies on extremely limited information (e.g.\ solely "$\text{D} = \text{C}$"). The TE-L training phase is exactly the same as in TE (Section \ref{subsec:transitiveequivalence}) but introduces a single novel item ($\mathbf{x}'$) at the end of TE training, namely doing so by presenting two linking-pair trials $(\mathbf{x}, \mathbf{x}')$ and $(\mathbf{x}', \mathbf{x})$, where $\mathbf{x}$ is a familiar item. In the testing phase, the model must 0-shot infer that $\mathbf{x}'$ inherits $\mathbf{x}$'s equivalence relationships.

Multi-task 0-shot continual learning, also introduced here, tests the ability to acquire and retain knowledge across multiple ZI tasks. We specifically test the ability to learn TI then TE (TI-TE) and vice versa (TE-TI), testing ZI responses only after both tasks are presented. Critically, we omit task-identity cues, a more stringent setting than task-cued (also called "task-aware" or "task-incremental" \cite{van2022three}) settings common in studies of CL (e.g.\ \cite{kirkpatrick2017overcomingcatastrophic,zenke2017continuallearning,lopez2017gradient}). Within an episode, two abilities are tested: (i) inferring the currently active ZI task without task-identity cues, and (ii) retaining knowledge from tasks presented earlier. To our knowledge, each of these abilities -- particularly in the context of ZI -- has received little attention in past work. We provide further details in reporting results in \ref{subsec:continuallearningresults}.

\section{Results}\label{sec:results}

We tested AutoHebb on the ZI and CL tasks described on Section \ref{sec:tasks}, comparing performance to that of leading models in meta-learning and CL. In choosing models to test we had two aims, namely (i) to test established and competitive modern learning models against AutoHebb, and (ii) to test modern learning models that are systematically comparable to AutoHebb in model components, and in so doing provide the equivalent of ablation analysis of AutoHebb. A more complete rationale is provided in Appendix \ref{appendixsec:experimentdetails}. Briefly, the models spanned three categories: meta-learning models (MAML (\cite{finn2017modelagnosticmetalearning},  Reptile \cite{nichol2018firstordermetalearning}, OML \cite{javed2019metalearning}, MER \cite{riemer2018learning}, ANML \cite{beaulieu2020learning}, FLP \cite{lindsey2020learningtolearn}), 2-factor Hebbian plastic networks (\cite{miconi18differentiableplasticity,najarro2020meta}), denoted 2-factor-A (2F-A) and 2-factor-ABCD (2F-ABCD), and explicit continual learning strategies (weight regularization methods EWC \cite{kirkpatrick2017overcomingcatastrophic} and SI \cite{zenke2017continuallearning}; and replay-based methods DER \cite{buzzega2020darkexperience} and A-GEM \cite{chaudhry2018efficient}). The two specific AutoHebb architectures we tested -- termed 1- and 2-neuromodulator AutoHebb (AH-1NM and AH-2NM), the latter a model variant inspired by multiple neuromodulator systems in the brain -- are diagrammed in Figure \ref{appendixfig:autohebb_two_arch}. Complete experimental details are in Appendix \ref{appendixsec:experimentdetails}. Benchmarking results take the mean of the average accuracy over 6 meta-training runs (model instances) with different random parameter initializations and report two standard errors from the mean; in figures showing example outputs (e.g. Figures \ref{fig:ti_task}, \ref{fig:te_task}, \ref{fig:linking_tasks}), average accuracy over 100 episodes for a single meta-trained model is reported, and is indicated accordingly. 

\subsection{Zero-shot inference (ZI)}\label{subsec:zeroshotresults}

We tested AutoHebb and leading learning models on ZI tasks, followed by algorithmic analyses of AutoHebb networks. The testing results disclosed unexpected strengths and weaknesses of leading models, in addition to AutoHebb's exceptional capacity for ZI. Further details of task setup, benchmarks, architectures, and training are in Appendices \ref{appendixsec:autohebbtraining} and \ref{appendixsec:experimentdetails}.

\paragraph{ZI in the TI task}

We meta-trained each model using the episode structure outlined in Section \ref{subsec:transitiveinference}. Each meta-training episode randomly samples a set of $8$ items (input representations) and a sequence of $20$ training and $10$ test trials (or separate batches in the cases of MAML and Reptile \cite{nichol2018firstordermetalearning}). The test trials contain both adjacent and non-adjacent pairs of items from the underlying ranking. After meta-training, all parameters are frozen, and the outputs of the models for each trial type (item pair) are collected.

To evaluate ZI ability, we benchmarked model performance on `critical' item pairs: non-adjacent pairs that do not contain the first or last items in the ranking (dotted item pairs in Figure \ref{fig:ti_task}a), following standard practice in studies of TI \cite{mcgonigle1977monkeys,jensen2017serial,jensen2019discovering,ramawat2023transitive}. Model outputs and performance of leading models are shown in Figures \ref{fig:ti_task}b and \ref{fig:ti_task}c \cite{beaulieu2020learning}. These results indicated that AutoHebb networks achieve substantially better performance than established models. Interestingly, FLP and 2F-ABCD performed comparably to AH-1NM, and a number of additional models showed above-chance performance, though AH-2NM outperformed all models (Table \ref{table:zeroshot}). Additional example model outputs are in Figures \ref{fig:ti_supp_symb} and \ref{fig:ti_supp_grid}.

\paragraph{Reverse-engineering learning/inference algorithms from AutoHebb networks}
A primary motivation for AutoHebb is the algorithmic tractability potentially afforded by virtue of solely feedforward architecture. Given this possibility, we adopted a procedure (outlined in Appendix \ref{appendixsec:autohebbjustification}) by which we sought to reverse-engineer learning/inference algorithms from meta-learned AutoHebb networks that successfully performed a given task. Importantly, we sought to do so while retaining understanding of how algorithmic variables and operations related to specific neural components (e.g.\ activities, weights, neuromodulators, plasticity rules, non-linearities, etc.; \cite{tyulmankov2022meta, miconi2025neural, lippl2024amathematicaltheory, jian2025discovering}).

In Appendix \ref{appendixsec:neuralmechanisms}, we provide algorithms implemented by AutoHebb for TI and TI-L. In order to find networks with the fewest possible components (e.g. hidden layers, plastic weights, nonlinearities, etc.), we meta-trained a collection of AutoHebb networks with components progressively removed and observed their task performance. For TI, this led to an architecture that consists of only a fixed and plastic readout (Eq. \ref{eq:minimaltinetwork}). In Appendix \ref{appendixsubsec:neuralmechti}, we find that this minimal TI network learns a scalar rank representation for each item and applies a subtraction mechanism between the two items' scalar rank representations. During learning, only the ranks of items presented in a trial are updated. Algorithmic analysis of AutoHebb networks performing TI-L are presented further below.

\begin{table}
  \caption{\textbf{Model performance on two ZI tasks: TI and transitive equivalence TE tasks.} Best scores are \textbf{bolded}. Both AutoHebb models outperform all other methods on TI and TE. Among other models on TI, MAML, Reptile, ANML, FLP, and 2F-ABCD methods show viable performance on TI, whereas MAML, OML, and 2F-A perform relatively poorly. On TE, apart from Autohebb, ANML, 2F-A, and 2F-ABCD show strong performance, while all other models do not perform above chance. Trial type-based breakdowns of TE performance are reported in Table \ref{table:zeroshotbreakdown}.\\}
  \label{table:zeroshot}
  \centering
  \setlength{\tabcolsep}{5pt}
  \begin{tabular}{lcc}
    \toprule
    Model & TI & TE \\
    \midrule
    MAML \cite{finn2017modelagnosticmetalearning} & $0.697 \pm 0.025$ & $0.502 \pm 0.010$ \\
    Reptile \cite{nichol2018firstordermetalearning} & $0.744 \pm 0.012$ & $0.509 \pm 0.010$ \\
    MER \cite{riemer2018learning} & $0.500 \pm 0.009$ & $0.495 \pm 0.002$ \\
    OML \cite{javed2019metalearning} & $0.603 \pm 0.015$ & $0.503 \pm 0.003$ \\
    FLP \cite{lindsey2020learningtolearn} & $0.845 \pm 0.011$ & $0.504 \pm 0.002$ \\
    ANML \cite{beaulieu2020learning} & $0.702 \pm 0.011$ & $0.831 \pm 0.005$ \\
    2F-A & $0.765 \pm 0.040$ & $0.717 \pm 0.050$ \\
    2F-ABCD & $0.878 \pm 0.022$ & $0.844 \pm 0.015$ \\
    \midrule
    AH-1NM & $0.904 \pm 0.053$ & $\mathbf{0.945 \pm 0.042}$ \\
    AH-2NM & $\mathbf{0.986 \pm 0.006}$ & $0.933 \pm 0.014$ \\
    \bottomrule
  \end{tabular}
\end{table}

\paragraph{ZI on the TE task}

TE provides a simple paradigm for investigating ZI of grouping:\ a concept fundamental to numerous concepts and cognitive domains \cite{sidman1994equivalence,zentall2008concept}. Recent theoretical work \cite{lippl2024amathematicaltheory} shows that kernel models can perform TI but not TE, highlighting the qualitative difference in inductive bias (and representation) at stake between these tasks. We evaluated models on 3 groups of 3 items, corresponding to 81 item pairs (63 training, 18 test) (Figure \ref{fig:te_task}a). For each episode during meta-training, we presented all 63 training pairs and then 30 test pairs (randomly sampled) for each episode during meta-training, adopting the same stopping criteria as TI meta-training. Results are mean accuracy across 100 episodes per model, averaged (± 2 s.e.) across six model instances.

Table \ref{table:zeroshot} reports mean model performance, equally weighted for the two categories of item pairs:\ equivalent (`same-group') pairs and the non-equivalent (`different-group') pairs. Breakdowns of performance based on these categories are in Table \ref{table:zeroshotbreakdown}, with further details in Appendix \ref{appendixsec:experimentdetails}. Notably, prior to any learning, all (meta-trained) models defaulted to outputting `different-group' (for all item pairs except for same-item pairs), likely due to the inherent imbalance of these trial type categories in the TE task. Different-group accuracy thus saturates for many models (Table \ref{table:zeroshotbreakdown}), making same-group item pairs the critical test of ZI for this task. Strikingly, among all models, only AutoHebb, ANML, 2F-A, and 2F-ABCD performed above chance (Figure \ref{fig:te_task}b-c; see also example outputs in Figure \ref{fig:te_supp_grid}), with AutoHebb networks moreover outperforming all models.

\paragraph{Emergent (untrained) learning and generalization abilities in AutoHebb}
One possibility afforded by meta-learning is the possibility of finding "meta-test" abilities: learning or generalization abilities that go beyond what was explicitly optimized in meta-learning ("meta-trained"). Such abilities, if expressed, provide evidence that meta-learning conferred an inductive bias (here implemented as fixed weights in a network) that better implements the learning/generalization ability in question. For the present tasks (TI and TE), we were especially interested in this possibility given that the concepts/relations they probe (ranking and grouping, respectively) are in principle applicable to an indefinitely large number of items and group sizes.

Strikingly, we found that AutoHebb generalized to larger numbers of items and group sizes (Figures \ref{supplementfig:length_gen14}, \ref{supplementfig:length_gen20}, \ref{supplementfig:length_gen_overall}, \ref{supplementfig:te_item_length_gen} \ref{supplementfig:te_group_length_gen}) in both TI and TE (length generalization; in TI, meta-learning performed for 8 items, with AutoHebb showing generalization to 20 items (or more); in TE, meta-learning performed for 3 items and 3 groups, with AutoHebb showing generalization to 5 items and 6 groups of 3 items; see Appendix \ref{appendixsec:autohebbextraresults} for further results); further, two representative models (OML and ANML) showed relatively weaker or no length generalization.

Apart from length generalization, we were also interested in probing the ability of models to self-regulate learning in conditions not seen during meta-learning, as such abilities are naturally relevant to CL. In particular, previous work on TI shows that repeated exposure of traditional learning models to a single trial type (e.g.\ item pair EF) can severely interfere with subsequent correct responses, whereas animals performing TI are robust to this manipulation \cite{jensen2019discovering,jensen2017transitive}. Performing the trial exposure procedure (solely in meta-testing), we found that AutoHebb similarly showed robustness in subsequent responses, in contrast to two representative models tested (OML and ANML, Figure \ref{supplementfig:mass_symb}). Notably, direct examination of the neuromodulator in AutoHebb networks showed that networks progressively down-regulated neuromodulation during massed trial presentation (Figure \ref{supplementfig:mass_neuromod}).

Lastly, for the TI-L and TE-L tasks, we adopted a training-testing strategy such that performance of these tasks constitutes evidence of an additional emergent ability that has been termed "rapid knowledge assembly" \cite{nelli2023neural,miconi2025neural}. We provide further details in presenting results from these two tasks further below.

\subsection{Continual learning (CL)}\label{subsec:continuallearningresults}
We next tested AutoHebb on three non-traditional CL tasks: TI-L, TE-L, and multi-task 0-shot learning (outlined in Section \ref{subsec:methodscontinual}); further, as with the ZI tasks (Section \ref{subsec:zeroshotresults}), we also tested established learning models on these tasks. Surprisingly, we found that, besides AutoHebb, no model uniformly succeeded on all three CL tasks. However, on a subset of tasks, ANML, FLP, 2F-A, and 2F-ABCD showed viable performance.

\begin{table}
  \caption{\textbf{Model performance on two CL tasks:\ TI-L and TE-L.} Best scores are \textbf{bolded}. TE-L Original corresponds to ZI performance on the originally presented groups after the networks are shown the new item. TE-L New Item is the ZI performance on all test trials that have the new item as one input. TE-L Total is the equal-weighted average of the two scores. On TI-L, both AutoHebb models (AH-1NM and AH-2NM) outperform all other models. Both 2-factor models (2F-A, 2F-ABCD) and FLP also show successful linking. On TE-L, both AutoHebb models outperform on both the original TE task (TE-L Original) after being presented the linking trial, as well as the linking task (TE-L New Item), with AH-1NM slightly outperforming AH-2NM. Among other models, ANML shows the strongest linking ability. Trial type-based breakdowns of TE performance are reported in Table \ref{table:linkingtebreakdown}.\\}
  \label{table:zeroshotlinking}
  \centering
  \setlength{\tabcolsep}{5pt}
  \begin{tabular}{lc|ccc}
    \toprule
    Model & TI-L & TE-L Original & TE-L New Item & TE-L Total \\
    \midrule
    MAML \cite{finn2017modelagnosticmetalearning} & $0.621 \pm 0.035$ & $0.498 \pm 0.006$ & $0.501 \pm 0.005$ & $0.500 \pm 0.005$ \\
    Reptile \cite{nichol2018firstordermetalearning} & $0.726 \pm 0.008$ & $0.510 \pm 0.009$ & $0.500 \pm 0.003$ & $0.505 \pm 0.005$ \\
    MER \cite{riemer2018learning} & $0.504 \pm 0.011$ & $0.499 \pm 0.008$ & $0.493 \pm 0.014$ & $0.496 \pm 0.004$ \\
    OML \cite{javed2019metalearning} & $0.663 \pm 0.015$ & $0.501 \pm 0.017$ & $0.509 \pm 0.012$ & $0.505 \pm 0.012$ \\
    FLP \cite{lindsey2020learningtolearn} & $0.761 \pm 0.020$ & $0.500 \pm 0.009$ & $0.497 \pm 0.014$ & $0.499 \pm 0.004$ \\
    ANML \cite{beaulieu2020learning} & $0.601 \pm 0.017$ & $0.816 \pm 0.012$ & $0.736 \pm 0.013$ & $0.776 \pm 0.011$ \\
    2F-A & $0.787 \pm 0.072$ & $0.699 \pm 0.041$ & $0.542 \pm 0.010$ & $0.621 \pm 0.025$ \\
    2F-ABCD & $0.838 \pm 0.041$ & $0.798 \pm 0.010$ & $0.571 \pm 0.008$ & $0.684 \pm 0.006$ \\
    \midrule
    AH-1NM & $0.918 \pm 0.029$ & $\mathbf{0.922 \pm 0.037}$ & $\mathbf{0.862 \pm 0.042}$ & $\mathbf{0.892 \pm 0.040}$ \\
    AH-2NM & $\mathbf{0.976 \pm 0.009}$ & $0.907 \pm 0.017$ & $0.837 \pm 0.011$ & $0.872 \pm 0.013$ \\
    \bottomrule
  \end{tabular}
\end{table}

\paragraph{CL on TI-L task} In evaluating TI-L, we were expressly interested in the possibility that models meta-trained (or hyperparameter-tuned) only on TI would emergently show an ability to perform TI-L. We therefore took the models meta-trained on TI (but not TI-L) and simply evaluated their performance on TI-L, i.e. keeping all meta-trained parameters fixed and presenting the networks with new episodes. In these episodes, the training phase followed that in Section \ref{subsec:methodscontinual} (under TI-L), in particular where $8$ items were generated, then $4$ items were assigned to two separate item rankings (lists). We showed $10$ training trials from list 1, then $10$ training trials from list 2, and finally $4$ linking trials DE/ED (Figure \ref{fig:linking_tasks}a). TI-L performance was measured on critical (inferential) item pairs: across-list item pairs excluding end items (dotted item pairs in Figure \ref{fig:linking_tasks}a). 

We found that nearly all models performed poorly and/or inconsistently on these inferential item pairs (see Figure \ref{fig:linking_tasks}b for example outputs from OML and ANML; see also example outputs for all models in Figures \ref{fig:ti_supp_symb} and \ref{fig:ti_supp_grid}); interestingly, the exceptions were models incorporating Hebbian plasticity, namely FLP, 2F-A, and 2F-ABCD, which showed viable performance both on average (Table \ref{table:zeroshotlinking}) and across all trial types (Figures \ref{fig:ti_supp_symb} and \ref{fig:ti_supp_grid}). AutoHebb outperformed all models, with two neuromodulators (AH-2NM) providing marginal improvement over one (AH-1NM) (Table \ref{table:zeroshotlinking}).

In algorithmic analyses of AutoHebb networks, we found that two sets of plastic weights appeared necessary to perform TI-L (Eq. \ref{eq:tilarch}), unlike the minimal TI network. In particular, the simplest AutoHebb network that could perform TI-L had one hidden layer with plastic weights both in the hidden and readout layers, without any nonlinearity. Plastic weights in the hidden layer stored information linking currently shown items to items that had been (in past trials) co-presented with the currently shown items. In contrast, the plastic weights in the readout encode past network responses associated with each item (via prior association with a specific presented item pair; e.g.\ item A associated with an AB trial). Given an input, the network 'recalls' associated items not currently shown in the trial and uses the trial outcomes associated with all recalled items to compute a scalar rank representation. In contrast, the minimal TI network produces a scalar rank representation for a given item based only on the outcome of trials with (directly) co-presented items. See Appendix \ref{appendixsubsec:neuralmechtil} and Algorithm \ref{algorithm:til} for complete analysis and discussion.

\begin{figure}[t]
    \centering
    \includegraphics[width=1\linewidth]{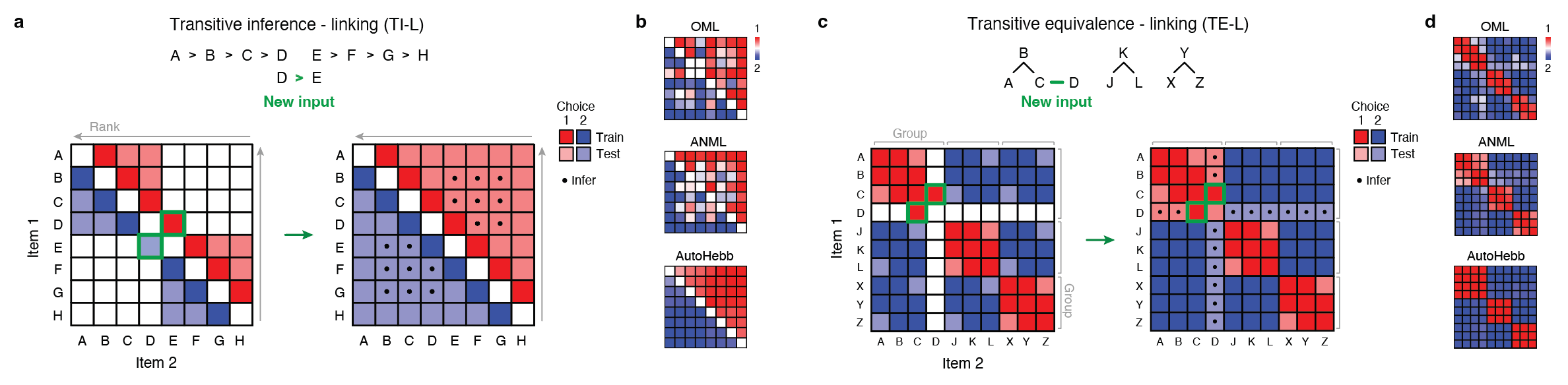}
    \caption{\textbf{Two continual learning (CL) tasks and performance of representative models.} \textbf{a }and \textbf{c}, Task schematics. Each task (TI-L and TE-L) is based on the ZI tasks TI and TE, respectively, and tests 0-shot inference following presentation of a new item pair (new input; diagramed in green). In TI-L, also known as "list-linking" \cite{treichler1996concurrent,nelli2023neural,miconi2025neural}, the item pair is a "linking pair" (here DE or ED) that connects separate previously learned rankings (here A to D and E to H). In TE-L, the new input is an item pair (here CD/DC) containing a novel item (here D) and previously presented item (here C). Note that both tasks test the ability to perform structure-sensitive transfer to items beyond that of the new item pair.  \textbf{b }and\textbf{ d}, Average model output (across all item pairs) following presentation of the new item pair. See Figures  \ref{fig:ti_supp_symb} and \ref{fig:ti_supp_grid} for example outputs across all models. Model performance across all models reported in Table \ref{table:zeroshotlinking}.}
    \label{fig:linking_tasks}
\end{figure}

\paragraph{CL on TE-L task} As in TI-L, with TE-L we were expressly interested in the possibility that models meta-trained (or hyperparameter-tuned) only on TE would emergently show an ability to perform TE-L. We therefore took the models meta-trained on TE and simply evaluated their performance on TE-L. The episode training phase followed that in Section \ref{subsec:methodscontinual}, where the original 63 TE training trials are presented, then 4 linking trials of the form CD/DC as shown in Figure \ref{fig:linking_tasks}c. TE-L performance was measured on critical (inferential) item pairs: pairs containing D and an item not already in the linking trial (dotted item pairs in Figure \ref{fig:linking_tasks}c). We found that AutoHebb achieved the highest performance across all models, both on inferential linking trials and on inferential trials for originally presented groups (example outputs in Figure \ref{fig:linking_tasks}d; performance across models in Tables \ref{table:zeroshotlinking} and \ref{table:linkingtebreakdown}). Notably, ANML was the only baseline that showed substantial above-chance performance.

\paragraph{Multi-task ZI learning}

We lastly tested AutoHebb and established models on a challenging multi-task paradigm. Given limited prior work on continual learning across single-concept 0-shot tasks, we examined models' ability to perform TI followed by TE (TI-TE) and vice versa (TE-TI), testing successful (0-shot) inference after both tasks were presented. To obtain models capable of this task, during meta-learning each episode consisted of presentation of TI then TE or \textit{vice versa} (Figure \ref{fig:multi_task_continual}a), where only training phases were shown for each task (as described in Sections \ref{subsec:methodscontinual} and \ref{subsec:zeroshotresults}); further, to avoid overfitting to task boundaries,  TI and TE episode lengths were randomized within $20$-$100$ and $63$-$120$ trials, respectively (Figure \ref{fig:multi_task_continual}a).

As an overall assessment across continual learning models, we plot in Figure \ref{fig:multi_task_continual}b model performance across random sequences of 10 tasks (each either TI or TE); performance here is the mean performance (specifically on 0-shot inference trials) across all tasks presented up to that point (task \#). Strikingly, this test indicated exceptional performance of AutoHebb. The performances of ANML and OML decayed more slowly, but asymptoted to a lower accuracy. Full experimental details are in Appendix \ref{appendixsubsec:multitaskswitchingfiguredetails}.

On 2-task tests (TI-TE and TE-TI; Table \ref{table:multitask}), AH-2NM was the only model successful on all tasks, while ANML was successful on all tasks except on TE ZI trials where items belonged to the same group (Table \ref{table:multitaskbreakdown}). The limited ability of models to perform these tasks suggests that it is a considerable challenge to perform different kinds of ZI when task-identity cues are absent; that said, AutoHebb shows exceptional multi-task ability among models, demonstrating a unique capacity to engage different relational inductive biases (ranking vs.\ grouping) depending on the input presented (i.e.\ the sequence of trials shown in testing, corresponding to uncued tasks), in addition to retaining ZI knowledge from past lifetime experience.

\begin{table}[t]
  \caption{\textbf{Model performance on multi-task ZI: TI$\rightarrow$TE and TE$\rightarrow$TI.} Best scores are \textbf{bolded}. EWC, SI, A-GEM, and DER are non-meta-learning methods. AutoHebb with two neuromodulators (AH-2NM) is the only model that can continually learn multiple ZI tasks. While ANML shows reasonable performance overall, Table \ref{table:multitaskbreakdown} it only performs chance on ZI trials when items belong to the same group. Of the meta-learning models, 2F-A, 2F-ABCD, and AH-1NM show good TI performance in the TI$\rightarrow$TE direction. Of the non-meta-learning methods, A-GEM shows strong performance on TI in the TI$\rightarrow$TE direction while all other tasks remain at chance performance. Trial type-based breakdowns of TE performance are reported in Table \ref{table:multitaskbreakdown}. $^*$Only 3 runs are reported. \\}
  \label{table:multitask}
  \centering
  \setlength{\tabcolsep}{5pt}
  \begin{tabular}{lcc|cc}
    \toprule
    & \multicolumn{2}{c}{TI $\rightarrow$ TE} & \multicolumn{2}{c}{TE $\rightarrow$ TI} \\
    Model & TI & TE & TI & TE \\
    \midrule
     MAML \cite{finn2017modelagnosticmetalearning} & $0.545 \pm 0.010$ & $0.518 \pm 0.003$ & $0.552 \pm 0.008$ & $0.517 \pm 0.009$ \\
    Reptile \cite{nichol2018firstordermetalearning} & $0.585 \pm 0.030$ & $0.517 \pm 0.006$ & $0.614 \pm 0.034$ & $0.525 \pm 0.008$ \\
    MER \cite{riemer2018learning} & $0.500 \pm 0.006$ & $0.495 \pm 0.005$ & $0.502 \pm 0.005$ & $0.503 \pm 0.010$ \\
    OML \cite{javed2019metalearning} & $0.502 \pm 0.007$ & $0.502 \pm 0.005$ & $0.513 \pm 0.010$ & $0.505 \pm 0.007$ \\
    FLP \cite{lindsey2020learningtolearn} & $0.498 \pm 0.006$ & $0.502 \pm 0.004$ & $0.490 \pm 0.009$ & $0.502 \pm 0.008$ \\
    ANML \cite{beaulieu2020learning} & $0.630 \pm 0.010$ & $0.700 \pm 0.012$ & $0.617 \pm 0.017$ & $0.693 \pm 0.012$ \\
    2F-A & $0.711 \pm 0.016$ & $0.526 \pm 0.005$ & $0.600 \pm 0.029$ & $0.543 \pm 0.014$ \\
    2F-ABCD & $0.716 \pm 0.026$ & $0.519 \pm 0.004$ & $0.663 \pm 0.035$ & $0.537 \pm 0.018$ \\
    \midrule
    EWC \cite{kirkpatrick2017overcomingcatastrophic} & $0.507 \pm 0.005$ & $0.499 \pm 0.005$ & $0.494 \pm 0.010$ & $0.501 \pm 0.008$ \\
    SI \cite{zenke2017continuallearning} & $0.489 \pm 0.007$ & $0.498 \pm 0.008$ & $0.495 \pm 0.009$ & $0.505 \pm 0.005$ \\
    A-GEM \cite{chaudhry2018efficient} & $0.874 \pm 0.008$ & $0.505 \pm 0.001$ & $0.469 \pm 0.007$ & $0.512 \pm 0.006$ \\
    DER \cite{buzzega2020darkexperience} & $0.489 \pm 0.009$ & $0.502 \pm 0.005$ & $0.506 \pm 0.006$ & $0.492 \pm 0.004$ \\
    \midrule
    AH-1NM$^*$ & $0.700 \pm 0.207$ & $0.567 \pm 0.061$ & $0.637 \pm 0.162$ & $0.548 \pm 0.042$ \\
    AH-2NM & $\mathbf{0.877 \pm 0.009}$ & $\mathbf{0.772 \pm 0.014}$ & $\mathbf{0.888 \pm 0.009}$ & $\mathbf{0.777 \pm 0.014}$ \\
    \bottomrule
  \end{tabular}
\end{table}

\section{Discussion}
We introduce AutoHebb, a neural network architecture that addresses three major limitations of current models of learning: (1) lack of autonomous control of learning, (2) lack of strong (or appropriate) inductive biases, and (3) limited simplicity and expressivity. AutoHebb is a variant of plastic neural network \cite{miconi2018differentiable,miconi2018backpropamine,najarro2020meta,tyulmankov2022meta} and is thus comprised of fixed and plastic weights; fixed weights are identified via meta-learning optimization. On a set of non-traditional zero-shot inference (ZI) and continual learning (CL) tasks, we found that AutoHebb not only consistently outperformed popular learning models, but also showed multiple emergent learning-generalization abilities (abilities not explicitly optimized in meta-learning) which other models either did not express or expressed less strongly -- these abilities include rapid knowledge assembly, length generalization, and self-regulation of plasticity. Lastly, beyond performance, the architectural simplicity of AutoHebb also enabled reverse-engineering the networks' algorithms, demonstrating the tractability of these models to analysis and explanation.

Importantly, our findings delineate the model components essential to AutoHebb's exceptional capacity for ZI/CL, doing so in two ways:\ by evaluating performance on various qualitatively different ZI/CL tasks (thereby probing different kinds of inductive biases and CL capacities), and by evaluating performance of models that share one or more components with AutoHebb (thereby controlling for the shared model components). In particular, it was notable that performance of the various models tested was generally poor, with the exception of three models -- ANML, which implements meta-learning and neuromodulation but lacks Hebbian plasticity; 2-factor Hebbian, which implements meta-learning and Hebbian plasticity but lacks neuromodulation; and FLP, which also implements meta-learning and Hebbian plasticity but lacks neuromodulation -- showed competitive performance on at least one (though not all) tasks (e.g. ANML on TE and TE-L; 2-factor Hebbian on TI, TI-L, and TE; FLP on TI and TI-L). This pattern of performance across these related models implies that the three main model components of AutoHebb -- i.e. neuromodulation (autonomous control of learning), Hebbian plasticity, and meta-learning -- are jointly essential for AutoHebb's consistently exceptional ZI/CL abilities, at least as evidenced in the present tasks.

\begin{figure}[t]
    \centering
    \includegraphics[width=1.0\linewidth]{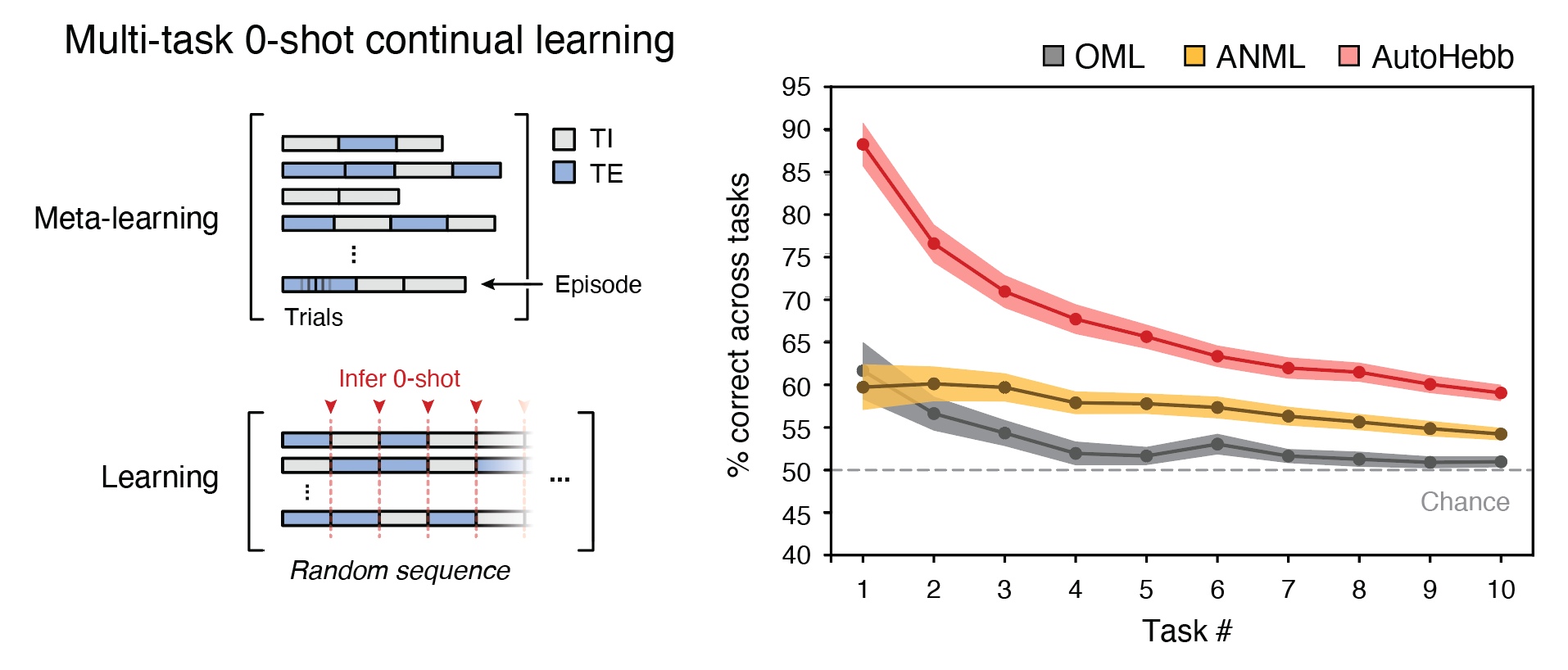}
    \caption{\textbf{Multi-task 0-shot continual learning:\ model setup, tasks, and performance.} Left, schematic of meta-learning and learning modeling procedure. Models were meta-trained on episodes in which 2 to 4 TI and TE tasks (i.e.\ having different randomly generated items) were serially presented and with variable numbers of trials. Models were then presented with a random sequence of 10 TI and TE tasks at evaluation. Right, performance of models at testing (mean ± 2 s.e., n=100 task sequences). Full details are in Appendix \ref{appendixsubsec:multitaskswitchingfiguredetails}.}
    \label{fig:multi_task_continual}
    \vspace{-\intextsep}
\end{figure}

In future work, we envision a variety of use cases and task settings where AutoHebb may contribute. In neuroscience, AutoHebb is of potential interest as a "fully neural" learning model (i.e. solely made of units, weights, local learning rules, and global signals like neuromodulators), providing a naturally more brain-like alternative to models relying on externally applied learning signals (e.g.\ researcher-provided gradients \cite{lillicrap2016random,Richards2019deeplearning,lillicrap2020backpropagation,yang2021artificial}); further, among meta-learned neural network models, AutoHebb provides a simpler alternative to those based on recurrence \cite{miconi2018backpropamine,miconi2025neural} or spikes \cite{NEURIPS2020_bdbd5ebf,confavreux2025balancing}, while also providing a more expressive alternative to models based on 2-factor Hebbian plasticity \cite{najarro2020meta,tyulmankov2022meta}; aside from potential enhancement in task performance, these differences may facilitate connections to experimental data (behavior and neural activity). In machine learning, AutoHebb is an alternative to gradient-based optimization in simulating within-lifetime learning (currently dominant), particularly in the setting of CL \cite{wang2024comprehensivesurveycontinuallearning}. Further, components of AutoHebb may be combined with existing models such as transformers for online learning or test-time adaptation through post-training methods (e.g. GRPO \cite{shao2024deepseekmathpushinglimitsmathematical}) either by strategically placing plastic weights within blocks and throughout layers (analogous to LoRA \cite{hu2022lora}) or by adding AutoHebb as an entirely separate module. Importantly -- whether in a basic or applied settings -- the combined simplicity and expressiveness of AutoHebb networks may be suited for algorithmic discovery, analysis, and interpretability, specifically in the context of brain-like constraints.

Lastly, we emphasize that the basic AutoHebb architecture (prior to meta-learning optimization) is fairly general-purpose, and thus applicable to many settings, including more standard learning tasks. Even so, our focus on single concept-based ZI tasks affords exceptionally unambiguous ways to test and analyze whether and how models realize transferable (abstract) structure, in contrast to tasks in which transferable structure across tasks is minimal (e.g. permuted MNIST \cite{goodfellow2016deeplearning,kirkpatrick2017overcomingcatastrophic,zenke2017continuallearning}) or multifarious and relatively ambiguous (e.g. split-CIFAR \cite{zenke2017continuallearning,lopez2017gradient}, split-MiniImageNet \cite{vinyals2016matching}, Continual World \cite{wolczyk2021continualworldroboticbenchmark}). In this way, the tasks we investigate in this study point to a non-traditional -- and possibly advantageous -- basis for future work seeking to identify new neurally grounded algorithms for learning, inference, and generalization.

\bibliographystyle{unsrt} 
\bibliography{references} %

\newpage

\appendix

\section{Reverse-engineering algorithms from AutoHebb networks }\label{appendixsec:autohebbjustification}
One of the major modeling aims potentially enabled by the AutoHebb architecture is to identify (or possibly discover) the essential algorithm implemented in performant AutoHebb networks. Further, doing so would also entail identifying the respective roles of the neural components of AutoHebb (fixed weights, plastic weights, main vs. control networks, neuromodulator, local learning rule, etc.), of which AutoHebb is fully comprised. Briefly, we consider the explicit identification of algorithms important for a number of reasons, including (a) facilitating or enabling improvement of algorithms in an applied setting, (b) anchoring scientific hypotheses about learning and generalization in the brain, as putatively realized in AutoHebb networks, to the hypotheses' algorithm-level descriptions, which would serve thus as the basis for deriving neural or behavioral predictions.

To identify algorithms in AutoHebb networks, our basic approach consisted of four steps: (1) identify the simplest AutoHebb architecture capable of adequately performing the task (which we term "minimal network"), doing so by re-training (meta-training) networks with progressively simpler architecture, (2) collect observations of a performant network's hidden activity, (3) identify components of the minimal network which can be idealized without alteration of task performance (or other behavioral patterns), (4) analyze the idealized network (under various task conditions) to obtain a explicit algorithm. These four steps are approximate; in practice, there is potential iteration between steps. We provide further details as follows.

 For (1), we meta-trained cohorts of AutoHebb networks that varied these architectural hyperparameters: depth (of main and control networks), number of layers containing plastic weights, and the homo/heterogeneity of the plastic weight scales (denoted \textbf{A}). We did not algorithmically analyze 2-neuromodulator networks, though we did benchmark these latter networks. Additional potential hyperparameters include the unit non-linearities and layer widths; however, we did not explore these. For (2), we typically examined
 \begin{itemize}
	\item fixed weights (and biases) as a function of item position (item 1 vs. 2)
	\item neuromodulator activity over trials
	\item neuromodulator activity with respect to trial output
	\item neuromodulator activity with respect to trial outcome (external feedback)
	\item hidden activity as a function of task variables (item identity, trial output, trial outcome)
	\item hidden activity as a function of task structure (e.g.\ ranking, grouping)
\end{itemize}
For (3), we took findings from (2) that suggested idealization would be possible while retaining essential task behavior(s); we then either reasoned that the idealization was directly accurate or re-ran meta-training with the idealization as a hard constraint to re-assess task performance and behavior. For (4), we summarized the outcome of plastic weight updates to the network's output (e.g.\ beginning with initial state of the network) expressed in explicit mathematical form, then simplified and interpreted these plastic weight terms in an appropriate task- and condition-specific manner.

\section{Identification of learning algorithms from AutoHebb networks}\label{appendixsec:neuralmechanisms}

Here we applied the approach described in Appendix \ref{appendixsec:autohebbjustification} to identify algorithms for two of the tasks investigated in this study, namely TI and TI-L. We first provide general notation across tasks, then for each algorithm we provide a basic outline of the algorithm, followed by a specific analysis of neural components. 

\paragraph{Notation.} Our notation is as follows: let $N$ be the number of items with item dimension $d$ and $\mathbf{1}_{d} = \begin{bmatrix} 1 & \cdots & 1 \end{bmatrix} \in \mathbb{R}^{1 \times d}$ the vector of ones viewed as a $1 \times d$ matrix. If $\mathbf{x}_A \in \{-1,+1\}^{d}$ (viewed as a $d \times 1$ matrix) is an item representation corresponding to item $A$, then we let $\mathbf{x}_{AB} = \text{Concat}(\mathbf{x}_A, \mathbf{x}_B) \in \{-1, +1\}^{2d}$ be the concatenation of two items $A$ and $B$, which is the input format for trial type $AB$ in the tasks presently studied. (We suppress the subscript $AB$ when it is implicitly understood or refers to an arbitrary trial input.) Let $\mathbf{P}(t)$ be the plastic weight matrix following the 3-factor Hebbian update rule at time $t$ (Eq. \ref{eq:3factorrule}). All networks we consider have at most one hidden layer, so we denote the hidden representation corresponding to $\mathbf{x}$ at time $t$ as $\mathbf{h}(t)$ and similarly for $\mathbf{h}_{AB}(t)$ and $\mathbf{h}_A(t)$.

\subsection{A learning algorithm for TI}\label{appendixsubsec:neuralmechti}
In this section, we analyze a minimal AutoHebb network for transitive inference (TI) (described in Section \ref{subsec:transitiveinference}). As further elaborated below, the basic algorithm performs a subtraction of the ranking of the two presented input items in each trial. Each item's ranking is represented by a single scalar value, and the network uses a subtraction mechanism to compare these values to determine which item has a higher ranking. If the network is incorrect, it updates the scalar values (corresponding to ranking) of the two presented items by an amount proportional to the difference in the scalar values.

\paragraph{Main network.} Meta-training progressively simpler AutoHebb networks (see \ref{appendixsec:autohebbjustification}) yielded a minimal network that had a simple perceptron architecture (with both fixed and plastic weights) given by 
\begin{equation}\label{eq:minimaltinetwork}
    y =(\mathbf{\bar{1}}+\mathbf{P}(t))\mathbf{x},
\end{equation}
where $\mathbf{\bar{1}}_{d} = \text{Concat}(\mathbf{1}_d, -\mathbf{1}_d) \in \mathbb{R}^{1 \times 2d}$. Note that the fixed weights $\mathbf{\bar{1}}_{d}$ implements a subtraction between item representations given different item positions (first vs. second position). This anti-correlated relationship between position-specific weights was observed in all meta-trained networks (Fig. \ref{appendixfig:tianticorrelatedweights}). Further, since no meta-learning is needed in Eq. \ref{eq:minimaltinetwork}, we remove the sigmoid that would be needed to be applied to the readout $y$ (Section \ref{subsec:autohebbarch}) in meta-training, and instead determine the choice based on the sign of $y$ using the standard Heaviside function $H(y)$. 

\paragraph{Control network.} In the minimal network, the neuromodulator was found to be predominantly determined by whether reward was received or not ($r = \pm 1$). We idealize this behavior here by setting the neuromodulator value to $0$ for correct responses and $-\eta$ for incorrect responses, where $\eta$ is a fixed positive value. Note that we chose this idealization because it was sufficient to capture basic TI performance; it is also possible that the native meta-trained control network, in addition to other possible idealizations, may produce additional TI-related behavioral patterns of interest (\cite{jensen2019discovering}).

\paragraph{Plastic weight updates.} As a step towards providing a mechanistic account of successful task performance on TI for the idealized network, we analyze the behavior of the plastic weights throughout training trials. Under mild conditions on the item representations and $\eta$, we show the plastic weights converge in the limit for any infinite sequence of training trials. We note the idealized fixed weights are not used in this section. However, the analysis on learning TI in the subsequent section relies on the fixed weights in the idealized main network together with further constraints on item representations and $\eta$.

Given the above description of the control network, the update rule for $P(t)$ can be broken down into two cases:
\begin{align}\label{eq:plasticweightupdateti}
\mathbf{P}(t+1) = \begin{cases} \mathbf{P}(t), & \quad \text{correct response} \\ \mathbf{P}(t) - \eta y\mathbf{x}^\top, & \quad \text{incorrect response}. \end{cases}
\end{align}
Note that this model updates its plastic weights $\mathbf{P}(t)$ only after incorrect responses and therefore self-regulates learning in a trivial way after successfully learning a specific trial type (see Appendix \ref{appendixsubsec:masspresentation}).

Since the pre-synaptic activity is always the input $\mathbf{x}$, and the scalar readout is the post-synaptic activity, we can group together terms in the plastic weights $\mathbf{P}(t)$ based on trial types and write a decomposition
\begin{equation}\label{eq:plasticweightdecompti}
\mathbf{P}(t) = \sum_{XY}\rho_{XY}(t)\mathbf{x}_{XY}^T\, ,
\end{equation}
where $\rho_{XY}(t)$ are coefficients obtained through the 3-factor learning rule for the trial consisting of the ordered pair of items $X$ and $Y$. If $y_{XY}(t) = (\mathbf{\bar{1}}+\mathbf{P}(t))\mathbf{x}_{XY}$ we can further expand $\rho_{XY}(t)$ as
\begin{equation}\label{eq:rhodecompti}
\rho_{XY}(t) = -\eta\sum_{i=0}^{t-1}y_{XY}(i)\mathbb{1}_{\{XY \text{ shown at time }i \text{ and model incorrect}\}}\, ,
\end{equation}
where $\mathbb{1}_{\{\bm{\cdot}\}}$ is the indicator function. The following theorem shows the absolute convergence of $\lim_{t \to \infty}\rho_{XY}(t)$ and, hence, the convergence of the plastic weights $\mathbf{P}(t)$ in the limit. The required conditions are the existence of some decision boundary (not necessarily the limit of the plastic weights) that can correctly classify training trials and a small enough value for $\eta$.

\begin{theorem}\label{thm:ticonvergence}
    Assume the network in Eq. \ref{eq:minimaltinetwork} and update rule in Eq. \ref{eq:plasticweightupdateti}. Let $\mathbf{x}_1$, $\mathbf{x}_2$, \ldots, $\mathbf{x}_N$ be $N$ items in $\{-1,+1\}^d$. For a trial $\mathbf{x}_{IJ}$, let $s(\mathbf{x}_{IJ})$ be the sign of a correct model response for trial type $IJ$. Suppose there exists a weight vector $\mathbf{w}^* \in \mathbb{R}^{1 \times 2d}$ such that for each trial type $IJ$ we have
    \begin{equation*}
    \mathbf{w}^*s(\mathbf{x}_{IJ})\mathbf{x}_{IJ} > 0\,.
    \end{equation*}
    Let $\mathbf{Q} = \{\mathbf{q}_t\}_{t=0}^\infty$ be a sequence of trials with inputs $\mathbf{q}_t$. If $\eta \leq \frac{1}{d}$ as in Eq. \ref{eq:plasticweightupdateti}, then $\mathbf{P}(t)$ converges as $t \to \infty$.
\end{theorem}
\begin{proof}
Let $\mathbf{w}(t) = \mathbf{\bar{1}}+\mathbf{P}(t)$, and define
\[V(t) = \|\mathbf{w}(t)-\mathbf{w}^*\|^2\,.\]
If the model's response $y_{\mathbf{q}_t}$ is incorrect for trial $\mathbf{q}_t$, then by Eq. \ref{eq:plasticweightupdateti}, we have
\begin{align*}
V(t+1) & = \|\mathbf{w}(t+1)-\mathbf{w}^*\|^2 \\
& = \|\mathbf{w}(t) - \eta y_{\mathbf{q}_t}\mathbf{q}_t^\top-\mathbf{w}^*\|^2 \\
& = V(t) - 2\eta y_{\mathbf{q}_t}(\mathbf{w}(t)-\mathbf{w}^*)\mathbf{q}_t + 2\eta^2 y_{\mathbf{q}_t}^2d \\
& =  V(t) - 2\eta s(\mathbf{q}_t)y_{\mathbf{q}_t}(\mathbf{w}(t)-\mathbf{w}^*)s(\mathbf{q}_t)\mathbf{q}_t + 2\eta^2 y_{\mathbf{q}_t}^2d \\
& = V(t) - 2\eta s(\mathbf{q}_t)y_{\mathbf{q}_t}\mathbf{w}(t)s(\mathbf{q}_t)\mathbf{q}_t +2\eta s(\mathbf{q}_t)y_{\mathbf{q}_t}\mathbf{w}^*s(\mathbf{q}_t)\mathbf{q}_t + 2\eta^2 y_{\mathbf{q}_t}^2d \\
& = V(t) - 2\eta y_{\mathbf{q}_t}^2  + 2\eta s(\mathbf{q}_t)y_{\mathbf{q}_t}\mathbf{w}^*s(\mathbf{q}_t)\mathbf{q}_t + 2\eta^2 y_{\mathbf{q}_t}^2d \\
& = V(t) - 2\eta y_{\mathbf{q}_t}^2(1 - \eta d)  + 2\eta s(\mathbf{q}_t)y_{\mathbf{q}_t}\mathbf{w}^*s(\mathbf{q}_t)\mathbf{q}_t\, .
\end{align*}
Since $\eta \leq \frac{1}{d}$, we have the inequality
\begin{align*}
V(t+1) & \leq V(t) + 2\eta s(\mathbf{q}_t)y_{\mathbf{q}_t}\mathbf{w}^*s(\mathbf{q}_t)\mathbf{q}_t \\
& \leq V(t) - 2\eta |y_{\mathbf{q}_t}|\mathbf{w}^*s(\mathbf{q}_t)\mathbf{q}_t \, ,
\end{align*}
where we use the fact $y_{\mathbf{q}_t}$ and $s(\mathbf{q}_t)$ have opposite signs. Furthermore, since there are only finitely many trial types $IJ$, there exists a $\gamma$ such that $\mathbf{w}^*s(\mathbf{x}_{IJ})\mathbf{x}_{IJ} \geq \gamma >  0$, so we have the inequality
\begin{equation*}
V(t+1) \leq V(t) - 2\eta |y_{\mathbf{q}_t}|\gamma \,.
\end{equation*}
Rearranging terms, we get
\begin{equation*}
V(t+1) + 2\eta |y_{\mathbf{q}_t}|\gamma\leq V(t)\,.
\end{equation*}
If we start from $t=0$ and chain the inequalities, we obtain
\begin{equation*}
2\eta\gamma\sum_{t=1}^\infty |y_{\mathbf{q}_t}|\mathbb{1}_{\{\text{incorrect response for }\mathbf{q}_t\}}\leq V(0)\,.
\end{equation*}
After dividing both sides of the inequality by $2\gamma$, we obtain an absolute convergence for the sum of all coefficients in Eq. \ref{eq:plasticweightdecompti} (expanding the $\rho_{XY}(t)$ using update rule Eq. \ref{eq:plasticweightupdateti}). Thus, the plastic weights converge in the limit.
\end{proof}

Next, we provide bounds for the probability of the existence of $\mathbf{w}^*$ in Theorem \ref{thm:ticonvergence}, starting with the lower bound by calculating the probability of linearly independent trial vectors. We remark that the existence of $\mathbf{w}^*$ implies the existence of a vector $\mathbf{u} \in \mathbb{R}^{1 \times d}$ such that \[\mathbf{v}^* = \begin{bmatrix} \mathbf{u}^\top \\ -\mathbf{u}^\top \end{bmatrix}\] also satisfies the condition of Theorem \ref{thm:ticonvergence}. If \[\mathbf{w}^* = \begin{bmatrix} \mathbf{a}^\top \\ \mathbf{b}^\top \end{bmatrix}\] for some $\mathbf{a}$ and $\mathbf{b}$ in $\mathbb{R}^{1 \times d}$, then it is straightforward to verify \[\mathbf{u} = \mathbf{a} - \mathbf{b}\] gives such a $\mathbf{v}^*$.

\begin{proposition}\label{prop:boundbelowprobseparation}
Let $P(N,d)$ be the probability of the existence of $\mathbf{w}^*$ in Theorem \ref{thm:ticonvergence} for $N$ items in $d$ dimensions. Then
\[P(N,d) \geq \prod_{i=2}^{N}(1-2^{i-2-d})\, .\]
\end{proposition}
\begin{proof}
Let $\mathbf{x}_1$, $\mathbf{x}_2$, \ldots, $\mathbf{x}_N$ be $N$ item representations such that the ranking of $\mathbf{x}_i$ is greater than the ranking of $\mathbf{x}_j$ if and only if $i < j$. Further suppose that the $d \times (N-1)$ matrix of adjacent item differences $\mathbf{D}$ whose $i$th column is $\mathbf{x}_i-\mathbf{x}_{i+1}$ has full rank $N-1$. Then there exists a vector $\mathbf{u}$ in $\mathbb{R}^d$ such that \[\mathbf{D}^\top\mathbf{u} = \mathbf{1}_{N-1}\, ,\] i.e. for each $i < N$, we have \[\mathbf{u}^\top(\mathbf{x}_i - \mathbf{x}_{i+1}) = 1\, .\] Thus, the vector \[\mathbf{w}^* = \begin{bmatrix} \mathbf{u}^\top \\ -\mathbf{u}^\top\end{bmatrix}\] satisfies the constrain in Theorem \ref{thm:ticonvergence}, and we only need to compute the probability of the existence of item representations such that $\mathbf{D}$ has full rank to obtain a lower bound.

Since the vectors $\mathbf{x}_i - \mathbf{x}_1$ for $i > 1$ are linear combinations of the columns of $\mathbf{D}$, the matrix $\mathbf{D}$ has full rank $N-1$ if and only if the $\mathbf{x}_i$ are affinely independent. For $k-1$ affinely independent vectors, the dimension of the affine subspace generated by those vectors is $k-2$. By Lemma \ref{lem:affinehypercube}, this subspace contains at most $2^{k-2}$ possible item representations. Thus, the probability that introducing a new item vector $\mathbf{x}_k$ affinely independent from the first $k-1$ vectors is \[P(\text{$\mathbf{x}_k$ affinely independent $|$ $\mathbf{x}_1$, $\mathbf{x}_2$, $\ldots$, $\mathbf{x}_{k-1}$}) \geq 1-\frac{2^{k-2}}{2^d} = 1 - 2^{k-2-d} \, .\] Taking the product over all conditional probabilities gives the result.
\end{proof}

The lower bound in Proposition \ref{prop:boundbelowprobseparation} remains close to $1.0$ for approximately $N < d-10$ and then decays rapidly to $0$. However, simulations with the idealized network show successful task performance for approximately $N < \frac{5}{4}d$, beyond which networks empirically fail to learn TI. In order to provide a better assessment of the capacity of the idealized network, we provide an upper bound on $P(N, d)$.

\begin{proposition}\label{prop:boundaboveprobseparation}
With the notation from Proposition \ref{prop:boundbelowprobseparation}, we have \[P(N,d) \leq \min\left\{ 1, \frac{ 2\displaystyle\sum_{i=0}^{d-1} \binom{\binom N2-1}{i} }{ N! } \right\}\, .\]
\end{proposition}
\begin{proof}
There are $\binom{N}{2}$ unordered pairs of $N$ items each determining a hyperplane via the equation \[\mathbf{u}^\top(\mathbf{x}_i-\mathbf{x}_j) = 0\, ,\] where $\mathbf{u}$ is the vector of variables. As long as $\mathbf{u}$ is away from these $\binom{N}{2}$ hyperplanes, the weight vector $\mathbf{w} = \begin{bmatrix}\mathbf{u}^\top \\ -\mathbf{u}^\top\end{bmatrix}$ determines a strict ordering among all items. By Lemma \ref{lem:hyperplanechambercount}, there are at most \[2\displaystyle\sum_{i=0}^{d-1} \binom{\binom N2-1}{i}\] regions in $\mathbb{R}^d$ determined by the hyperplanes. From our earlier remark, any weight vector $\mathbf{w}^*$ from Theorem \ref{thm:ticonvergence}, if it exists, induces a vector of type $\mathbf{w}$ that also satisfies the condition of $\mathbf{w}^*$. Therefore, the probability $\mathbf{w}^*$ exists is equivalent to the probability of a $\mathbf{w}^*$ of the form given by $\mathbf{w}$. Any such $\mathbf{w}$ lying away from all decision boundaries must be contained in one of these chambers, and there is at most one chamber containing weight vectors giving the correct order. Since any ordering is equally likely for any sampled set of items, the probability of a chamber containing weight vectors that correctly order the items is at most $1/N!$, which gives the desired bound.
\end{proof}

For a fixed $d$, the numerator of the upper bound in Proposition \ref{prop:boundaboveprobseparation} is a polynomial in $N$ of $O(N^{2d-2})$, which grows much more slowly compared to the denominator $N!$ as $N \to \infty$. For example, if $d=200$, the upper bound rapidly decays to $0$ when $N > 266$, with $P(270, 200)$ already bounded from above by $0.0001$.

In summary, we have shown that the plastic weights converge in the limit for $\eta \leq 1/d$ under the condition that a weight vector $\mathbf{w}^*$ exists. We then showed that the probability a given set of $N$ items in $d$ dimensions is bounded by
\begin{equation*}
    \prod_{i=2}^{N}(1-2^{i-2-d}) \leq P(N,d) \leq \min\left\{ 1, \frac{ 2\displaystyle\sum_{i=0}^{d-1} \binom{\binom N2-1}{i} }{ N! } \right\} \, ,
\end{equation*}
which tightly bounds $P(N,d)$ close to $1$ for $N < d-10$ and close to $0$ for $N > (1+\epsilon)d$ for some small $\epsilon$. We note that a crude analysis of the upper bound shows $\epsilon \to 0$ as $d \to \infty$. Thus, we have shown that for most $N < d$, the plastic weights converge for any sequence of trials with probability almost $1$. However, our analysis only assumes such a $\mathbf{w}^*$ exists and not that $\mathbf{P}(t)$ converges to $\mathbf{w}^*$. We note that there is a proof of convergence without the need for $\mathbf{w}^*$ (Theorem \ref{thm:plasticweightconvergencetiweak}), but it leaves out the crucial case $\eta = 1/d$, which is required for our theoretical results on successful TI performance in the next section.

\paragraph{Learning.} The previous section analyzed plastic weight behavior in the limit without reference to TI performance. Here, we discuss when the converged plastic weights contribute to successful TI. The analysis amounts to the consideration of two cases: 1) trial sequences with infinite model errors, and 2) trial sequences with finite errors. In case 1), the model converges to $y=0$ for trial types with infinite errors, while, in case 2), the model converges to $y \neq 0$ with probability $1$. Furthermore, in case 2), if each trial type is shown infinitely often, then $y$ necessarily converges to the correct response.

In case 1), the argument is straightforward: since $y_{IJ}$ converges in the limit for any trial type $IJ$ and there are infinite errors, we have infinitely many $t$ such that \[|y_{IJ}(t+1) - y_{IJ}(t)| = 2\eta d|y_{IJ}(t)|\, ,\] which converges to $0$ by the definition of a convergent limit. For case 2), we require an additional assumption that $\eta = 1/d$, where a weight update are a trial type $IJ$ becomes a reflection across the decision hyperplane for trial $IJ$:
\begin{align}\label{eq:householderreflectti}
\mathbf{\bar{1}} + \mathbf{P}(t+1) & = \mathbf{\bar{1}} + \mathbf{P}(t) - \eta y_{IJ}\mathbf{x}_{IJ}^{\top} \nonumber \\
& = (\mathbf{\bar{1}} + \mathbf{P}(t)) - \eta (\mathbf{\bar{1}} + P(t))\mathbf{x}_{IJ}\mathbf{x}_{IJ}^{\top} \nonumber \\
& = (\mathbf{\bar{1}} + \mathbf{P}(t))(\mathbf{I}_{2d} - \eta\mathbf{x}_{IJ}\mathbf{x}_{IJ}^{\top}) \nonumber \\
& = (\mathbf{\bar{1}} + \mathbf{P}(t))\left(\mathbf{I}_{2d} - \frac{2}{2d}\mathbf{x}_{IJ}\mathbf{x}_{IJ}^{\top}\right) \nonumber \\
& = (\mathbf{\bar{1}} + \mathbf{P}(t))\left(\mathbf{I}_{2d} - 2\frac{\mathbf{x}_{IJ}\mathbf{x}_{IJ}^{\top}}{\|\mathbf{x}_{IJ}\|^2}\right) \, ,
\end{align}
where $\mathbf{I}_{2d}$ is the $2d \times 2d$ identity matrix. This mitigates the decay of the magnitude of the error-induced weight updates since reflections preserve the distance to the decision hyperplane and forces finitely many nonzero ($y \neq 0$) errors under convergence. However, weight updates of one trial type $IJ$ can cause the response of another correlated trial type $KL$ to land directly on the decision hyperplane, which would prevent the model from learning a correct response to $KL$.

\begin{theorem}\label{thm:finitelymanyerrors}
Fix the number of items $N$ and their dimension $d$. If $\eta = 1/d$, then:
\begin{enumerate}
    \item[1)] If a sampled set of items has a $\mathbf{w}^*$ as in Theorem \ref{thm:ticonvergence}, then any infinite sequence of trials induces finitely many nonzero error updates.
    \item[2)] Let $\epsilon > 0$, and let $\xi \sim \text{Unif}([-\epsilon, \epsilon]^{2d})$ be a random perturbation of $\mathbf{\bar{1}}$.  The probability that the perturbation $\xi$ satisfies the following is $1$:
    \begin{itemize}
    \item Let $\mathbf{w}_{\xi}(t) = \mathbf{\bar{1}}+\xi + \mathbf{P}(t)$ be the perturbed weights, where the plastic weight rule in Eq. \ref{eq:plasticweightupdateti} treats $\xi$ as part of the fixed weights. For every set of item representations having a corresponding weight vector $\mathbf{w}^*$ from Theorem \ref{thm:ticonvergence} and any sequence of trials $\mathbf{Q} = \{\mathbf{q}_t\}_{t=0}^\infty$ where each trial type $IJ$ is presented infinitely often, the limiting weights $\lim_{t \to \infty} \mathbf{w}_{\xi}(t)$ converge, and we have
    \begin{equation*}
        \lim_{t \to \infty} |\mathbf{w}_{\xi}(t)\mathbf{x}_{IJ}| > 0\, .
    \end{equation*} 
    \end{itemize}
\end{enumerate}
\end{theorem}
\begin{proof}
For 1), since $\eta \leq 1/d$, the weights \[\mathbf{w}(t) = \mathbf{\bar{1}}+\mathbf{P}(t)\] converge by Theorem \ref{thm:ticonvergence}. Suppose towards a contradiction that there are infinitely many nonzero errors in some sequence of trials. Let $E$ be the set of all trial types that have infinitely many nonzero errors. For each trial type $IJ$ in E, let $H_{IJ}$ be its decision hyperplane in weight space, and define \[L := \bigcap_{IJ}H_{IJ}\, .\] For trial types not in $E$, the weights $\mathbf{w}(t)$ stop updating for those trial types after some finite time $T$. Thus, for $t > T$, the weight vector $\mathbf{w}(t)$ preserves its distance to $L$ and continues to reflect through the decision hyperplanes $H_{IJ}$. However, we have already shown infinitely many nonzero error updates implies $\mathbf{w}(t)$ must converge to a weight vector in $L$, and since the distance of $\mathbf{w}(t)$ to $L$ is constant after finite time $T$, this constant must be $0$. This contradicts the assumption of infinite nonzero errors.

To prove 2), we show the set of weight vectors $\mathbf{w}$ in $\mathbb{R}^{2d}$ lying on the decision boundary of any trial type after any finite sequence of trials has measure zero. In the following, we consider arbitrary weight initializations $\mathbf{w}_0 =\mathbf{w}(0)$, not only $\mathbf{w}_{\xi}(0)$. Each trial type defines a weight update given by a Householder reflection (Eq. \ref{eq:householderreflectti}) through the decision boundary for the trial type, which is independent of $\mathbf{w}(t)$. Thus, if $\mathbf{R}_{\tau}$ is the Householder reflection corresponding to trial $\tau$, the dynamics of any sequence of weight updates can be represented by a sequence of reflections (we treat $\mathbf{w}(t)$ as a row vector for notational convenience):
\begin{equation*}
    \mathbf{w}(t) = \mathbf{w}_0\mathbf{R}_{\tau_1}\mathbf{R}_{\tau_2}\cdots\mathbf{R}_{\tau_t} \, .
\end{equation*}
For simplicity and without loss of generality, we ignore trials where no weight update is applied. Let $\sigma = (\tau_{1}, \tau_{2}, \ldots, \tau_{t})$ be a sequence of reflections, and let \[\mathbf{Q}_{\sigma} = \mathbf{R}_{\tau_{t}}\cdots\mathbf{R}_{\tau_2}\mathbf{R}_{\tau_1}\, .\] For each trial type $IJ$, the weights $\mathbf{w}(t)$ lie on the decision boundary for $IJ$ if 
\begin{align*}
\langle \mathbf{w}(t), \mathbf{x}_{IJ} \rangle & = \langle \mathbf{w}_0\mathbf{Q}_{\sigma}^\top, \mathbf{x}_{IJ}\rangle \\
& = \langle \mathbf{w}_0, \mathbf{Q}_{\sigma}^\top \mathbf{x}_{IJ} \rangle \\
& = 0
\end{align*}
Since reflections are invertible, the matrix $\mathbf{Q}_{\sigma}^\top$ is invertible and $\mathbf{Q}_{\sigma}^\top\mathbf{x}_{IJ} \neq 0$. Thus \[\langle \mathbf{w}_0, \mathbf{Q}_{\sigma}^\top \mathbf{x}_{IJ} \rangle = 0\] defines a hyperplane in the weight space $\mathbb{R}^{2d}$ for initializations $\mathbf{w}_0$. Since there are
\begin{itemize}
\item finitely many possible collections of items,
\item finitely many trial types per collection of items,
\item and countably many finite sequences of reflections $\sigma$,
\end{itemize}
we have countably infinitely many hyperplanes where $\mathbf{w}(0)$ lies on the decision boundary of some trial type. A countably infinite set of measure zero sets has measure zero, hence, a uniformly randomly sampled $\mathbf{w}_0$ does not lie on the decision boundary of any possible trial type after any finite sequence of weight updates with probability $1$. Since $\xi$ is drawn uniformly from an $\epsilon$-hypercube around $\mathbf{\bar{1}}$, the same conclusion holds for $\mathbf{w}_{\xi}(0)$.

The convergence of $\mathbf{w}_{\xi}(t)$ follows from Theorem \ref{thm:ticonvergence}---the proof does not rely on the weight initialization. By part 1) (which also does not require any special weight initialization by the more general form of Theorem \ref{thm:ticonvergence}), the weight updates for $\mathbf{w}_{\xi}(t)$ cease after a finite number of trials, so the limit is nonzero outside a set of measure zero for the initialization $\mathbf{w}_{\xi, 0}$.

\end{proof}

It follows immediately from the update rule Eq. \ref{eq:plasticweightupdateti} the limiting weights for $\mathbf{w}_{\xi}$ in Theorem \ref{thm:finitelymanyerrors} 2) lie on the correct side of the decision boundary for every trial type as long as every trial trial type is presented infinitely often. Note that we have not made any claims about learning when $\mathbf{\xi} = \mathbf{0}$. Although $\epsilon$ can be made arbitrarily small, the limiting case when $\mathbf{\xi} = \mathbf{0}$ introduces algebraic identities that complicate extending the proof of Theorem \ref{thm:finitelymanyerrors} to the specialization $\mathbf{w}_0 = \mathbf{\bar{1}}$. In our simulations, we found that $\mathbf{w}(t) = \mathbf{\bar{1}} + \mathbf{P}(t)$ was rarely contained in the decision boundary of some trial type at the end of an episode. Nonetheless, this proof illustrates the essential learning behavior of our idealized network.

\paragraph{A representation of rank.} We now describe how the model represents rank. First, we will show that, in expectation over trial sequences, the learned rank representation is position-independent. Second, we show that when the network is presented with an item $I$ in a particular trial, it determines its rank by recalling previous trial outcomes only involving $I$ (in the same position as in the current trial). Further, the rank representation of $I$ is changed only when the model responds incorrectly to a trial containing item $I$.

Let \[\mathbf{w}(t) = \mathbf{\bar{1}} + \mathbf{P}(t)\] be the weights at time $t$. We can express $\mathbf{w}(t)$ as a concatenation of two vectors $\mathbf{a}(t)$ and $\mathbf{b}(t)$ in $\mathbb{R}^{d \times 1}$ acting on the individual items in a trial: \[\mathbf{w}(t) = \begin{bmatrix} \mathbf{a}(t)^\top & \mathbf{b}(t)^\top\end{bmatrix}\, .\] Then, averaging over the distribution of trial sequences, we show $\mathbb{E}[\mathbf{a}(t)] = -\mathbb{E}[\mathbf{b}(t)]$ by exploiting a symmetry in the distribution of sequences of trials. If we define \[\mathbf{w}(t)^{\dagger} = \begin{bmatrix} -\mathbf{b}(t)^\top & -\mathbf{a}(t)^\top\end{bmatrix}\, ,\] then $\mathbf{w}(t)\mathbf{x}_{IJ}$ is incorrect if and only if $\mathbf{w}(t)^{\dagger}\mathbf{x}_{JI}$ is incorrect. Furthermore, the weight updates on the incorrect responses mirror each other:
\begin{align*}
\mathbf{w}(t+1) & = \mathbf{w}(t) - \eta y_{IJ}(t)\mathbf{x}_{IJ}^\top \\
& = \begin{bmatrix} \mathbf{a}(t)^\top & \mathbf{b}(t)^\top\end{bmatrix} - \eta(\mathbf{a}(t)^\top\mathbf{x}_I + \mathbf{b}(t)^\top\mathbf{x}_J )\mathbf{x}_{IJ}^\top \, ,
\end{align*}
and, if $y_{JI}(t)^{\dagger} = \mathbf{w}(t)^{\dagger}\mathbf{x}_{JI}$, then

\begin{align*}
\mathbf{w}(t)^{\dagger} - \eta y_{JI}(t)^{\dagger}\mathbf{x}_{JI}^\top 
& =  \begin{bmatrix} -\mathbf{b}(t)^\top & -\mathbf{a}(t)^\top\end{bmatrix} - \eta(-\mathbf{a}(t)^\top\mathbf{x}_I -\mathbf{b}(t)^\top\mathbf{x}_J )\mathbf{x}_{JI}^\top \\
& = \mathbf{w}(t+1)^{\dagger}\, .
\end{align*}
Let $\mathbf{x}_{IJ}^{\dagger} = \mathbf{x}_{JI}$. Since $\mathbf{\bar{1}} = \mathbf{\bar{1}}^{\dagger}$, and we have a bijection between finite sequences of trials $\mathbf{Q} = \{\mathbf{q}_t\}_{t=0}^T$ and $\mathbf{Q}^{\dagger} = \{\mathbf{q}_t^{\dagger}\}_{t=0}^T$, we have \[\mathbb{E}[\mathbf{a}(t)] = -\mathbb{E}[\mathbf{b}(t)]\, .\] If $r_{I_L}(t) = \mathbf{a}(t)^\top\mathbf{x}_I$ and $r_{I_R}(t) = \mathbf{b}(t)^\top\mathbf{x}_I$, then 
\begin{align*}
\mathbb{E}[r_{I_L}(t)] & = \mathbb{E}[\mathbf{a}(t)]^\top\mathbf{x}_I \\
& = -\mathbb{E}[\mathbf{b}(t)]^\top\mathbf{x}_I \\
& = -\mathbb{E}[r_{I_R}(t)]\, ,
\end{align*}
so the magnitude of the rank of an item is independent of its the position in expectation. 

Now we consider the components of the representation. In the following, we only consider $r_{I_L}$, but the arguments are symmetric for $r_{I_R}$. Let $\mathbf{x}_{I0}$ be the input vector with item $I$ in the first position and the vector of zeros in the second position and vice versa for $\mathbf{x}_{0J}$. Then, using Eq. \ref{eq:plasticweightdecompti}, the readout for an item $I$ in the first position takes on a simple form (suppressing time $t$ notation)

\begin{align*}
\mathbf{\bar{1}}\mathbf{x}_{I0} + \mathbf{P}\mathbf{x}_{I0} & = c_I + \sum_{XY}\rho_{XY}\mathbf{x}_{XY}^T\mathbf{x}_{I0} \\
& = c_I + \sum_{Y}\rho_{IY}\mathbf{x}_{IY}^T\mathbf{x}_{I0} + \epsilon \\
& = c_I + d\sum_{Y}\rho_{IY} + \epsilon\, ,
\end{align*}
where $\epsilon$ is a small quantity (in expectation) measuring interference between inputs relative to their dimension. Thus, we have, approximately, 
\begin{equation}\label{eq:rankrepresentationti}
    r_{I_L}(t) \approx c_I + d\sum_{Y}\rho_{IY}(t) \, .
\end{equation}
Due to this observation and Eq. \ref{eq:plasticweightdecompti}, we see that updates only affect items presented in the trial responsible for the update; for a trial $IJ$ where the model responds incorrectly, we have the plastic weight update:
\begin{align*}
    \mathbf{P}(t+1) & = \mathbf{P}(t) - \eta y_{IJ}(t)\mathbf{x}_{IJ}^\top \\
    & = \left(\sum_{XY}\rho_{XY}(t)\mathbf{x}_{XY}^T\right) - \eta y_{IJ}(t)\mathbf{x}_{IJ}^\top \\
    & = \left(\sum_{XY \neq IJ}\rho_{XY}(t)\mathbf{x}_{XY}^T\right) + (\rho_{IJ}(t) - \eta y_{IJ}(t))\mathbf{x}_{IJ}^\top \, ,
\end{align*}
so, by Eq. \ref{eq:rankrepresentationti}, the rank representation for $I$ at time $t+1$ is
\begin{align*}
    r_{I_L}(t+1) & \approx c_I + d\sum_{Y}\rho_{IY}(t+1) \\
    & = c_I + \left(d\sum_{Y \neq J}\rho_{IY}(t)\right) + d\rho_{IY}(t+1) \\
    & = c_I + \left(d\sum_{Y \neq J}\rho_{IY}(t)\right) + d(\rho_{IY}(t) - \eta y_{IJ}(t)) \\
    & = c_I + \left(d\sum_{Y}\rho_{IY}(t)\right) - \eta dy_{IJ}(t) \\
    & \approx r_{I_L} - \eta dy_{IJ}(t)\, .
\end{align*}
By a similar calculation, we have $r_{K_L}(t+1) \approx r_{K_L}(t)$ for $K \neq I$.

\paragraph{An algorithm for TI.} We are now able to present an algorithm based on our analyses. First, we remark that the position-dependent behavior of the network can be removed if, for example, a fixed random projection without plastic weights but with an anti-correlated structure on the item-wise weights is inserted before the readout (e.g. a row-wise concatenation $\text{Concat}\left(\mathbf{U}, -\mathbf{U}\right)$ for some random matrix $\mathbf{U} \in \mathbb{R}^{2d \times d}$) combined with modifying $\mathbf{\bar{1}}$ to be simply a vector of ones without any negative values. However, we remain faithful to this position-dependent behavior in Algorithm \ref{algorithm:ti} as part of our aims outlined in Appendix \ref{appendixsec:autohebbjustification} and initialize rank vectors $\mathbf{R}_L$ and $\mathbf{R}_R$ corresponding to the left and right position rank representations, respectively. These are initialized from a common randomly sampled vector $\mathbf{R}$ (line 1), which can be viewed as the values $c_{X_i}$ for items $X_i$, almost directly, except we allow more flexibility in the choice of sampling distribution.

The rank difference $\rho$ in line 3 is exactly our idealized network response. Line 4 ensures updates only occur on incorrect trials. Our analysis above shows that the network only updates it rank representation for items present in the given trial, which manifests in lines 5 and 6. The precise update in our analysis for a trial $IJ$ was $\pm\eta dy_{IJ}(t)$, but the dimension $d$ in the algorithm is no longer relevant. A suitable learning rate $\eta$ can be found via a simple hyperparameter search. For example, if $M=8$ and $\mathbf{R}$ is initialized by uniformly sampling from $[-0.1, 0.1]$, then $\eta = 0.7$ will successfully perform TI.

\begin{figure}[t]
    \centering
    \includegraphics[width=0.5\linewidth]{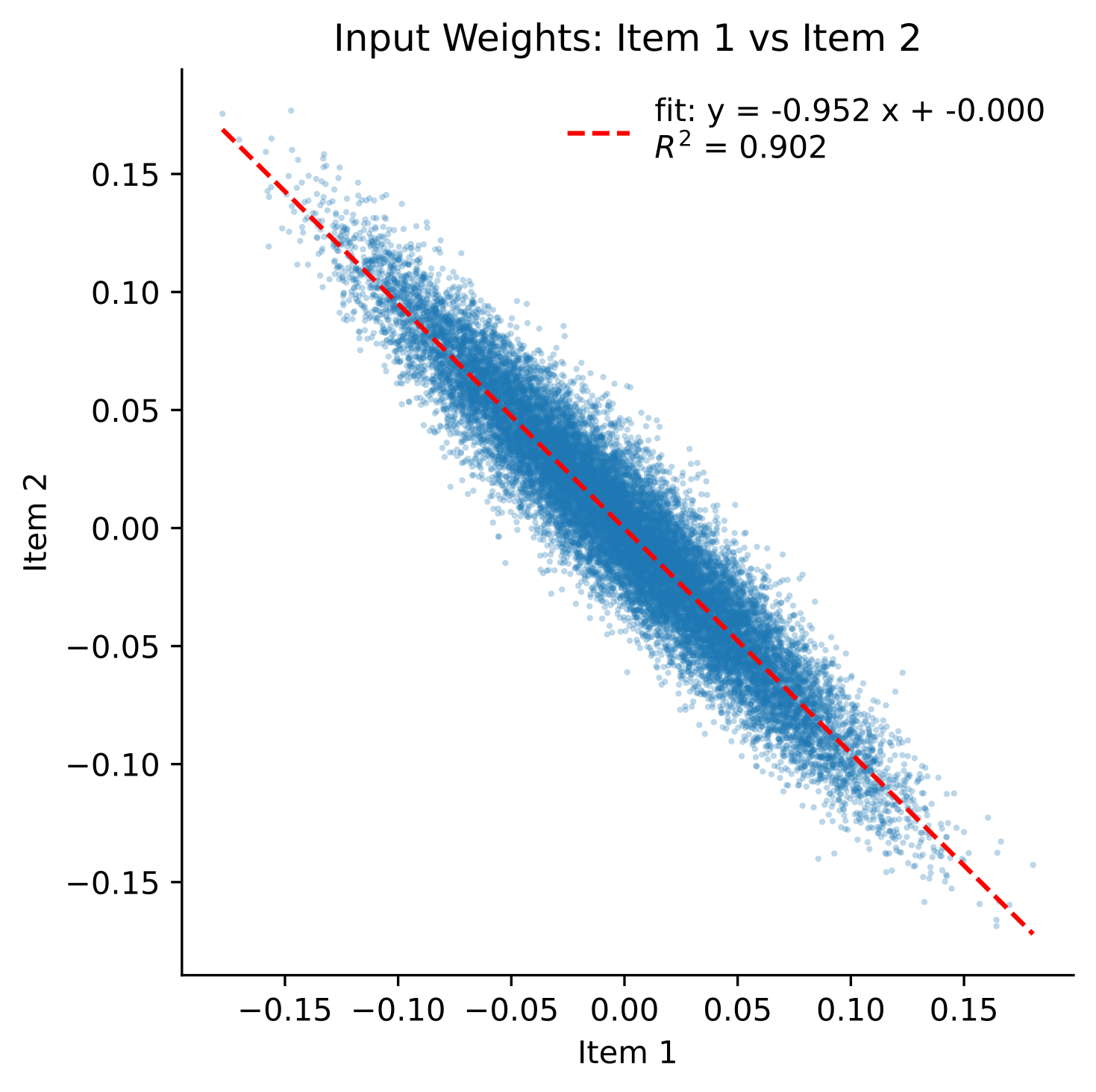}
    \caption{Anti-correlative relationship between item 1 and item 2 weights in the first layer of an AutoHebb TI model.}
    \label{appendixfig:tianticorrelatedweights}
\end{figure}

\begin{algorithm}[t]
\caption{TI Algorithm}
\begin{algorithmic}[1]
\Require Positive integer $M$ \Comment{Number of items} 
\Require $\eta > 0$ \Comment{Neuromodulator value for negative reward}
\Require Randomly initialize rank vector $\mathbf{R} \in \mathbb{R}^M$ \Comment{Innate rank representation}
\Require Episode $\mathcal{E} = \{(i_1, j_1), (i_2, j_2), \ldots, (i_T, j_T)\}$, \Statex \qquad\qquad\qquad\qquad\qquad where $|i_k-j_k| = 1$ and $1 \leq i_s,j_t \leq M$
\State $\mathbf{R}_L = \mathbf{R}_R = \mathbf{R}$ \Comment{Position-dependent ranks}
\For{$i,j$ in $\mathcal{E}$}
\State $\rho \gets \mathbf{R}_L[i] - \mathbf{R}_R[j]$ \Comment{Relative rank}
\If{($\rho \leq 0$ and $i > j$) or ($\rho > 0$ and $i < j$)} \Comment{Update on errors only}
\State $\mathbf{R}_L[i] \gets \mathbf{R}_L[i] - \eta\rho$
\State $\mathbf{R}_R[j] \gets \mathbf{R}_R[j] + \eta\rho$
\EndIf
\EndFor
\end{algorithmic}
\label{algorithm:ti}
\end{algorithm}

\subsection{A neural mechanism for TI-L}\label{appendixsubsec:neuralmechtil}

A successful model for TI-linking (TI-L, Section \ref{subsec:methodscontinual}),  must be able to modify its internal ranking of items absent in a given trial. Therefore, the toy network in Appendix \ref{appendixsubsec:neuralmechti} is insufficient, according to our algorithm. While we have not proved that any one-plastic-weight network with many hidden layers combined with a rich control network cannot perform TI-L, the mechanism outlined in Appendix \ref{appendixsubsec:neuralmechti} is consistent with the two-hidden layer networks with only one set of plastic weights we trained but did not report here. The AutoHebb models used our experimental results in Table \ref{table:zeroshotlinking} show successful TI-L performance with an additional set of plastic weights. Through ablating model components of these successful networks, we arrived at the toy network used throughout this subsection.

\paragraph{Main network.} Like the idealized network in Appendix \ref{appendixsubsec:neuralmechti}, we found that nonlinearities were unnecessary for successful TI-L. However, we could not obtain successful linking without at least two sets of plastic weights, which led us to a network with one hidden layer. Throughout our simulations, we observed that meta-trained AutoHebb models carry an anti-correlative structure between item-wise weights in the first set of meta-trainable weights in the main network, so the analogue of the subtraction between item representations in Appendix \ref{appendixsubsec:neuralmechti} occurs earlier in this network. We set $\mathbf{W} = \text{Concat}\left(\mathbf{U},-\mathbf{U}\right)$ (denoted $\begin{bmatrix}\mathbf{U} & -\mathbf{U} \end{bmatrix}$ for simplicity), a row-wise concatenation of a random matrix $\mathbf{U} \in \mathbb{R}^{2d \times d}$, and we obtain the following idealized AutoHebb network:

\begin{equation}\label{eq:tilarch}
y =(\mathbf{1}+\alpha_2\mathbf{P}_2(t))(\mathbf{W}+\alpha_1\mathbf{P}_1(t))\mathbf{x} \quad,
\end{equation}

where $\mathbf{1}$ is a $1 \times 2d$ matrix of $1$'s and $\alpha_1$ and $\alpha_2$ are positive scalars.

\paragraph{Control network.} The essential behavior of the control network for the meta-trained networks was observed to be similar to that in Appendix \ref{appendixsubsec:neuralmechti}. The control network for our idealized network in Eq. \ref{eq:tilarch} behaves exactly as in Appendix \ref{appendixsubsec:neuralmechti}; we fix a neuromodulator value $-\eta < 0$ for incorrect model responses but do not update plastic weights on correct responses.

\paragraph{Plastic weights.} 
As in the TI case, we will analyze the general behavior of the plastic weights throughout learning. We first discuss general convergence criteria and then analyze the interactions between the two sets of innate and plastic weights for learning and inference, which yields a four different interactions. Together, these four interactions yield a richer mechanism compared with the idealized TI network of Eq. \ref{eq:minimaltinetwork}. In the idealized TI-L network of Eq. \ref{eq:tilarch}, the first set of plastic weights store past item associations that are retrieved by inputs, as in the idealized TI network. However, the associative data stored are the hidden representations rather than the trial outcomes. These trial outcomes are now stored downstream in the second set of plastic weights, which are accessed by hidden states. With two successive plastic weights, the network produces a hidden state that can recall trial outcomes of associated items not present in the input. In other words, the network utilizes the ranks of all items associated with the input items to generate a choice, whereas the idealized TI network considers only the input items themselves.

Based on the idealized control network above, the update rules for $\mathbf{P}_1(t)$ and $\mathbf{P}_2(t)$ are as follows:
\begin{equation}\label{eq:tilplasticupdaterule}
\begin{aligned}
    \mathbf{P}_1(t+1) & = \begin{cases} \mathbf{P}_1(t), & \quad \text{correct response} \\ \mathbf{P}_1(t) - \eta \mathbf{h}\mathbf{x}^\top, & \quad \text{incorrect response} \end{cases} \\\\
    & \text{and} \\\\
    \mathbf{P}_2(t+1) & = \begin{cases} \mathbf{P}_2(t), & \quad \text{correct response} \\ \mathbf{P}_2(t) - \eta y\mathbf{h}^\top, & \quad \text{incorrect response}. \end{cases} \qquad \\
\end{aligned}
\end{equation}

If we absorb the $\alpha_i$ into $\eta$, we can view the network as having distinct neuromodulator values for incorrect response, $\lambda_1 = \alpha_1\eta$ and $\lambda_2 = \alpha_2\eta$. Similar to Theorem \ref{thm:plasticweightconvergencetiweak}, the plastic weights $\mathbf{P}_1$ and $\mathbf{P}_2$ converge as $t \to \infty$ when $\lambda_1 < 1/d$ and $\lambda_2$ is any positive value. While there is a similar analysis one can perform for the case of separating weight vector $\mathbf{w}^*$ when $\lambda_1 = 1/d$, we will show in our theoretical analysis for TI-L that the optimal value for $\lambda_1$ is less than $1/d$. Since we are not concerned with general TI behavior in this section, we omit this analysis.

\paragraph{Learning.} We now aim to understand how the idealized network utilizes both $\mathbf{P}_1$ and $\mathbf{P}_2$ to learn. We show the updates for $\mathbf{P}_1(t)$ enable the network to recall past trial outcomes in $\mathbf{P}_2(t)$ for items different from the inputs. We end with a theoretical statement that guarantees successful linking when each individual list is learned.

Like Eq. \ref{eq:plasticweightdecompti}, we can decompose each of the plastic weights:

\begin{equation}
    \mathbf{P}_1(t) = \sum_{IJ}\tilde{\mathbf{h}}_{IJ}(t)\mathbf{x}_{IJ}^\top \, ,
\end{equation}
where
\begin{align}\label{eq:tilhiddendecomp}
    \tilde{\mathbf{h}}_{IJ}(t) = -\eta\sum_{i=0}^{t-1}\mathbf{h}_{IJ}(i)\mathbb{1}_{\{IJ \text{ shown at time }i \text{ and model incorrect}\}}\, ,
\end{align}
and
\begin{equation}\label{eq:tilsecondplasticdecomp}
    \mathbf{P}_2(t) = \sum_{IJ}\mathbf{\beta}_{KL}(t)\mathbf{w}_{KL}^\top \, ,
\end{equation}
where $\mathbf{w}_{KL} = \mathbf{W}\mathbf{x}_{KL}$ and $\beta_{IJ}(t)$ is some time-evolving coefficient depending on the responses $y$.

Assuming orthogonality of item representations ($\mathbf{x}_I\cdot\mathbf{x}_J = 0$ for $I \neq J$) and an incorrect response for some trial $KL$, we get the following expression for $\mathbf{P}_1(t+1)$:
\begin{equation}\label{eq:tilp1updatedecomp}
\begin{aligned}
    \mathbf{P}_1(t+1) & = \mathbf{P}_1(t) -\eta\mathbf{h}_{KL}(t)\mathbf{x}_{KL}^\top \\
    & = \sum_{IJ}\tilde{\mathbf{h}}_{IJ}(t)\mathbf{x}_{IJ}^\top -\eta\mathbf{h}_{KL}(t)\mathbf{x}_{KL}^\top \\
    & = \sum_{IJ}\tilde{\mathbf{h}}_{IJ}(t)\mathbf{x}_{IJ}^\top -\eta(\mathbf{W}\mathbf{x}_{KL} + \alpha_1\mathbf{P}_1(t)\mathbf{x}_{KL})\mathbf{x}_{KL}^\top \\
    & = \sum_{IJ}\tilde{\mathbf{h}}_{IJ}(t)\mathbf{x}_{IJ}^\top -\eta\mathbf{W}\mathbf{x}_{KL}\mathbf{x}_{KL}^\top - \lambda_1d\left(\sum_{J}\tilde{\mathbf{h}}_{KJ}(t) + \sum_{I}\tilde{\mathbf{h}}_{IL}(t)\right)\mathbf{x}_{KL}^\top \\
    & = \sum_{IJ \neq KL}\tilde{\mathbf{h}}_{IJ}(t)\mathbf{x}_{IJ}^\top + \left(\tilde{\mathbf{h}}_{KL}(t) -\eta\mathbf{W}\mathbf{x}_{KL} - \lambda_1d\left(\sum_{J}\tilde{\mathbf{h}}_{KJ}(t) + \sum_{I}\tilde{\mathbf{h}}_{IL}(t)\right)\right)\mathbf{x}_{KL}^\top \, .
\end{aligned}
\end{equation}
Thus, only the vector coefficient for $\mathbf{x}_{KL}^\top$ is updated. The term $-\eta\mathbf{W}\mathbf{x}_{KL}$ contributes a fixed innate vector update, while \[-\lambda_1d\left(\sum_{J}\tilde{\mathbf{h}}_{KJ}(t) + \sum_{I}\tilde{\mathbf{h}}_{IL}(t)\right)\] contributes a collection of item associations stored in the plastic weights in the form of $\tilde{\mathbf{h}}_{IJ}(t)$, where at least one of $I = K$ or $J = L$. Note that the $\tilde{\mathbf{h}}_{IJ}(t)$ themselves consist of past hidden states $\mathbf{h}_{IJ}(i)$ by Eq. \ref{eq:tilhiddendecomp}. With sufficient training examples and by a recursive argument, the $\mathbf{h}_{KL}(t)$ can have nontrivial dot product with $\mathbf{w}_{IJ}$ for $I \neq K$ and $J \neq L$. Through $\mathbf{P}_2(t)$ and its decomposition given in Eq. \ref{eq:tilsecondplasticdecomp}, this memory structure enables the network to recall past trial outcomes whose inputs consist of items distinct from both $K$ and $L$.

Under idealized assumptions about the learned representations of each individual list, we show that the network is guaranteed to link both lists and form the correct order among all items after seeing a single linking pair.

\begin{theorem}\label{thm:tilguarantee}
     Let $\mathbf{x}_1$, \ldots, $\mathbf{x}_{2N}$ be $2N$ items in $\{-1,+1\}^d$. Assume the network \[f_t(\mathbf{x}_I, \mathbf{x}_J) = (\mathbf{1}+\alpha_2\mathbf{P}_2(t))(\mathbf{W}+\alpha_1\mathbf{P}_1(t))\mathbf{x}_{IJ}\] in Eq. \ref{eq:tilarch} and update rule in Eq. \ref{eq:tilplasticupdaterule}. Let $\delta \in [0,1]$, and suppose for some time $T$ the following conditions hold:
\begin{enumerate}
    \item[(1)] (Sufficiently uncorrelated inputs.) For $I \neq J$, we have $|\mathbf x_I^\top\mathbf x_J| \leq d\delta$.
    \item[(2)] (Perfectly symmetric within-list TI.) If we set \[\mathbf{A}_T:=\mathbf W+\alpha_1\mathbf P_1(T)=[\mathbf{A}_L\;\mathbf{A}_R]\,, \quad \mathbf{h}_i:=\mathbf{A}_L\mathbf x_i\,, \quad r_i:=f_T(\mathbf x_i,\mathbf 0)\,, \] then we have:
    \begin{itemize}
        \item $\mathbf{A}_R\mathbf x_i=-\mathbf{h}_i$ for all $i$;
        \item $r_i = r_{i+N}$ for $i \leq N$;
        \item $r_i < r_{i+1}$ for $i < N$.
    \end{itemize}
    \item[(3)] (Learning rates in right range to link and preserve within-list ranks.) Let $\Delta =r_N-r_1$ and $g =\min_{i<N}(r_{i+1}-r_i)$. Set \[D = \|\mathbf{h}_N-\mathbf{h}_{N+1}\|\, , \text{ and}
    \quad
    K_h = \max_{i<N}\frac{\max\{\|\mathbf{h}_{i+1}-\mathbf{h}_i\|,\|\mathbf{h}_{i+N+1}-\mathbf{h}_{i+N}\|\}}{r_{i+1}-r_i}\,.\] The effective learning rates $\lambda_1 = \alpha_1\eta$ and $\lambda_2 = \alpha_2\eta$ satisfy \[0 \leq \lambda_1<(2d)^{-1}\, , \quad \frac{1}{D^2} < \lambda_2 < \frac{g+2d\lambda_1\Delta\delta}{g\Delta D K_h+2d\lambda_1\Delta\delta D^2}\, .
\]
\end{enumerate}
Then, after one update on either $(\mathbf x_N,\mathbf x_{N+1})$ or $(\mathbf x_{N+1},\mathbf x_N)$, we have \[f_{T+1}(\mathbf x_i,\mathbf x_j) < 0 < f_{T+1}(\mathbf x_j,\mathbf x_i)\] for every $1 \leq i < j \leq 2N$.
\end{theorem}
\begin{proof}
The proof proceeds in two steps. First, we analyze the effect of the update on the readout and show that the concatenated ranking due to only this readout update is correct. However, this part of the proof does not incorporate the update from the first layer. In the second step, we analyze the effect of the update on the first layer and ensure it does not interfere with the readout-updated ranking.

We provide some notation to start. Let $\mathbf{b}_T =\mathbf 1+\alpha_2\mathbf{P}_2(T)$ and $\mathbf{z}=\mathbf{h}_N-\mathbf{h}_{N+1}$. By assumption (2), we have \[\mathbf{A}_T\mathbf{x}_{IJ} = \mathbf{h}_I -\mathbf{h}_J\, ,\] and 
\[\mathbf{b}_T\mathbf{h}_I=r_I\, .\] In particular, $\|\mathbf{z}\|=D$ and
$\mathbf{b}_T\mathbf{z}=\Delta$. Let $\mathbf{q}$ denote the concatenated input of the presented linking pair, and assume that the linking pair is
$(\mathbf x_N,\mathbf x_{N+1})$. The proof for reverse case is symmetric. The pre-update hidden activity and output are $\mathbf{z}$ and $\Delta$, respectively. By assumption (2), the model response is incorrect, and the readout update is
\[\mathbf{b}_{T+1}=\mathbf{b}_T-\lambda_2\Delta\mathbf z^\top\,.\]

We first examine the effect of the readout update on the original hidden representations. We will show that the readout-updated scores $\widetilde{r}_I=\mathbf{b}_{T+1}\mathbf{h}_I$ have the desired concatenated ordering. For the linking pair, we have
\begin{align*}
\widetilde{r}_{N+1}-\widetilde{r}_N
&=-\mathbf{b}_{T+1}\mathbf{z}\\
&=-\Delta+\lambda_2\Delta D^2\\
&=\Delta(\lambda_2D^2-1).
\end{align*}
Thus the readout update reverses the linking pair precisely when $\lambda_2D^2>1$. Write $H:=\lambda_2D^2-1$, so that $\Delta H$ is the resulting gap between the lists' boundary items.

Next, consider adjacent items $I<J$ in either list. By Cauchy--Schwarz and the definition of $K_h$,
\begin{align*}
\widetilde{r}_J-\widetilde{r}_I
&=(r_J-r_I)-\lambda_2\Delta\mathbf{z}^\top(\mathbf{h}_J-\mathbf{h}_I)\\
&\ge (r_J-r_I)-\lambda_2\Delta D\|\mathbf{h}_J-\mathbf{h}_I\|\\
&\ge (1-\lambda_2\Delta D K_h)(r_J-r_I).
\end{align*}
We therefore write $c:=1-\lambda_2\Delta D K_h$ for the guaranteed fraction of each within-list gap retained by the readout update. Assumption (3) gives $H>0$ and \[gc>2d\lambda_1\Delta\delta H\ge0.\]
In particular, $c>0$, so the $\widetilde{r}_I$ preserve both within-list orders and separate their boundary items. Their within-list gaps are at least $gc$, and their cross-list gaps are at least $\Delta H$. It remains to show that the first-layer update cannot change these inequalities.

Since $\mathbf{b}_{T+1}\mathbf{z}=-\Delta H$, the exact post-update output is
\begin{equation}\label{eq:perfectlinkingpostupdate}
\begin{aligned}
f_{T+1}(\mathbf{x}_I,\mathbf{x}_J)
&=\mathbf{b}_{T+1}(A_T-\lambda_1\mathbf{z}\mathbf{q}^\top)\mathbf{x}_{IJ}\\
&=\widetilde{r}_I-\widetilde{r}_J+\lambda_1\Delta H\,\mathbf{q}^\top\mathbf{x}_{IJ}.
\end{aligned}
\end{equation}
Thus the first-layer update contributes an additional term that we must show cannot change the readout-updated rank inequalities. The analysis is split between 2 cases: within-list trials and non-linking cross-list trials.
\begin{itemize}
\item (Within-list) First consider $I<J$ in the same list. Then assumption (1) gives
\begin{align*}
|\mathbf{q}^\top\mathbf{x}_{IJ}|&
=|\mathbf{x}_N^\top\mathbf{x}_I+\mathbf{x}_{N+1}^\top\mathbf{x}_J| \\
& \leq 2d\delta \, .
\end{align*}
For $\mathbf{x}_{JI}$, one item may overlap, so we obtain a lower bound 
\begin{align*}
    \mathbf{q}^\top\mathbf{x}_{JI} & =\mathbf{x}_N^\top\mathbf{x}_J+\mathbf{x}_{N+1}^\top\mathbf{x}_I \\
    & \geq -2d\delta \, .
\end{align*}
Since $|\widetilde{r}_{I}-\widetilde{r}_J| \geq gc$ for within-list items, these bounds together with Eq. \ref{eq:perfectlinkingpostupdate} give
\begin{align*}
f_{T+1}(\mathbf{x}_I,\mathbf{x}_J)
&\leq-gc+2d\lambda_1\Delta\delta H<0,\\
f_{T+1}(\mathbf{x}_J,\mathbf{x}_I)
&\geq gc-2d\lambda_1\Delta\delta H>0.
\end{align*}
Thus every within-list comparison remains correct.

\item (Cross-list) Next consider $I\le N<J$. As shown above, we have  $|\widetilde{r}_I-\widetilde{r}_J| \geq \Delta H$. For cross-list trial types, the overlap can include the presented linking pair itself, so we instead use the upper bound
\[
|\mathbf{q}^\top\mathbf{x}_{IJ}|
\le\|\mathbf{q}\|\,\|\mathbf{x}_{IJ}\|=2d,
\]
which also holds for $\mathbf{x}_{IJ}$. This inequality together with Eq. \ref{eq:perfectlinkingpostupdate} yields
\begin{align*}
f_{T+1}(\mathbf{x}_I,\mathbf{x}_J)
&\leq -\Delta H + 2d\lambda_1\Delta H = -\Delta H(1-2d\lambda_1)<0\,,\\
f_{T+1}(\mathbf{x}_J,\mathbf{x}_I) & \geq \Delta H - 2d\lambda_1\Delta H = \Delta H(1-2d\lambda_1) > 0\, .
\end{align*}
The last inequalities follow from $H>0$ and $\lambda_1<(2d)^{-1}$ (assumption (3)).
\end{itemize}
Thus, a fully correct concatenated rank ordering is established.
\end{proof}
This theorem shows that given a sufficiently good within-list TI learner and the right learning rates, the model can list-link. Empirically, we have seen that on average, the rank representations of items in each list satisfy assumption (2), so assumption (3) is where meta-training or hyperparameter tuning is necessary. Note that $D$ must be greater than $\Delta K_h$ in order have a non-empty interval for $\lambda_2$. We leave the relaxing of the assumptions in the proof to future work.

\paragraph{An algorithm for TI-L.}
Algorithm \ref{algorithm:til} abstracts the memory-recall structure described above. We can think of $\mathbf{P}_1(t)$ as encoding a position-dependent hash table (line 1), whose keys are items and whose values are "scaled" items, which are keys for accessing rank representations (line 2). The scalar determines how strongly to consider the rank of that item for its response. As in Algorithm \ref{algorithm:ti}, the $\mathbf{R}_L$ and $\mathbf{R}_R$ vectors represent the position-dependent ranks. If the difference between two ranks is incorrect (line 5), then we update the memories following our analysis in Eq. \ref{eq:tilp1updatedecomp} (lines 6-8). Line 9 updates the readout, and lines 10-11 are simply using an abstracted version of the idealized network to store the updated the position-dependent ranks.

\begin{algorithm}[t]
\caption{TI-L Algorithm}
\begin{algorithmic}[1]
\Require Number of items $N$
\Require Randomly initialized rank vector
         $\mathbf R\in\mathbb R^N$
\Require Learning rates $\lambda_1,\lambda_2>0$
\Require Episode $\mathcal{E} = \{(i_1, j_1), (i_2, j_2), \ldots, (i_T, j_T)\}$, \Statex \qquad\qquad\qquad\qquad\qquad where $|i_k-j_k| = 1$ and $1 \leq i_s,j_t \leq N$

\State $\mathbf M_L=\mathbf M_R=\mathbf 0$
\State $\mathbf R_L=\mathbf R_R=\mathbf R$
\Statex

\For{$(i,j)$ in $\mathcal E$}

    \State $\rho\gets\mathbf R_L[i]-\mathbf R_R[j]$

    \If{($\rho \leq 0$ and $i > j$) or ($\rho > 0$ and $i < j$)}
        \Statex
        \State $\mathbf a\gets
        \mathbf e_i-\mathbf e_j+
        \mathbf M_L[i,:]^\top+
        \mathbf M_R[j,:]^\top$
        \Statex
        \State $\mathbf M_L[i,:]\gets
        \mathbf M_L[i,:]-\lambda_1\mathbf a^\top$ \Comment{Update $\mathbf{P}_1$}

        \State $\mathbf M_R[j,:]\gets
        \mathbf M_R[j,:]-\lambda_1\mathbf a^\top$
        \Statex
        \State $\mathbf R\gets
        \mathbf R-\lambda_2\rho\mathbf a$ \Comment{Update $\mathbf{P}_2$}
        \Statex
        \State $\mathbf R_L\gets
        (I+\mathbf M_L)\mathbf R$ \Comment{Update ranks}

        \State $\mathbf R_R\gets
        (I-\mathbf M_R)\mathbf R$

    \EndIf
\EndFor

\end{algorithmic}
\label{algorithm:til}
\end{algorithm}

\section{Additional emergent learning-generalization abilities ("meta-test" behaviors) }\label{appendixsec:autohebbextraresults}
In this section, we report additional results indicating that AutoHebb networks have emergent learning and generalization abilities, i.e.\ abilities not explicitly optimized during meta-learning. These abilities were assessed by "meta-testing" models (both AutoHebb and other models):\ that is, presenting models with lifetime experiences (new tasks, consisting of systematically different sequences of trials) that differ from those presented to the model during meta-learning or hyperparameter tuning. As follows we report results testing three representative models (OML, ANML, and 1-neuromodulator AutoHebb) to evaluate three such emergent abilities.

\subsection{Length generalization in TI}\label{appendixsubsec:lengthgen_ti}
In humans and animals, the ability to perform TI is thought to be based on an abstractable ability to perform serial ordering that extends beyond a specific number of items, notwithstanding limitations due solely to memory capacity. In our model setup, meta-learning (or hyperparameter tuning) was conducted only on episodes that presented 8 items; thus it remains an open question whether models can `generalize' their ability to perform TI to larger numbers of items. In testing three representative models (OML, ANML, and AutoHebb), we found that AutoHebb showed an exceptional ability to generalize to larger numbers of items, despite the ability of both OML and ANML to perform TI adequately. The results in this section are reported for training phases consisting of $3(N-1)$ trials (N is number of items), where the $2(N-1)$ different training trial types (covering all `adjacent' item pairs, e.g.\ 14 item pairs for 8-item TI) are presented in random order, followed by an additional $N-1$ randomly sampled training trials.

\begin{figure}[!htbp]
    \centering
    \includegraphics[width=1.0\linewidth]{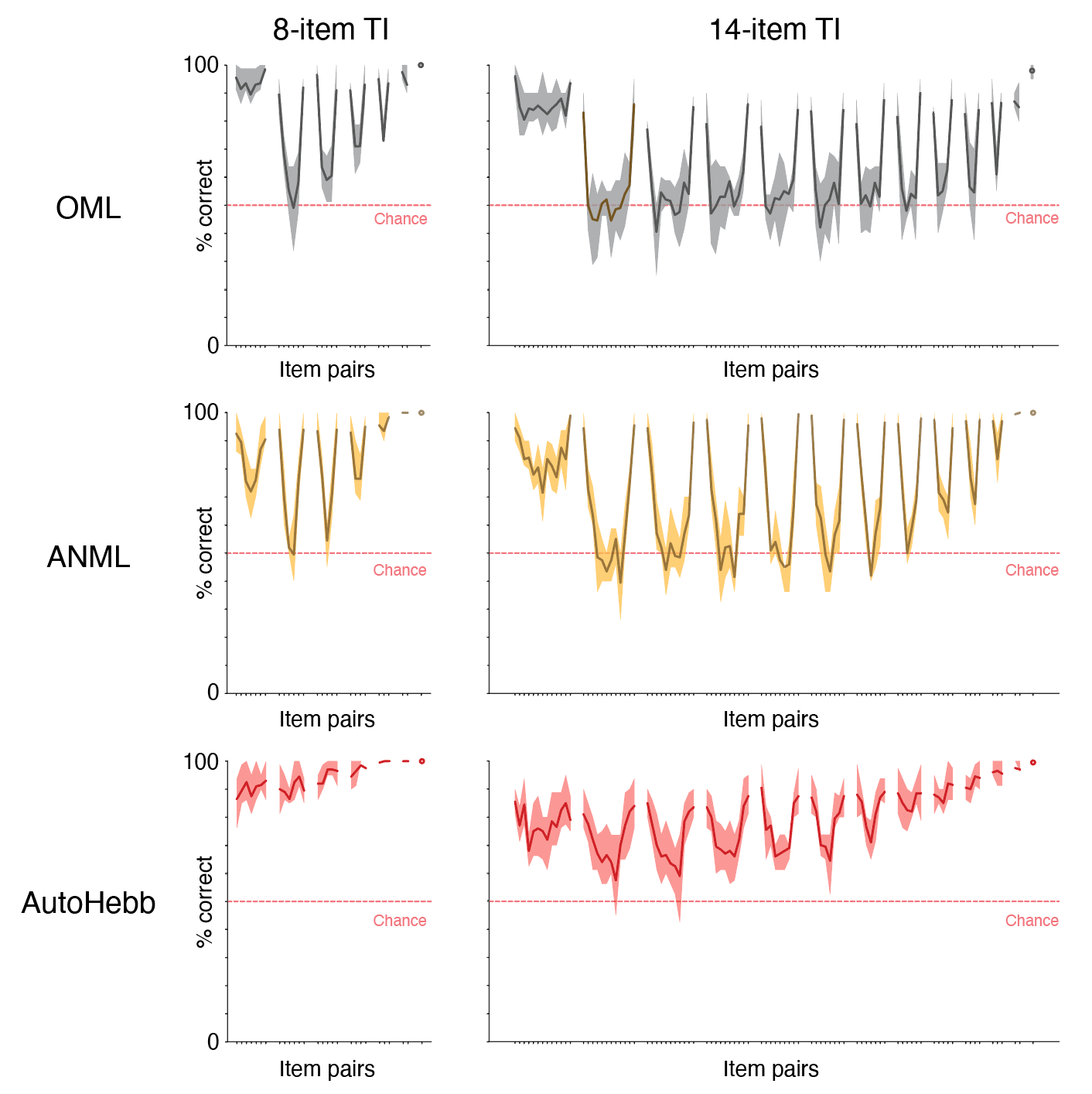}
    \caption{\textbf{TI length generalization with 14 items.} Plots follow the format based on Figure \ref{fig:ti_task}, where item pairs (on x-axis) are grouped and organized from left to right with increasing symbolic distance. Each model was meta-trained on 8 items (left) and subsequently evaluated (after meta-training) on 14 items (right). Note that ANML and OML show near-chance performance on many inferential trial types (item pairs) except for item pairs with larger symbolic distances. AutoHebb is the only method with above chance performance on all inferential trial types.}
    \label{supplementfig:length_gen14}
\end{figure}
\FloatBarrier

\begin{figure}[!htbp]
    \centering
    \includegraphics[width=1.0\linewidth]{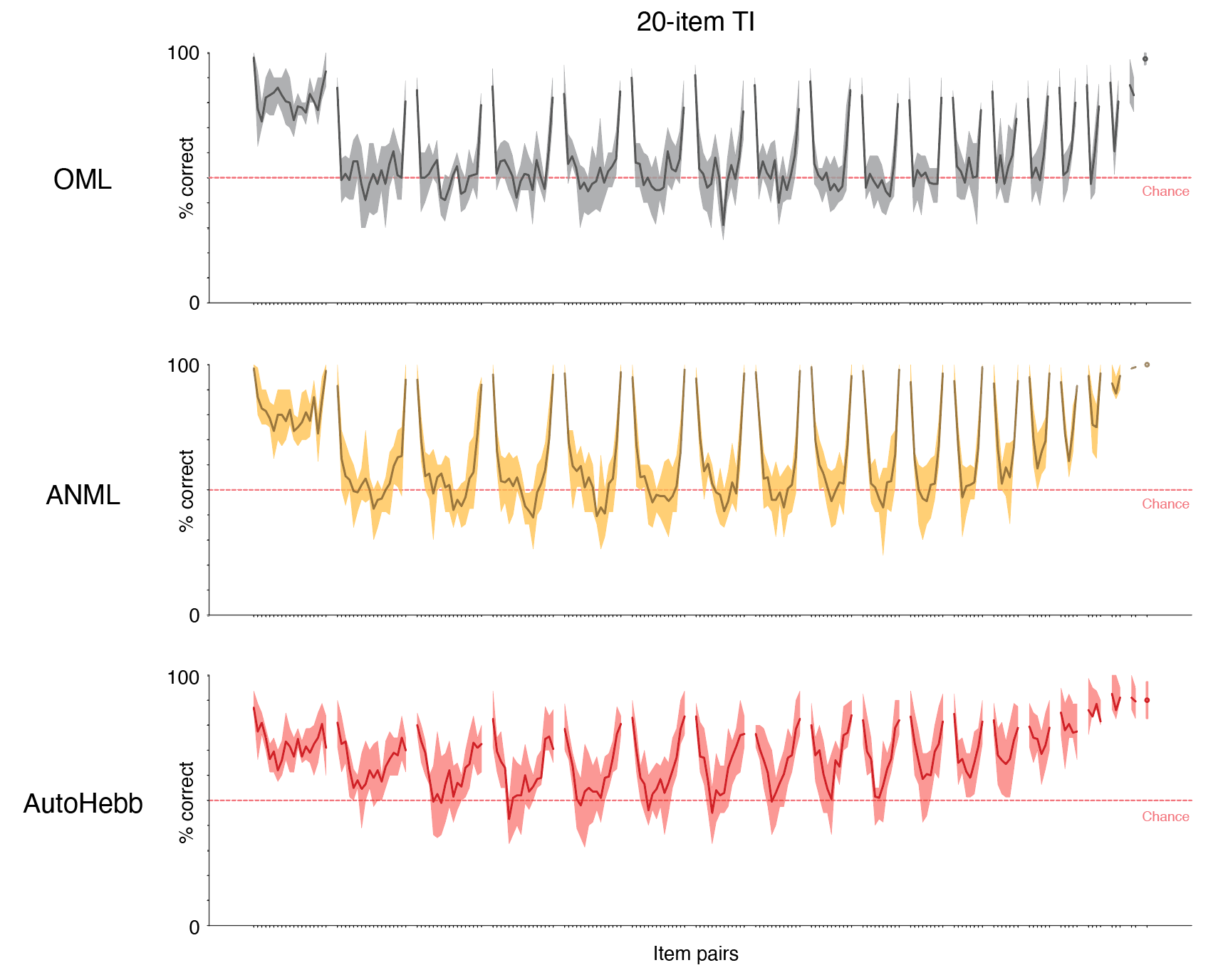}
    \caption{\textbf{TI length generalization with 20 items.} Plots follow the format based on Figure \ref{fig:ti_task}, where item pairs grouped from left to right with increasing symbolic distance. All models are meta-trained on only 8 items (see Figure \ref{supplementfig:length_gen14}, left) and evaluated on 20 items. Note that ANML and OML show relatively weaker performance than AutoHebb on inferential trials, though on this more difficult task AutoHebb chance performance on some inferential trials, indicating limitation in length generalization and/or memory capacity (e.g.\ compare to 14 items Figure \ref{supplementfig:length_gen14}).}
    \label{supplementfig:length_gen20}
\end{figure}
\FloatBarrier

\begin{figure}[!htbp]
    \centering
    \includegraphics[width=0.75\linewidth]{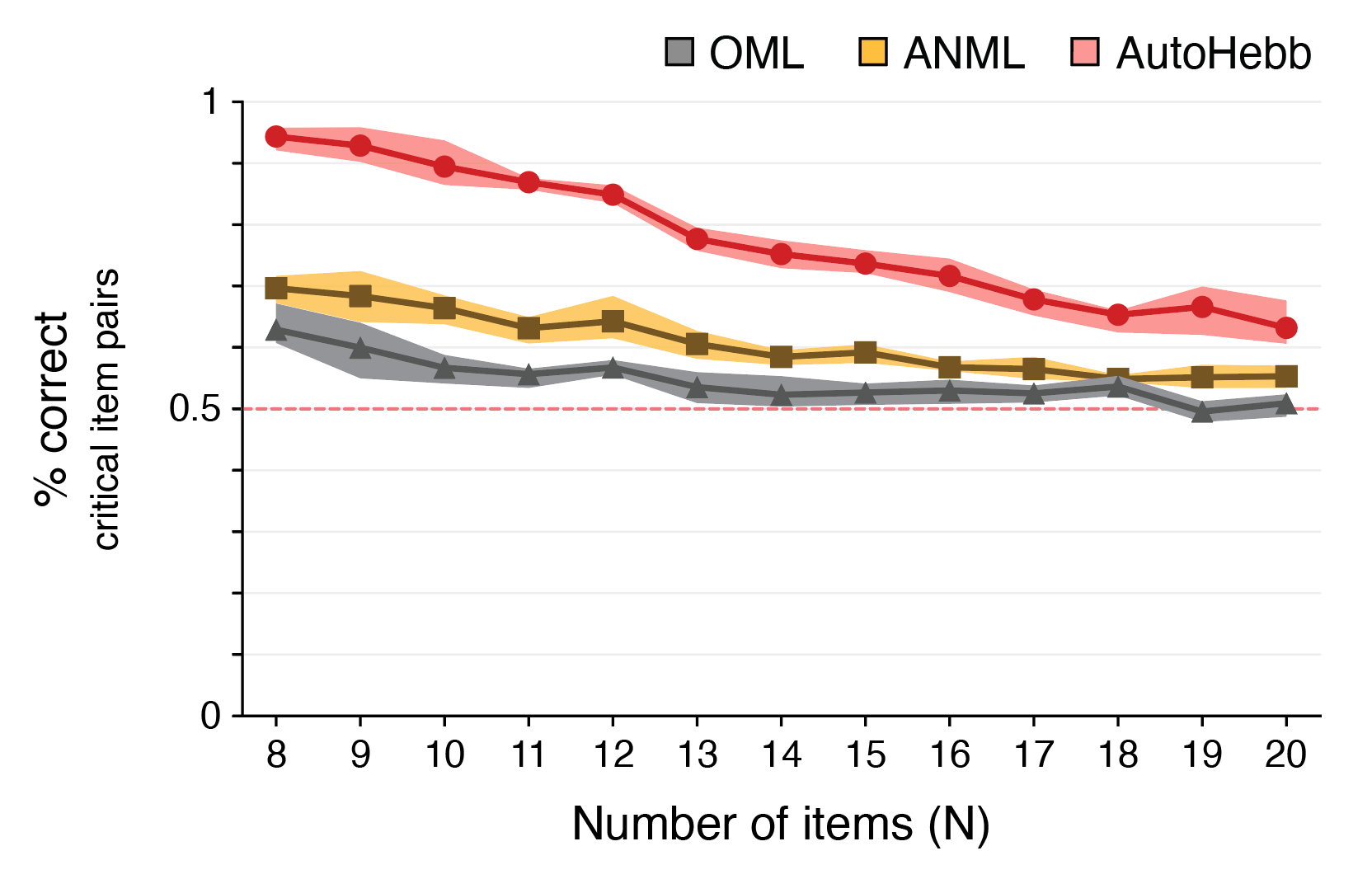}
    \caption{\textbf{Summary plot of length generalization in TI (8 to 20 items).} Performance shown for inferential ("0-shot" inference) trials only. AutoHebb consistently outperforms OML and ANML across all numbers of items.}
    \label{supplementfig:length_gen_overall}
\end{figure}
\FloatBarrier

\subsection{Length generalization in TE}\label{appendixsubsec:lengthgen_te}
As with TI, the ability to perform TE in humans and animals is thought to be based on an abstractable ability to perform grouping that extends beyond a specific number of items per group ($M$; generally $M = 3$ in experimental tests), and beyond a specific number of groups ($Q$), the latter of which is indeed commonly experimentally tested in TE-based experimental studies (e.g.\ \cite{shohamy2008integrating,zeithamova2012thehippocampus,schlichting2015learning}). In our model setup, meta-learning (or hyperparameter tuning) was conducted solely on episodes that presented TE with $M=3$ and $Q=3$; thus it remains an open question whether models can generalize their ability to perform TE to larger values of $M$ and $Q$. In testing three representative models (OML, ANML, and AutoHebb), we found that AutoHebb showed an exceptionally strong ability to generalize to larger numbers of items than ANML (note that OML did not perform the base TE task). However, ANML showed a competitive ability to perform TE on larger numbers of groups, with AutoHebb outperforming ANML on $Q = 3$ and $4$, but ANML outperforming AutoHebb on $Q > 5$. The results in this section are reported for training phases consisting of all training trial types presented exactly once in random order.

\begin{figure}[!htbp]
    \centering
    \includegraphics[width=1.0\linewidth]{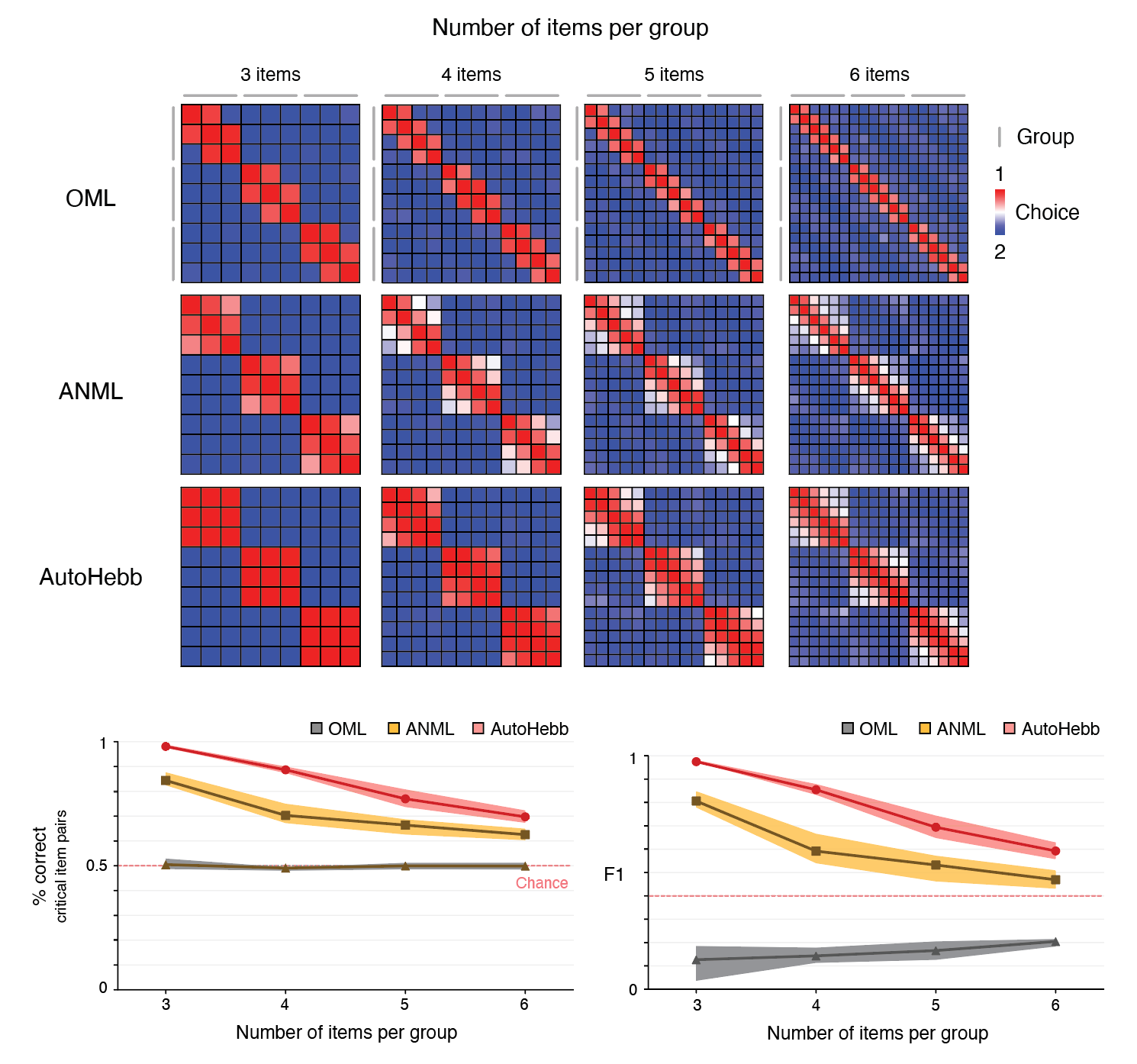}
    \caption{\textbf{Length generalization on TE with respect to numbers of items per group.} Top: Heatmaps show average model choices (1 for same group, 2 for different group, see Figure \ref{fig:te_task}), from 3 items per group to 6 items per group (left to right). In contrast to OML and ANML, AutoHebb generalizes to 4 items per group, though has limited ability on more than 4 items. Bottom: accuracy (calculated the same way as in Section \ref{subsec:zeroshotresults}) and F1 plots, where F1 is computed for same-group ZI trials.}
    \label{supplementfig:te_item_length_gen}
\end{figure}
\FloatBarrier

\begin{figure}[!htbp]
    \centering
    \includegraphics[width=1.0\linewidth]{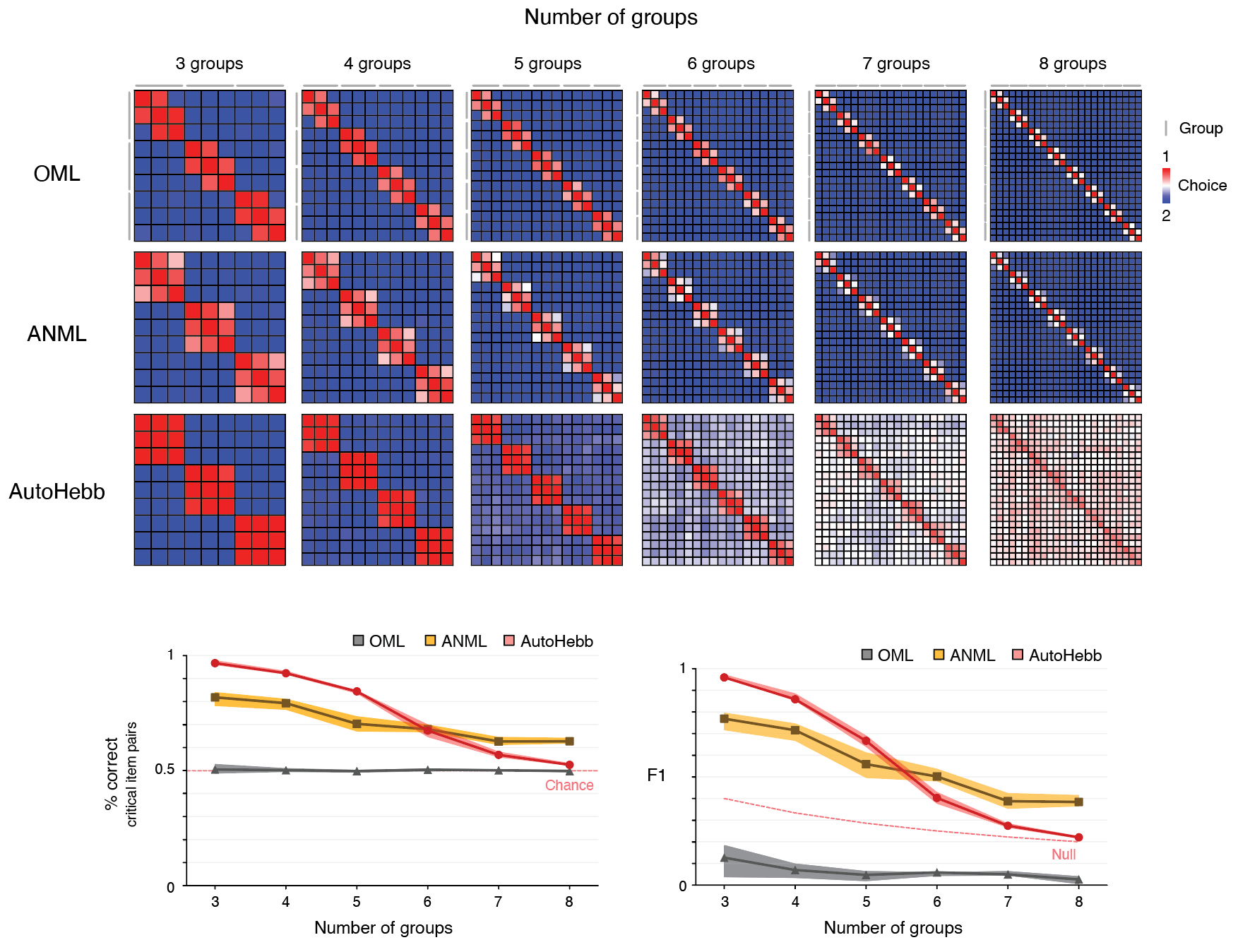}
    \caption{\textbf{Length generalization on TE with respect to numbers of groups.} Top: Heatmaps show average model choices (1 for same group, 2 for different group, see Figure \ref{fig:te_task}), from 3 groups of 3 items to 6 groups (left to right). AutoHebb shows successful ZI ability up to 6 groups. Notably, though ANML shows weaker generalization up to 5 groups, its overall accuracy outperforms the other two models after 6 groups. Generally, AutoHebb maintains its same-group ZI ability while its different-group accuracy decays, even for training trials. In contrast, ANML maintains its different group bias and loses its same group accuracy, even on training trials. These behaviors suggest different underlying inductive biases, as would be expected given the models' different architecture and learning rules. OML has similar behavioral bias as ANML, though OML cannot infer on same-group trials at any group size. Bottom: accuracy (calculated the same way as in Section \ref{subsec:zeroshotresults}) and F1 plots, where F1 is computed for same-group ZI trials.}
    \label{supplementfig:te_group_length_gen}
\end{figure}
\FloatBarrier

\subsection{Robustness to trial overexposure in TI }\label{appendixsubsec:masspresentation}
A critical aspect of natural (and successful) learning is self-regulation, or the ability to control or stabilize plastic changes that would otherwise destabilize or interfere with important prior learning \cite{kudithipudi2022biological,ParisiContinuallifelonglearning2019,wang2024comprehensivesurveycontinuallearning}. Intriguingly, in  TI tasks, experimental work has found that living subjects successfully retain the ability to report previously learned serial order when a single item pair is shown a large number of times (trial overexposure or "massed presentation", \cite{jensen2019discovering,jensen2017transitive}); notably, this behaviorally exhibited instance of stable learning is not expressed in some models of TI \cite{jensen2019discovering}. In this way, testing this behavioral pattern (which is not optimized during meta-training in our model setup) evaluates a non-trivial ability to self-regulate learning, moreover one grounded in empirical findings in living subjects.

Similarly to previous tests, we ran three representative meta-learned models (OML, ANML, and AutoHebb) on a modified trial sequence in TI, namely where 100 additional trials (all of which were the item pair DE) were presented after the 20 adjacent-only trials. We found that OML and ANML generally showed substantial decrease in performance on trial types involving D or E (besides DE itself), with some trial types showing decrease in performance to chance or below chance levels. In contrast, AutoHebb showed only minor changes in performance across all trial types. Examination of the neuromodulator signal in AutoHebb network revealed that the network rapidly downregulated neuromodulation (and thus plasticity) within approximately 20 trials of (repeated) exposure to DE.

\begin{figure}[!htbp]
    \centering
    \includegraphics[width=1.0\linewidth]{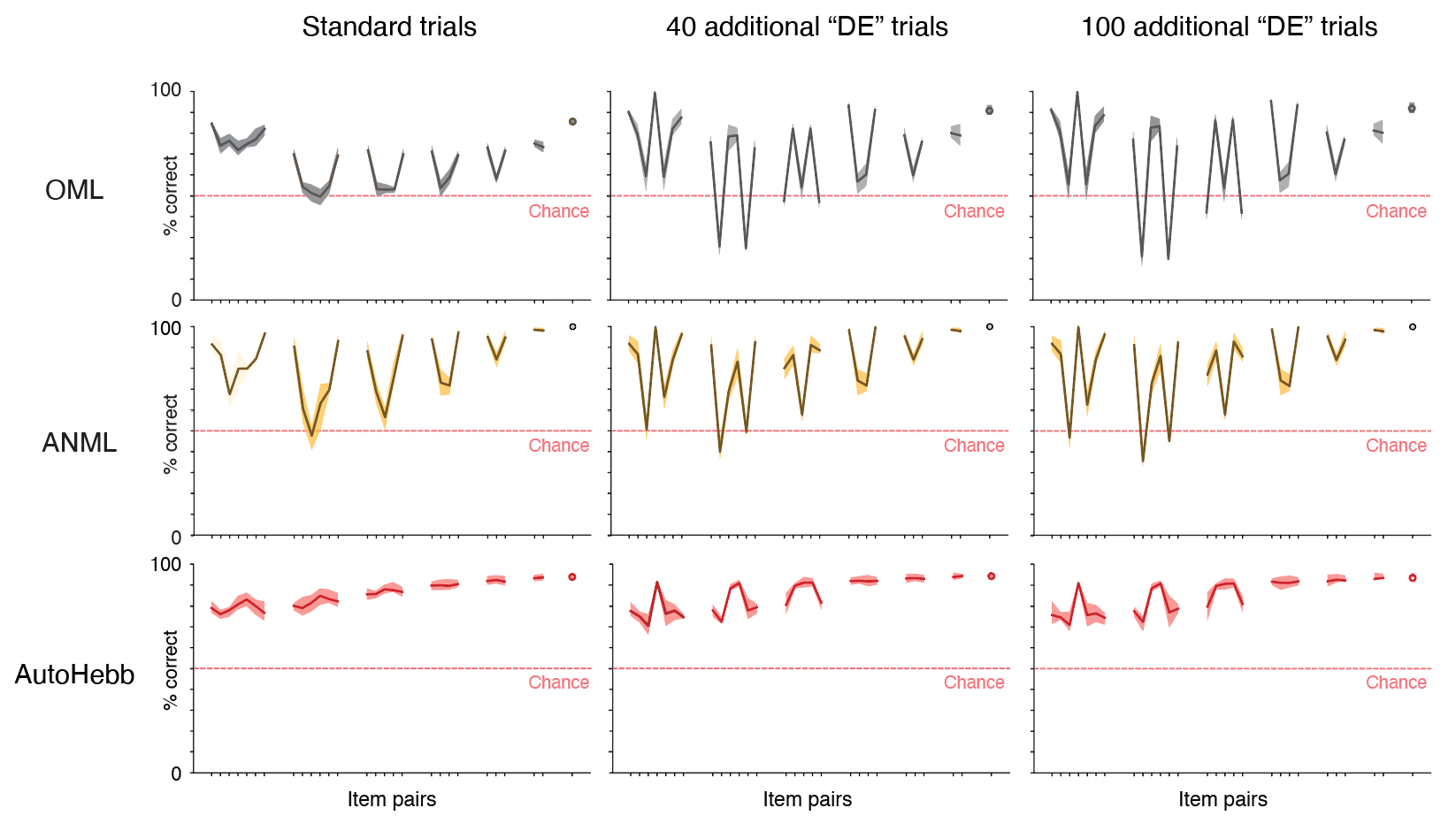}
    \caption{\textbf{Performance of three representative models on TI before vs.\ after overexposure of a single trial type (item pair DE).} Plot follows the format of Figure 4, where trial types are ordered by symbolic distance. (Left) Performance after the standard 20 training trials are shown, (Middle) Performance after an additional 40 trials of DE are shown, (Right) Performance after an additional (from standard training trials) 100 trials of DE are shown. Results are averaged over 100 episodes with 2 s.e. bands. Note that for OML and AMNL, several trial types approach (or reach) chance performance, indicating lack of robustness to the overexposure manipulation; in contrast, AutoHebb trials remain well above chance.}
    \label{supplementfig:mass_symb}
\end{figure}
\FloatBarrier

\begin{figure}[!htbp]
    \centering
    \includegraphics[width=1.0\linewidth]{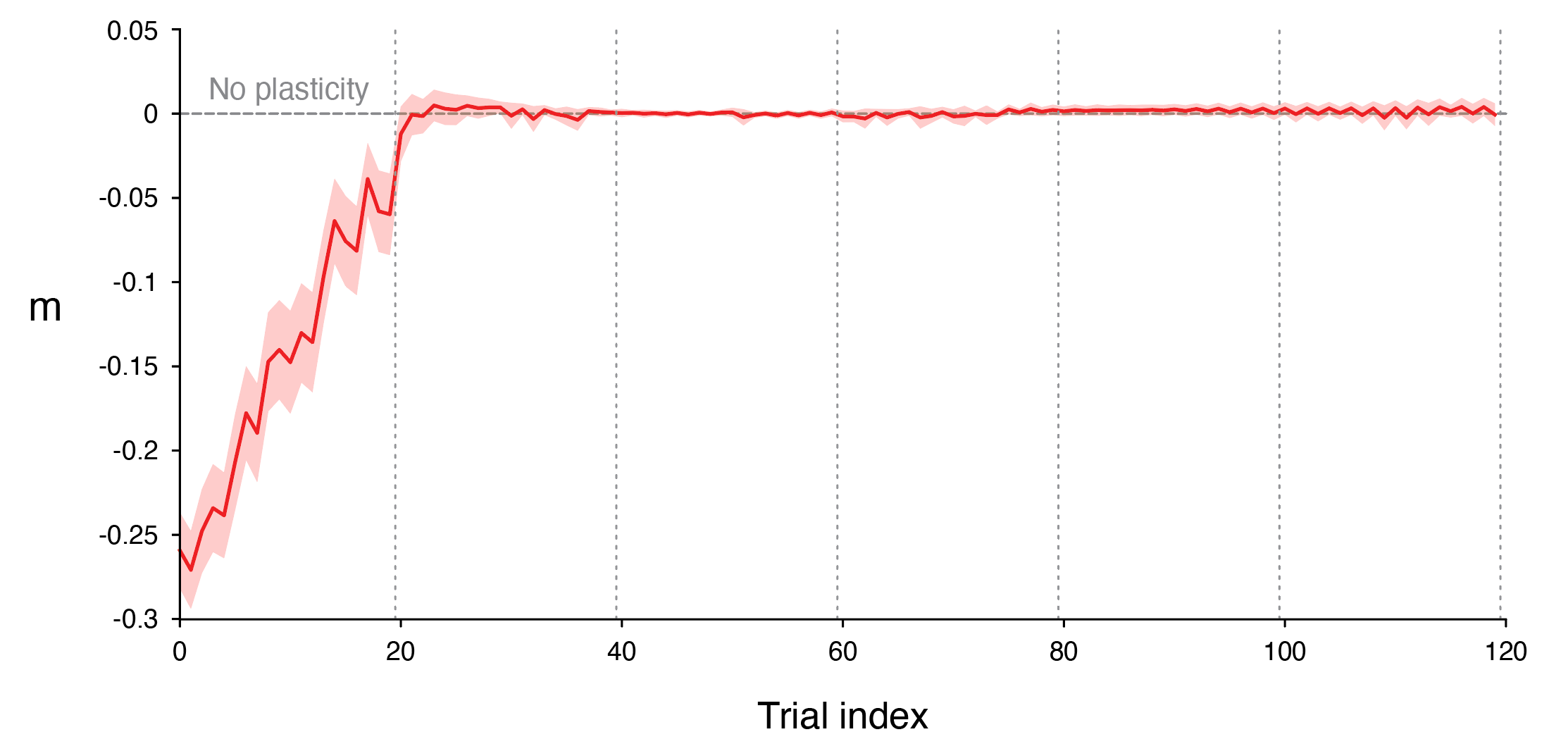}
    \caption{\textbf{Neuromodulator in AutoHebb network during overexposure to single trial type (item pair DE).} The x-axis shows the trial index of the repeat presentation of item pair DE. Note that the neuromodulator is fully suppressed after approximately 20 trials. Results are averaged over 100 episodes, and show mean ± 2 s.e.}
    \label{supplementfig:mass_neuromod}
\end{figure}
\FloatBarrier

\clearpage
\section{Supplementary figures}\label{appendixsec:extrafigures}

\begin{figure}[!htbp]
    \centering
    \includegraphics[width=1\linewidth]{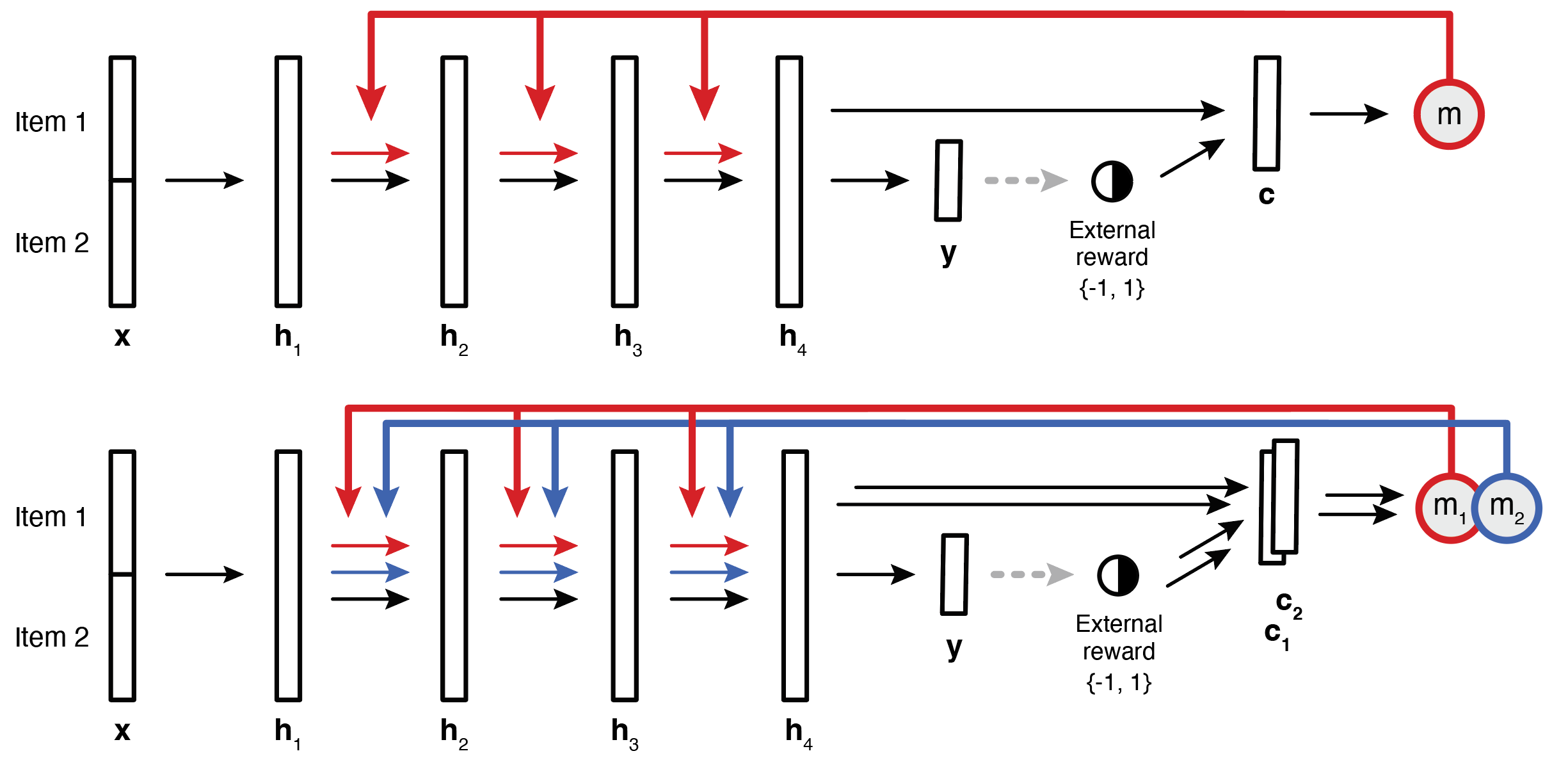}
    \caption{\textbf{AutoHebb networks benchmarked in the present study:\ 1- and 2-neuromodulator AutoHebb (AH-1NM and AH-2NM).} For both models, the main network was chosen to be a multi-layer perceptron (MLP) with four hidden layers; the models differed by the number of neuromodulators (and corresponding control network; in the 2-neuromodulator network, the control networks were parallel, as diagrammed). AH-2N is loosely inspired by the existernce of multiple neuromodulatory systems in the brain, and was included in benchmarking to evaluate the possibility that parallel neuromodulated plastic architecture would enhance task performance. Note that each set of plastic weights at a given layer in the 2-neuromodulator network were summed (Eq. \ref{eq:twoneuromodulators}). Top, AH-1N, Bottom, AH-2NM. Note that additional AutoHebb networks were also investigated in this study, and are indicated accordingly were applicable. Implementation details can be found in Appendix \ref{appendixsec:autohebbtraining}.}
    \label{appendixfig:autohebb_two_arch}
\end{figure}
\FloatBarrier

\begin{figure}[!htbp]
    \centering
    \includegraphics[width=1.0\linewidth]{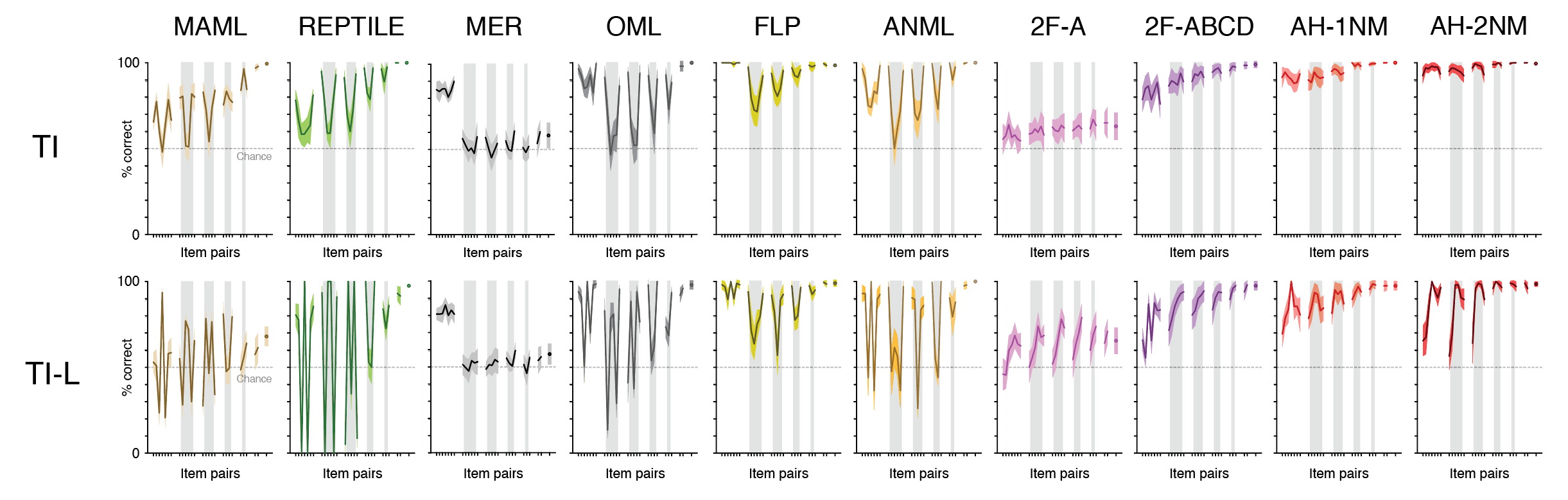}
    \caption{\textbf{TI and TI-L tasks: model performance by trial type.} Plots follow the convention in Figure \ref{fig:ti_task}c, where within each plot, trial types (item pairs) are organized by symbolic distance. Plotted are model outputs for an example model instance for each indicated model. Note that on TI-L, various models such as MAML, Reptile, OML, ANML show overtly inconsistent performance across trial types, in contrast to FLP, 2F-A, 2F-ABCD, AH-1NM, and AH-2NM.}
    \label{fig:ti_supp_symb}
\end{figure}
\FloatBarrier

\begin{figure}[!htbp]
    \centering
    \includegraphics[width=1.0\linewidth]{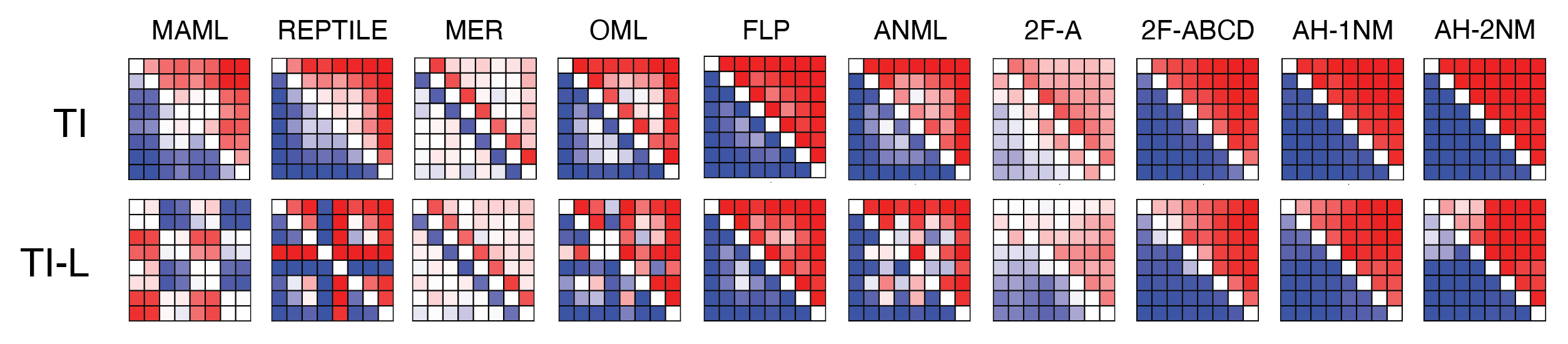}
    \caption{\textbf{TI and TI-L tasks: average model outputs.} Plots follow the convention in Figures \ref{fig:ti_task} and \ref{fig:linking_tasks}. Plotted are model outputs for an example model instance for each indicated model.}
    \label{fig:ti_supp_grid}
\end{figure}
\FloatBarrier

\begin{figure}[!htbp]
    \centering
    \includegraphics[width=1.0\linewidth]{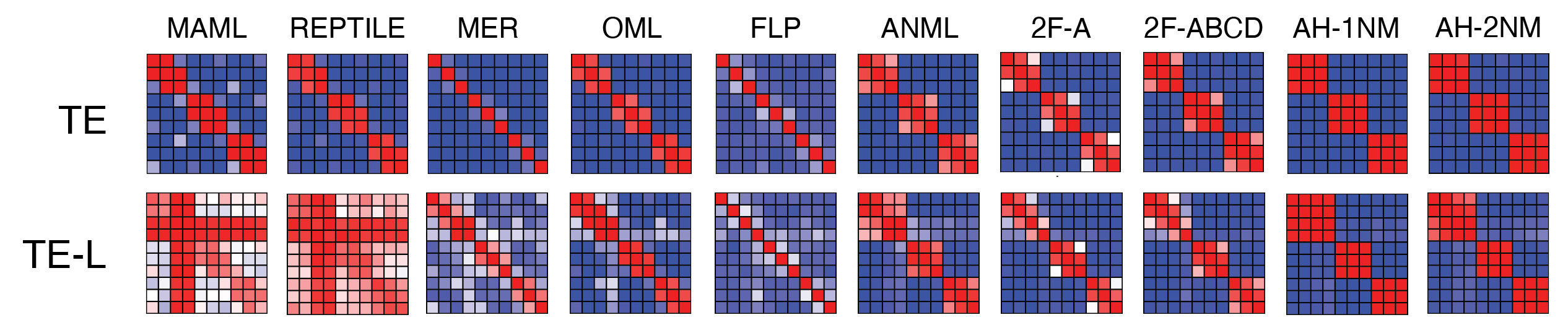}
    \caption{\textbf{TE and TE-L tasks: average model outputs.} Plots follow the convention in Figures \ref{fig:te_task} and \ref{fig:linking_tasks}. Plotted are model outputs for an example model instance for each indicated model.}
    \label{fig:te_supp_grid}
\end{figure}
\FloatBarrier

\begin{table}
  \caption{\textbf{Model performance on TE by trial type (See Figure \ref{fig:te_task}a).} Note that models generally show a bias to output "different" by default on the task, resulting in high performance on different-group trials.\\}
  \label{table:zeroshotbreakdown}
  \centering
  \setlength{\tabcolsep}{5pt}
  \begin{tabular}{lcc}
    \toprule
    Model & Diff-group & Same-group \\
    \midrule
    MAML \cite{finn2017modelagnosticmetalearning} & $0.731 \pm 0.010$ & $0.273 \pm 0.024$ \\
    Reptile \cite{nichol2018firstordermetalearning} & $0.822 \pm 0.042$ & $0.196 \pm 0.043$ \\
    MER \cite{riemer2018learning} & $0.948 \pm 0.008$ & $0.042 \pm 0.011$ \\
    OML \cite{javed2019metalearning} & $0.929 \pm 0.007$ & $0.077 \pm 0.008$ \\
    FLP \cite{lindsey2020learningtolearn} & $0.820 \pm 0.015$ & $0.187 \pm 0.018$ \\
    ANML \cite{beaulieu2020learning} & $0.959 \pm 0.006$ & $0.703 \pm 0.008$ \\
    2F-A & $0.901 \pm 0.014$ & $0.533 \pm 0.099$ \\
    2F-ABCD & $0.947 \pm 0.009$ & $0.741 \pm 0.027$ \\
    \midrule
    AH-1NM & $0.964 \pm 0.032$ & $0.926 \pm 0.055$ \\
    AH-2NM & $0.953 \pm 0.021$ & $0.913 \pm 0.026$ \\
    \bottomrule
  \end{tabular}
\end{table}

\begin{table}
  \caption{\textbf{Model performance on TE-L by trial type (See Figure \ref{fig:linking_tasks}c).} Results are separated into "Original" trials -- ZI trials involving only the three originally -presented groups as shown in Figure \ref{fig:te_task}a -- and "New Item" trials -- ZI trials involving the new item and items not seen during the linking phase of training as shown in Figure \ref{fig:linking_tasks}c. In contrast to Table \ref{table:zeroshotbreakdown}, the linking trials reduce the bias towards predicting "different" on TE-L Original trials and, in some cases, improve the models' predictions. However, the majority of models have chance or below chance on one category of the TE-L New Item trials. \\}
  \label{table:linkingtebreakdown}
  \centering
  \setlength{\tabcolsep}{5pt}
  \begin{tabular}{lcc|cc}
    \toprule
    & \multicolumn{2}{c}{TE-L Original} & \multicolumn{2}{c}{TE-L New Item} \\
    Model & Diff-group & Same-group & Diff-group & Same-group \\
    \midrule
    MAML \cite{finn2017modelagnosticmetalearning} & $0.273 \pm 0.015$ & $0.724 \pm 0.020$ & $0.066 \pm 0.005$ & $0.935 \pm 0.011$ \\
    Reptile \cite{nichol2018firstordermetalearning} & $0.273 \pm 0.059$ & $0.748 \pm 0.057$ & $0.109 \pm 0.037$ & $0.890 \pm 0.039$ \\
    MER \cite{riemer2018learning} & $0.505 \pm 0.159$ & $0.493 \pm 0.156$ & $0.478 \pm 0.153$ & $0.508 \pm 0.167$ \\
    OML \cite{javed2019metalearning} & $0.763 \pm 0.014$ & $0.240 \pm 0.024$ & $0.636 \pm 0.020$ & $0.381 \pm 0.035$ \\
    FLP \cite{lindsey2020learningtolearn} & $0.753 \pm 0.039$ & $0.247 \pm 0.037$ & $0.693 \pm 0.036$ & $0.302 \pm 0.039$ \\
    ANML \cite{beaulieu2020learning} & $0.847 \pm 0.016$ & $0.785 \pm 0.018$ & $0.729 \pm 0.020$ & $0.744 \pm 0.026$ \\
    2F-A & $0.887 \pm 0.020$ & $0.511 \pm 0.073$ & $0.822 \pm 0.014$ & $0.263 \pm 0.025$ \\
    2F-ABCD & $0.938 \pm 0.007$ & $0.658 \pm 0.021$ & $0.820 \pm 0.017$ & $0.322 \pm 0.025$ \\
    \midrule
    AH-1NM & $0.904 \pm 0.045$ & $0.940 \pm 0.039$ & $0.831 \pm 0.070$ & $0.892 \pm 0.054$ \\
    AH-2NM & $0.931 \pm 0.030$ & $0.883 \pm 0.030$ & $0.874 \pm 0.043$ & $0.799 \pm 0.047$ \\
    \bottomrule
  \end{tabular}
\end{table}

\begin{table}
  \caption{\textbf{Model performance on multi-task ZI on TE, by task order and trial type (See Figure \ref{fig:te_task}a).} Note a strong different-group bias appears for all models in both task orders (TI to TE and TE to TI). $^*$Only 3 runs are reported.\\}
  \label{table:multitaskbreakdown}
  \centering
  \setlength{\tabcolsep}{5pt}
  \begin{tabular}{lcc|cc}
    \toprule
    & \multicolumn{2}{c}{TI $\rightarrow$ TE} & \multicolumn{2}{c}{TE $\rightarrow$ TI} \\
    Model & Diff-group & Same-group & Diff-group & Same-group \\
    \midrule
    MAML \cite{finn2017modelagnosticmetalearning} & $0.765 \pm 0.014$ & $0.271 \pm 0.016$ & $0.735 \pm 0.012$ & $0.299 \pm 0.016$ \\
    Reptile \cite{nichol2018firstordermetalearning} & $0.856 \pm 0.006$ & $0.177 \pm 0.015$ & $0.809 \pm 0.008$ & $0.241 \pm 0.022$ \\
    MER \cite{riemer2018learning} & $0.845 \pm 0.013$ & $0.145 \pm 0.012$ & $0.623 \pm 0.023$ & $0.382 \pm 0.036$ \\
    OML \cite{javed2019metalearning} & $0.954 \pm 0.009$ & $0.050 \pm 0.005$ & $0.899 \pm 0.004$ & $0.111 \pm 0.012$ \\
    FLP \cite{lindsey2020learningtolearn} & $0.953 \pm 0.011$ & $0.050 \pm 0.010$ & $0.745 \pm 0.032$ & $0.259 \pm 0.031$ \\
    ANML \cite{beaulieu2020learning} & $0.929 \pm 0.015$ & $0.471 \pm 0.013$ & $0.898 \pm 0.014$ & $0.489 \pm 0.018$ \\
    2F-A & $0.929 \pm 0.019$ & $0.123 \pm 0.024$ & $0.814 \pm 0.042$ & $0.273 \pm 0.064$ \\
    2F-ABCD & $0.917 \pm 0.011$ & $0.122 \pm 0.017$ & $0.768 \pm 0.037$ & $0.306 \pm 0.071$ \\
    \midrule
    EWC \cite{kirkpatrick2017overcomingcatastrophic} & $0.508 \pm 0.015$ & $0.490 \pm 0.021$ & $0.513 \pm 0.007$ & $0.488 \pm 0.014$ \\
    SI \cite{zenke2017continuallearning} & $0.502 \pm 0.012$ & $0.494 \pm 0.022$ & $0.504 \pm 0.017$ & $0.507 \pm 0.019$ \\
    A-GEM \cite{chaudhry2018efficient} & $0.945 \pm 0.007$ & $0.065 \pm 0.008$ & $0.866 \pm 0.013$ & $0.158 \pm 0.004$ \\
    DER \cite{buzzega2020darkexperience} & $0.543 \pm 0.022$ & $0.461 \pm 0.021$ & $0.541 \pm 0.013$ & $0.442 \pm 0.012$ \\
    \midrule
    AH-1NM$^*$ & $0.661 \pm 0.504$ & $0.473 \pm 0.388$ & $0.644 \pm 0.413$ & $0.451 \pm 0.330$ \\
    AH-2NM & $0.930 \pm 0.028$ & $0.614 \pm 0.032$ & $0.911 \pm 0.026$ & $0.643 \pm 0.047$ \\
    \bottomrule
  \end{tabular}
\end{table}

\section{AutoHebb training}\label{appendixsec:autohebbtraining}
Given fixed parameters $\theta$, an episode consisting of input and output pairs $\mathcal{E} = \{(\mathbf{x}_1, \mathbf{y}_1), (\mathbf{x}_2, \mathbf{y}_2),\ldots, (\mathbf{x}_T, \mathbf{y}_T)\}$, and corresponding target outputs $\{\hat{\mathbf{y}}_1, \hat{\mathbf{y}}_2, \ldots, \hat{\mathbf{y}}_T\}$, we took the loss function for meta-learning to be binary cross-entropy loss (BCE): \[\mathcal{L}(\theta, \mathcal{E}) = -\frac{1}{T}\sum_{i=1}^T(1-\mathbf{y}_i)\log(1-\hat{\mathbf{y}}_i) + \mathbf{y}_i\log(\hat{\mathbf{y}}_i).\]

We generally trained models up to $100,000$ episodes or convergence based on meta-training loss. Since each inner loop episode consists of randomly sampled items from $\{-1,+1\}^d$, the probability of two different inner loop episodes using the same sets of items is neglibile. Thus, the criterion for convergence strictly used the meta-training loss.

Although we freeze plastic weights during ZI evaluations, we do not freeze plastic weights during the test phase in meta-training inner loops. Generally across all models in the paper, we empirically observed that allowing models to update parameters on the test trials in the inner loop led to better performance, AutoHebb included. Furthermore, in order to increase the speed and efficiency of learning ZI, the ZI test trials (i.e., excluding the trials seen during training) are weighted $3\times$ throughout meta-training as in \cite{miconi2025neural}. In fact, when zeroing meta-training loss for adjacent pairs in TI, AutoHebb achieved nearly perfect performance on TI and TI-L ZI. However, we found that the same approach for TE was detrimental, likely due to the fewer number of ZI test trials relative to TI.

For all tasks reported, our networks use a hidden dimension of $200$ and a $\tanh$ activation function with the exception of the network in Figure \ref{fig:multi_task_continual}, which is $512$-dimensional. The items are $100$-dimensional, i.e. $d=100$, making the trial inputs $200$-dimensional after concatenation. As discussed in Section \ref{subsec:autohebbarch}, the hidden layers for a 1-neuromodulator AutoHebb network is given by

\[\mathbf{h}_{i} = g((\mathbf{W}_i+(\mathbf{A}_i \odot\mathbf{P}_i))\mathbf{h}_{i-1} + \mathbf{b}_i).\]

However, for 2-neuromodulators, each hidden layer becomes

\begin{equation}\label{eq:twoneuromodulators}
\mathbf{h}_{i} = g((\mathbf{W}_i+(\mathbf{A}_i \odot\mathbf{P}_{i,1}) + (\mathbf{B}_i \odot \mathbf{P}_{i,2}))\mathbf{h}_{i-1} + \mathbf{b}_i)\,,
\end{equation}
where $\mathbf{B}_i$ is meta-trainable. Then the updates for each set of plastic weights are $\Delta\mathbf{P}_{i,j} = \eta_jm_j\mathbf{h}_i\mathbf{h}_{i-1}^\top$. Each neuromodulator $m_j$ has its own control network \[m_j = \tanh(\mathbf{W}_{\text{mod},j}\tanh(\mathbf{W}_{\text{hmod},j}\mathbf{h} + \mathbf{w}_{r,j}r + \mathbf{b}_{\text{mod},j}) + b_{\text{modout},j}).\]

Details for training are as follows: we use a learning rate of $0.0001$ for TE and multi-task models, and a learning rate of $0.00045$ for TI. We use the default Adam optimizer (\cite{kingma2017adammethodstochasticoptimization}) in PyTorch (\cite{paszke2019pytorchimperativestylehighperformance}) with eps set to $1e-6$ and gradient clipping within $[-2.0,2.0]$. The plastic weight decay term $\lambda$ is fixed to $1$, except in model used in Figure \ref{fig:multi_task_continual} (see Appendix \ref{appendixsubsec:multitaskswitchingfiguredetails}). Our batch size for the inner loop was 32 parallel episodes with different training trial orders and sampled test trials, but the item representations and ranking between items stayed the same. We found randomizing items across episodes between batches reduced efficiency, but such an approach may prevent overfitting. For the multi-task ZI learning, we used 2 gradient accumulation steps to achieve an effective batch size of 32, with batch splitting between both task orders. In order to prevent overfitting to task-learning regimes, we randomly sampled training task lengths between 20 to 100 training trials for TI and 63 to 120 trials for TE. We fixed those values for all episodes within the batch for both directions. The ZI testing phase always matched the order of tasks seen and used 30 test trials for each task. In TE, we used 30 test trials for each inner loop and 10 test trials for TI. We sampled 75\% of trials to be nonadjacent pairs for all tasks. Compute requirements for AutoHebb were relatively light compared to other methods, only requiring 8 GPU hours on an H100 for a 3-hidden layer network with 2 neuromodulators on TI to reach $100,000$ meta-training steps. The time increases to 12 GPU hours on a single H100 for TE due to its longer episodes. The multi-task models were the most compute-intensive, requiring 26 hours on a single H100. The only the multi-task ZI AutoHebb models reported in Table \ref{table:multitask} use 2-hidden layers, while all other implementations use 3-hidden layers. We found that the added depth significantly improved task performance in fewer meta-training episodes. Using more layers for the multi-task ZI model required gradient stabilization methods, which are beyond the scope of this work. Furthermore, the multi-task models required the use of gradient accumulation in order to avoid out-of-memory issues despite the small number of parameters due to backpropagation through time (BPTT). 

\section{Experimental details}\label{appendixsec:experimentdetails}
\subsection{Benchmark rationale}\label{appendixsubsec:benchmarkrationale}
In the present study, we had two main goals in the choice of benchmark models. The first was to compare AutoHebb to well-established and widely competitive learning models; the second was to provide performance comparison to models that are systematically comparable to AutoHebb in model components, and in so doing provide the equivalent of ablation analysis of AutoHebb. Both goals are critical, since with respect to the first goal we contend that AutoHebb (even despite being geared for understandability and analysis) is potentially competitive on a number of tasks; with respect to the second goal, we sought to provide model comparisons that would test the contribution of each of the three main model components of AutoHebb: in particular, meta-learning, Hebbian plasticity, and neuromodulation. As follows, we present our rationale along three main (overlapping) categories of models; we lastly highlight how a subset of these models constitute an effective ablation analysis of the main model components of AutoHebb.
\begin{enumerate}
    \item The first category of models are models that use meta-learning, a number of which are well-established in the machine learning literature. Importantly, these models also provide a critical point of contrast with AutoHebb, as these models (unlike AutoHebb) do not rely on Hebbian plasticity or neuromodulation. In this way, comparing AutoHebb to meta-learning models can help delineate any advantage or disadvantage conferred by neuromodulated Hebbian plasticity. The models we benchmarked are model-agnostic meta-learning (MAML \cite{finn2017modelagnosticmetalearning}), Reptile (\cite{nichol2018firstordermetalearning}), an online meta-learning algorithm (OML \cite{javed2019metalearning}), meta-experience replay (MER \cite{riemer2018learning}), a neuromodulated meta-learning algorithm (ANML \cite{beaulieu2020learning}), and feedback and local plasticity (FLP \cite{lindsey2020learningtolearn}). 
    
    \item The second category of models are 2-factor Hebbian models (feedforward variants of 2-factor Hebbian RNN models introduced in \cite{miconi18differentiableplasticity}), which also provide a critical point of contrast to AutoHebb. In particular, if AutoHebb's reliance on neuromodulation (global control of Hebbian plasticity) is critical to task performance, then models that retain Hebbian plasticity but lack neuromodulation should not perform comparably to AutoHebb. We therefore benchmarked 2-factor Hebbian models (in the same MLP architecture as AutoHebb's main network) that lacked the 3-factor (neuromodulatory) term essential to AutoHebb; in these models, to provide neccessary external feedback from each trial, external feedback signals (i.e.\ reward and non-reward) on a given trial was concatenated to the input vector and provided on the subsequent trial. Two 2-factor Hebbian models were tested: 2-factor-A (2F-A), which has a single Hebbian pre-post term in the plastic update, and 2-factor-ABCD (2F-ABCD, e.g.\ \cite{najarro2020meta}), which has pre-post (A), pre (B), post (C), and a constant (D) terms. In addition, FLP (\cite{lindsey2020learningtolearn}) also represents a 2-factor Hebbian learning model, though this model incorporates an external loss function and feedback connections. Note also that all of these 2-factor Hebbian models also make use of meta-learning to identify fixed weights and learning hyperparameters.
    \item The last category of models are weight regularization methods, which constitute a classical approach to CL. These are elastic weight consolidation (EWC \cite{kirkpatrick2017overcomingcatastrophic}), synaptic intelligence (SI \cite{zenke2017continuallearning}), dark-experience replay (DER \cite{buzzega2020darkexperience}), and averaged gradient episodic memory (A-GEM \cite{chaudhry2018efficient}). All of these methods update weights via gradient descent with an additional regularization term that decreases plasticity for weights that are important relative to past experiences in order to mitigate catastrophic forgetting. Note that, with the exception of DER, a major limitation of these models is that task identities need to be provided to the model. Note that for AutoHebb, we do not provide task cues.

\end{enumerate}
We highlight that a subset of the above models can be interpreted as ablations of one or two of the three main model components of AutoHebb. In particular, the meta-learning models that use gradient descent to learn (MAML, Reptile, OML, MER, ANML, and FLP) can be interpreted as retaining AutoHebb's reliance on meta-learning, though ablating both Hebbian plasticity and neuromodulation; the 2-factor Hebbian models (2F-A and 2F-ABCD) can be interpreted as retaining AutoHebb's reliance on Hebbian learning (and meta-learning, which still is performed on the fixed weights and fixed plasticity parameters), though ablating neuromodulation; lastly, ANML in particular can be interpreted as retaining AutoHebb's reliance on neuromodulation and meta-learning, though ablating Hebbian plasticity. Below we provide a table summarizing these relationships. Broadly, we suggest that exceptional performance of AutoHebb (as compared to these various benchmarks) argues for the importance of each of the three main model components of AutoHebb.

\begin{table}[h]
\centering
\renewcommand{\arraystretch}{1.3}
\begin{tabular}{r|c|c|c|}
\hline
 &
\textbf{Meta-learning} &
\textbf{Hebbian} &
\textbf{Neuromodulation} \\
\hline

EWC             &            &            &            \\

SI              &            &            &            \\

DER             &            &            &            \\

A-GEM           &            &            &            \\

MAML            & \filledcell  &            &            \\

Reptile         & \filledcell &            &            \\

OML             & \filledcell &            &            \\

ANML            & \filledcell &            & \filledcell \\

MER             & \filledcell &            &            \\

FLP             & \filledcell & \filledcell &            \\

2F-A      & \filledcell & \filledcell &            \\

2F-ABCD   & \filledcell & \filledcell &            \\
\hline
AutoHebb        & \filledcell & \filledcell & \filledcell \\
\hline

\end{tabular}
\caption{\textbf{Comparison of model types.}}
\label{tab:model_properties}
\end{table}

\subsection{Benchmarking measures}\label{appendixsubsec:benchmarkingdetails}
The reported values were computed as follows: the accuracies on critical pairs described in Section \ref{subsec:zeroshotresults} were averaged over 100 episodes, then the mean accuracies were averaged over 6 different meta-trained models. We provide 2 standard errors over these 6 runs. The TE values are the equal-weighted mean of Same-group and Diff-group values, i.e. $(\text{Same-group test accuracy} + \text{Diff-group test accuracy})/2$.

\subsection{Training}\label{appendixsubsec:experimenttrain}
With the exception of mini-batch gradient descent methods MAML and Reptile (\cite{finn2017modelagnosticmetalearning, nichol2018firstordermetalearning}), each model's online training setup was aligned to AutoHebb's as closely as possible while simultaneously leveraging their unique design elements. While meta-trainable parameter counts differ between models, all parameter counts are within the same order of magnitude, and the depths and widths are consistent throughout. Since the architectures were consistent throughout tasks, we selected the best hyperparameters based on the critical pair accuracy in TI and applied those hyperparameters throughout all tasks. Due to the heavy compute requirements for meta-training and large amounts of tunable hyperparameters in some methods such as MER (\cite{riemer2018learning}) and FLP (\cite{lindsey2020learningtolearn}), we relied on a combination of hyperparameter grid searches and recommendations from the corresponding literature. All methods use the default Adam optimizer in PyTorch (\cite{paszke2019pytorchimperativestylehighperformance}) where applicable as in AutoHebb's setup described in Appendix \ref{appendixsec:autohebbtraining}. For all models, we use 4 hidden layers (except OML), and all hidden layers are 200-dimensional, use Kaiming Uniform initialization and batch sizes of 32 in the inner loop unless otherwise stated. All models were trainable on a single H100 within 48-72 hours on all tasks. Relative to the training speed of these models, AutoHebb was around the median. The slowest model was MER due to its triple nested loop in its native form, which required us to use the batching variant provided in the appendix of \cite{riemer2018learning}. All non-meta-learning methods were trained for 100,000 episodes.

\paragraph{MAML} For the individual tasks we retain full second-order meta-gradients; for the multi-task setting we use first-order MAML (FOMAML) to reduce memory and compute costs. The meta-batch size is $K = 32$ tasks. The outer loop learning rate is $0.0003$. Per-phase inner learning rates are stored in log-space and meta-learned, initialized at $0.3$. Within each episode, the inner loop processes the training trials and test trials sequentially, with gradient descent (one update per batch) applied at each phase using its respective learned learning rate; the meta-loss is computed over the test-phase predictions. For single-task TI, the training phase consists of 20 adjacent trials and the test phase of 30 mixed trials; for single-task TE, 30 training and 30 test trials. At meta-test, we use 2 SGD steps per training trial. Note that the use (in MAML and in Reptile below) of batch updates, as well as adjustable number of update steps, have no obvious equivalent to the learning setting investigated in this study, namely online learning. Further, as online learning entails realistic updating given limited resources (time and processing), we evaluate MAML (and Reptile below) given a limited number of gradient update steps -- evaluating on a limited number of gradient steps also conforms to what is prescribed by these methods in practice (\cite{finn2017modelagnosticmetalearning}).

\paragraph{Reptile} We use the same architecture, and meta-batch size as MAML. The outer loop learning rate is $0.00005$ and inner loop learning rate is fixed at $0.1$. The readout layer is frozen during the inner loop. Reptile uses first-order updates by construction, computing the outer gradient as the displacement between the initial and adapted parameters averaged across tasks. We use one SGD update for the training batch and 5 updates for the query batch (30 trials). At meta-test, we use 3 SGD steps per training batch.

\paragraph{MER} We implement Algorithm 7 from \cite{riemer2018learning}, the single-batch speedup variant of Meta-Experience Replay. A reservoir-sampled replay buffer of capacity 200 is maintained across episodes. On each new trial, $k - 1 = 4$ examples are drawn from the buffer and combined with the current trial; the current trial receives a boosted learning rate of $s \cdot \alpha = 5 \times 0.1 = 0.5$ while buffer examples use $\alpha = 0.1$, followed by a single Reptile-style interpolation with coefficient $\gamma = 0.3$. Per-trial cost is $k = 5$ forward-backward passes. All trials are processed through the MER update loop and added to the buffer. At meta-test, the saved model parameters and buffer state are restored, and each training trial triggers the same MER update followed by reservoir insertion.

\paragraph{OML} We use a representation network (RN) of 2 hidden layers and a prediction network (PN) of 3 hidden layers \cite{javed2019metalearning}. The RN is frozen at meta-test; the PN is re-initialized fresh each episode and adapted via SGD at a learned inner learning rate (initialized to $0.01$ in log space). The outer loop learning rate is set to $0.003$. For each task we use 30 inner test trials and 30 outer evaluation trials. At meta-test, the PN is re-sampled each episode and adapted with 1 SGD steps per training trial.

\paragraph{FLP} We use a 4-hidden-layer MLP backbone with the last 2 hidden layers and the readout designated as plastic, following the direct feedback alignment architecture of \cite{lindsey2020learningtolearn}. Each plastic hidden layer $l$ carries meta-learned per-synapse plasticity coefficients $\alpha_l$ (same shape as $W_l$), a feedback-mixing scalar $\beta_l$ applied as $\sigma(\beta_l) \in (0, 1)$, a feedback weight matrix $B_l$ of shape $(d_{\text{hidden}}, 1)$ projecting the scalar output signal directly to the layer, and a per-neuron feedback bias $b_{\text{fb},l}$. The readout layer is always plastic (with $\alpha$ only; no $B$ or $\beta$). Earlier hidden layers are static and serve as a meta-learned feature extractor, analogous to OML's representation network. All plasticity parameters ($\alpha$, $\beta$, $B$, $b_{\text{fb}}$) are meta-learned via Adam, with a separate learning rate of $0.0001$ for $\alpha$ (which is highly sensitive due to amplification through the inner-loop Hebbian chain) and $0.001$ for all other parameters. At each inner-loop trial, the Oja-rule update is applied: $\Delta W_l = \alpha_l \odot \left( \mathbf{post}_l \cdot \mathbf{pre}_l^\top - \mathbf{post}_l^2 \cdot W_l \right)$, where $\mathbf{pre}_l$ is the input to layer $l$ and $\mathbf{post}_l$ is a learned mixture of the layer's feedforward activity and the rectified feedback signal: $\mathbf{post}_l = (1 - \sigma(\beta_l)) \cdot h_l + \sigma(\beta_l) \cdot \text{ReLU}(B_l \, e_{\text{out}} - b_{\text{fb},l})$. We use the raw target as the feedback signal ($e_{\text{out}} = y$) with $\beta_l$ initialized at $0.0$ ($\sigma(0) = 0.5$, equal mix of forward and feedback activity) and $\alpha_l$ initialized at $0.0$ (plasticity starts silent and is discovered by the meta-gradient). Full second-order meta-gradients are retained through the inner loop.

\paragraph{ANML} We use the parallel neuromodulator (NM) and prediction network (PN) architecture from \cite{beaulieu2020learning}, with 4 hidden layers in each trunk and all PN layers designated as plastic. The outer loop learning rate is set to $0.0001$. The inner loop runs 30 test trials and 30 evaluation trials in the outer loop. At meta-test, we use one SGD step per trial on the PN plastic parameters using the saved learned inner learning rate (intialized to $0.03$ in log space). The continual learning evaluation chains TI and TE training trials through the same plastic state without resetting.

\paragraph{2-Factor Hebbian} Both 2-factor variants share AutoHebb-1N's backbone architecture but replace the learned neuromodulator with the feedback signal appended to the input, giving dimension $2 \cdot d_{\text{item}} + 1$. Each trial involves a double forward pass: the first pass sets $r = 0$ and generates the model choice and feedback; the second pass sets $r = +1$ or $r = -1$ according to the feedback, and the resulting activations drive the Hebbian weight update. The plastic update becomes purely Hebbian---the two factors are pre-synaptic and post-synaptic activity, with the feedback signal entering only through the augmented input on the second forward pass.
 
The outer loop learning rate is set to $0.0001$. We evaluate two variants that differ only in the form of the plastic trace:
 
\begin{itemize}
    \item \textbf{2FAC\textsubscript{A}}: The standard outer-product trace, $\Delta W \propto \mathbf{post} \cdot \mathbf{pre}^\top$.
    \item \textbf{2FAC\textsubscript{ABCD}}: The generalized ABCD trace family, with per-layer meta-learned scalar coefficients:
    \begin{equation}
        \Delta W \propto A \cdot \mathbf{post} \cdot \mathbf{pre}^\top + B \cdot \mathbf{post} + C \cdot \mathbf{pre} + D.
    \end{equation}
    Coefficients are initialized at $(A, B, C, D) = (1, 0, 0, 0)$, recovering the standard outer-product rule at initialization, and are meta-trained jointly with all other parameters.
\end{itemize}

\paragraph{Non-meta-learning baselines}
We use a 4-hidden-layer MLPs without meta-learning for all baselines with Adam.
\begin{itemize}
\item \textbf{EWC} \cite{kirkpatrick2017overcomingcatastrophic} We use a learning rate of $0.03$. The importance penalty coefficient is $\lambda = 1000.0$. Since there are only two tasks in the multi-task ZI setting, only the second tasks are updated with the regularization term. We use 3 updates per training trial.
\item \textbf{SI} We set the learning rate to $0.1$ and importance penalty to $\lambda = 1.0$. Synaptic Intelligence \cite{zenke2017continuallearning} additionally tracks per-parameter path integrals during training, accumulating $\omega_i = \sum_t (-\partial L / \partial \theta_i) \cdot \Delta \theta_i$ and normalizing by the squared displacement plus a damping constant $\xi = 0.001$. As with EWC, only the second tasks are updated with the regularization term. We use one update per trial.
\item \textbf{A-GEM} \cite{chaudhry2018efficient} We use an Adam learning rate of $0.0003$ and 5 updates per trial. All trials from the first task are retained to constrain the second task optimization.
\item \textbf{DER} \cite{buzzega2020darkexperience} We use an Adam learning rate of $0.03$ and 3 updates per trial. All trials and logits from the first task are retained in the replay buffer, and we use a penalty coefficient of $10.0$.
\end{itemize}

\subsection{Figure \ref{fig:multi_task_continual} experimental details}\label{appendixsubsec:multitaskswitchingfiguredetails}
The models used for Figure \ref{fig:multi_task_continual} differ from the benchmark models as we observed some improvement with further meta-training beyond $100,000$ episodes. Each lifetime in meta-training randomly sampled 2-4 TI and TE tasks and was randomly assigned 54-65 training trials per task. In total, a lifetime could have 108-260 training trials. In TE, we excluded same-item trials to ensure models could not easily infer the task from the presentation of a special trial type that does not exist in TI. The meta-gradients were evaluated on 30 test trials for each of the last 2 presented tasks for a total of 60 test trials. These test trials for the second-to-last task could appear either at the end of training or right before the start of the last task. Further, the training trials for the last two tasks were part of the meta-gradient for all models. The learning rate for all models was $0.0001$. AutoHebb (2 neuromodulators), ANML, and OML all used the same architecture as in Appendix \ref{appendixsubsec:experimenttrain}, with one exception for AutoHebb: we used a meta-trainable plastic weight decay term $\lambda$ initialized to $\sigma(4.0)\approx.98$ as in Eq. \ref{eq:3factorrule}. At evaluation, we use the mean ZI accuracy of 100 random trajectories (random task sequence and item representations) that are shared across all models, where accuracy is averaged over all tasks learned up to that point. The bands report 2 standard errors. For each TI and TE task, we show 30 and 54 training trials, respectively.

\section{Supplementary mathematical statements}
\begin{lemma}\label{lem:affinehypercube}
An $r$-dimensional affine subspace of $\mathbb{R}^d$ contains at most $2^r$ vertices of the $d$-dimensional hypercube with vertices in $\{-1,+1\}^d$.
\end{lemma}
\begin{proof}
Let $A = \mathbf{x}+U$ be an affine subspace of $\mathbb{R}^d$ with $U$ an $r$-dimensional linear subspace of $\mathbb{R}^d$ and $\mathbf{x}$ a vector in $\mathbb{R}^d$. Let $\mathbf{u}_1$, $\mathbf{u}_2$, \ldots, $\mathbf{u}_r$ be a basis for $U$, and let $\mathbf{U} \in \mathbb{R}^{d \times r}$ be the matrix whose $i$th column is $\mathbf{u}_i$. Since $\mathbf{U}$ has rank $r$, there exists $r$ linearly independent rows in $\mathbf{U}$. If these rows are indexed by $i_1$, $i_2$, \ldots, $i_r$, then the projection \[\mathbf{\pi}:\mathbb{R}^d \to \mathbb{R}^r\] that sends coordinate $i_j$ to coordinate $j$ is bijective on $A \subset \mathbb{R}^d$. Thus, $\pi$ gives a bijection between $A \cap \{-1,+1\}^d$ and $\pi\left(A \cap \{-1,+1\}^d\right) \subset \{-1,+1\}^r$, which has at most $2^r$ vertices.
\end{proof}
 
\begin{lemma}\label{lem:hyperplanechambercount}
Let $M$ be the number of $(d-1)$-dimensional hyperplanes through the origin in $\mathbb{R}^d$. The number of regions $R(M,d)$ in $\mathbb{R}^d$ formed by the $M$ hyperplanes is at most \[R(M,d) \leq 2\sum_{i=0}^{d-1}{\binom{M-1}{i}}\, .\] 
\end{lemma}
\begin{proof}
Let $H_M$ be the $M$th hyperplane added to the current configuration of $M-1$ hyperplanes. New regions are added only where $H_M$ intersects an existing region. Therefore, the number of new regions added is at most the number of regions in the $(d-1)$-dimensional subspace $H_M$ determined by the $M-1$ hyperplanes, viewed as $(d-2)$-dimensional hyperplanes when restricted to $H_M$. This gives is a bound \[R(M,d) \leq R(M-1,d) + R(M-1,d-1) \, .\] With initial conditions
\begin{align*}
R(1,d) & = 2 \, ,\text{ and} \\
R(M,1) & = 2 \, ,
\end{align*}
we obtain the desired bound through recursion. The base cases for $M=1$ and $d=1$ are 
\begin{align*}
2\sum_{i=0}^{d-1}{0\choose i} = 2{0\choose 0} = 2 & = R(1,d)\, , \text{ and}\\
2\sum_{i=0}^{0}{M-1\choose i} = 2{M-1\choose 0} = 2 & = R(M,1)\, .
\end{align*}
For the general case, we have
\begin{align*}
R(M,d) & \leq R(M-1,d) + R(M-1,d-1) \\
 & = 2\sum_{i=0}^{d-1}{M-2\choose i} + 2\sum_{i=0}^{d-2}{M-2\choose i} \\
 & = 2\sum_{i=0}^{d-1}{M-2\choose i} + 2\sum_{i=1}^{d-1}{M-2\choose i-1} \\
 & = 2{M-2\choose 0} + 2\sum_{i=1}^{d-1}\left({M-2\choose i} + {M-2\choose i-1}\right)\, .
\end{align*}
If we apply Pascal's identity to he sum on the right-hand side, we get \[R(M,d) \leq 2{M-2\choose 0} + 2\sum_{i=1}^{d-1}{M-1\choose k}\, .\] Since \[{M-2\choose 0} = {M-1\choose 0} = 1\, ,\] we get \[R(M,d) \leq 2\sum_{i=0}^{d-1}{M-1\choose k}\, .\]
\end{proof}

\begin{theorem}\label{thm:plasticweightconvergencetiweak}
    Assume the network in Eq. \ref{eq:minimaltinetwork} and update rule in Eq. \ref{eq:plasticweightupdateti}. Let $\mathbf{x}_1$, $\mathbf{x}_2$, \ldots, $\mathbf{x}_N$ be $N$ items in $\{-1,+1\}^d$. For a trial $\mathbf{x}_{IJ}$, let $s(\mathbf{x}_{IJ})$ be the sign of a correct model response for trial type $IJ$. Let $\mathbf{Q} = \{\mathbf{q}_t\}_{t=0}^\infty$ be a sequence of trials with inputs $\mathbf{q}_t$. If $0 < \eta < \frac{1}{d}$ as in Eq. \ref{eq:plasticweightupdateti}, then $\mathbf{P}(t)$ converges as $t \to \infty$.
\end{theorem}

\begin{proof}
Let $\mathbf{w}_t = \mathbf{\bar{1}} + \mathbf{P}(t)$ and $y_t = \mathbf{w}_t\mathbf{q}_t$. Then for an update at time $t$, we have
\begin{align*}
    \|\mathbf{w}_{t+1}\|^2 & = \| \mathbf{w}_t - \eta y_t\mathbf{q}_t^\top\|^2\\
    & = \|\mathbf{w}_t\|^2 -2\eta y_t\mathbf{w}_t\mathbf{q}_t + 2\eta^2y_t^2d \\
    & = \|\mathbf{w}_t\|^2 - 2\eta y_t^2+2\eta^2y_t^2d \\
    & = \|\mathbf{w}_t\|^2 - 2\eta y_t^2(1-\eta d) \, .
\end{align*}
Since $\eta < 1/d$, the sequence of norms $\|\mathbf{w}_t\|^2$ is monotone decreasing and bounded below, hence the convergence of norms. Consequently, each dimension is a bounded sequence, so we can apply Bolzano-Weierstrass to each dimension to obtain a subsequence of weights $\{\mathbf{w}'_*\}$ that converges to some $\mathbf{w}_*$. We show $\mathbf{w}_*$ is the limit of the entire sequence.

For each trial vector $\mathbf{q}$, define the update matrix $\mathbf{M}_{\mathbf{q}} = \mathbf{I}_{2d}-\eta\mathbf{q}\mathbf{q}^\top$.
On an incorrect trial with input $\mathbf{q}$, we have
$\mathbf{w}_{t+1}=\mathbf{w}_t\mathbf{M}_{\mathbf{q}}$.
Since $\|\mathbf{q}\|^2=2d$, for any row vector $\mathbf{v}$,
\begin{equation}\label{eq:nonexpansivemap}
\|\mathbf{v}\mathbf{M}_{\mathbf{q}}\|^2
=
\|\mathbf{v}\|^2
-2\eta(1-\eta d)|\mathbf{v}\mathbf{q}|^2
\leq \|\mathbf{v}\|^2.
\end{equation}
Thus every update map is nonexpansive.

Let $\mathcal{R}$ be the set of trial vectors whose corresponding update maps are applied infinitely often. If $\mathcal{R}$ is empty, then there are only finitely many updates since there are only finitely many trial vectors, and the weights are eventually constant. Otherwise, since there are finitely many update types, there exists a time $T_0$ after which every applied update map belongs to $\mathcal{R}$.

Now, suppose, towards a contradiction, that some update
map in $\mathcal{R}$ does not fix $\mathbf{w}_*$.
Define $\mathcal{N} = \{\mathbf{q}\in\mathcal{R}:\mathbf{w}_*\mathbf{q}\neq0\}$. By assumption, $\mathcal{N}$ is nonempty. Since it is finite (a subset of possible trial vectors) and every value in the following minimum is positive, we have $a=\min_{\mathbf{q}\in\mathcal{N}}|\mathbf{w}_*\mathbf{q}|>0$. For any $\mathbf{q}\in\mathcal{N}$, the reverse triangle inequality gives
\[
\begin{aligned}
|\mathbf{w}\mathbf{q}|
&\geq |\mathbf{w}_*\mathbf{q}|
 -|(\mathbf{w}-\mathbf{w}_*)\mathbf{q}|\\
&\geq a-\sqrt{2d}\|\mathbf{w}-\mathbf{w}_*\|\,
\end{aligned}
\]
where we apply the $L^2$ operator norm inequality 
\begin{equation*}
\|\mathbf{M}\mathbf{x}\| \leq \|\mathbf{M}\|_{\text{op}}\|\mathbf{x}\| \leq \|\mathbf{M}\|_F\|\mathbf{x}\|\, .
\end{equation*}
Therefore, whenever $\|\mathbf{w}-\mathbf{w}_*\| < \frac{a}{2\sqrt{2d}}$, an update associated with any $\mathbf{q}\in\mathcal{N}$ satisfies $|\mathbf{w}\mathbf{q}|>a/2$ and, by Eq. \ref{eq:nonexpansivemap}, decreases the squared norm by more than
\[
\kappa
:=
\frac{\eta(1-\eta d)a^2}{2}>0.
\]

Recall that $\|\mathbf{w}_t\|^2$ converges, hence has some limit $L$, and there is a subsequence $\{\mathbf{w}_{t_k}\}$ that conveges to $\mathbf{w}_*$.
We may choose $t_k\geq T_0$ sufficiently large that $\|\mathbf{w}_{t_k} \mathbf{w}_*\| < \frac{a}{2\sqrt{2d}}$ and $\|\mathbf{w}_{t_k}\|^2-L<\kappa$.

Let $s\geq t_k$ be the first subsequent time at which an update associated with some
$\mathbf{q}\in\mathcal{N}$ is applied. Such a time exists because these update types occur infinitely often. Every update between $t_k$ and $s$ fixes $\mathbf{w}_*$, since its trial vector belongs to $\mathcal{R}\setminus\mathcal{N}$. By nonexpansiveness, such updates cannot increase the distance to $\mathbf{w}_*$. Hence
\[
\|\mathbf{w}_s-\mathbf{w}_*\|
\leq
\|\mathbf{w}_{t_k}-\mathbf{w}_*\|
<\frac{a}{2\sqrt{2d}}.
\]
The update at time $s$ therefore decreases the
squared norm by more than $\kappa$:
\[
\begin{aligned}
\|\mathbf{w}_{s+1}\|^2
&<\|\mathbf{w}_s\|^2-\kappa\\
&\leq\|\mathbf{w}_{t_k}\|^2-\kappa\\
&<L,
\end{aligned}
\]
where we use the fact that $\|\mathbf{w}_t\|$ is nonincreasing. This contradicts the fact that
$\|\mathbf{w}_t\|^2$ is nonincreasing and converges
to $L$, so every term must be at least $L$.

Therefore $\mathcal{N}$ is empty: every update map
that occurs infinitely often fixes $\mathbf{w}_*$.
For $t\geq T_0$, nonexpansiveness consequently gives
\[
\begin{aligned}
\|\mathbf{w}_{t+1}-\mathbf{w}_*\|
&=
\|(\mathbf{w}_t-\mathbf{w}_*)
  \mathbf{M}_{\mathbf{q}_t}\|\\
&\leq\|\mathbf{w}_t-\mathbf{w}_*\|
\end{aligned}
\]
on update steps, with equality on steps without
an update. Thus the sequence $\|\mathbf{w}_t-\mathbf{w}_*\|$ is nonincreasing for $t\geq T_0$. Since its subsequence $\|\mathbf{w}_{t_k}-\mathbf{w}_*\|$ converges to zero, the entire sequence converges to zero. Hence \[\lim_{t\to\infty}\mathbf{w}_t=\mathbf{w}_*.\] Finally, since $\mathbf{P}(t)=\mathbf{w}_t-\mathbf{\bar{1}}$, the plastic weights also converge.

\end{proof}


\end{document}